\documentclass[11pt]{article}
\usepackage[english]{babel}
\usepackage[latin1]{inputenc}

\usepackage[colorinlistoftodos]{todonotes}

\usepackage[pdfpagelabels]{hyperref}
\hypersetup{colorlinks = true,
  linktoc = all, linkcolor = black, citecolor = black}

\usepackage[normalem]{ulem}
\usepackage[toc,page]{appendix}
\usepackage{dsfont}
\usepackage{bm}
\usepackage{authblk}

\usepackage{amsmath}
\usepackage{amsfonts}
\usepackage{upgreek}
\usepackage{mathtools} %
\mathtoolsset{showonlyrefs}

\usepackage{algorithm}
\usepackage{algorithmic}

\usepackage{pgfplots}
\usepackage{xcolor}
\usepackage{graphicx}
\usepackage{subfig}
\usepackage{amssymb}
\usepackage{tikz}
\usetikzlibrary{arrows,calc} \usetikzlibrary{decorations,arrows}
\usetikzlibrary{decorations.pathreplacing}

\usepackage{lipsum}
\usepackage{ifthen}
\usepackage{xifthen} 
\usepackage{xparse}
\usepackage{pgffor}
\usepackage{enumitem}

\numberwithin{equation}{section}

\usepackage[english]{babel} %
\usepackage[latin1]{inputenc}
\usepackage[T1]{fontenc}
\usepackage{amsthm}
\usepackage{array,colortbl}
\usepackage{amsmath,amsfonts,amssymb,bm}

\definecolor{jz}{rgb}{0,0.7,0}
\definecolor{autha}{rgb}{0,0.2,1}
\definecolor{authb}{rgb}{0.89,0,0.13}%
\definecolor{authc}{rgb}{0.333,0,0.666}
\definecolor{authd}{rgb}{1,0.5,0}%

\newtheorem{theorem}{Theorem}[section]

\newtheorem{definition}[theorem]{Definition}
\newtheorem{example}[theorem]{Example}

\newtheorem{lemma}[theorem]{Lemma}

\newtheorem{proposition}[theorem]{Proposition}
\newtheorem{remark}[theorem]{Remark}

\newtheorem{assumption}[theorem]{Assumption}

\newcommand{\dup}[2]{\left\langle #1, #2\right\rangle}

\newcommand{\norm}[2][]{\| #2 \|_{#1}}

\newcommand{\set}[2]{\{#1\,:\,#2\}} \newcommand{\setc}[2]{\left\{#1\,
    :\,#2\right\}}

\newcommand{\dd}{\mathrm{d}}

\newcommand{\dfn}{\vcentcolon=}

\DeclareMathOperator{\supp}{supp}

\DeclareMathOperator{\argmin}{argmin}

\DeclareMathOperator{\tr}{tr}

\newcommand{\eps}{\varepsilon}
\newcommand{\C}{\mathbb{C}}
\newcommand{\N}{\mathbb{N}}
\newcommand{\R}{\mathbb{R}}

\DeclareFontEncoding{FMS}{}{}
\DeclareFontSubstitution{FMS}{futm}{m}{n}
\DeclareFontEncoding{FMX}{}{}
\DeclareFontSubstitution{FMX}{futm}{m}{n}
\DeclareSymbolFont{fouriersymbols}{FMS}{futm}{m}{n}
\DeclareSymbolFont{fourierlargesymbols}{FMX}{futm}{m}{n}
\DeclareMathDelimiter{\VERT}{\mathord}{fouriersymbols}{152}{fourierlargesymbols}{147}

\providecommand{\bbE}{\mathbb{E}}

\providecommand{\bbN}{\mathbb{N}}

\providecommand{\bbP}{\mathbb{P}}

\providecommand{\bbR}{\mathbb{R}}

\providecommand{\CA}{\mathcal{A}}
\providecommand{\CB}{\mathcal{B}}

\providecommand{\cF}{\mathcal{F}}

\providecommand{\CN}{\mathcal{N}}

\providecommand{\Be}{{\boldsymbol{e}}}

\providecommand{\Bv}{{\boldsymbol{v}}}

\providecommand{\Bx}{{\boldsymbol{x}}}
\providecommand{\By}{{\boldsymbol{y}}}

\newcommand{\Bmu}        {{\boldsymbol{\mu}}}
\newcommand{\Bnu}        {{\boldsymbol{\nu}}}
\newcommand{\Bxi}        {{\boldsymbol{\xi}}}

\newcommand{\Brho}       {{\boldsymbol{\rho}}}

\newcommand{{\bomega}}   {{\boldsymbol{\omega}}}
\newcommand{\Bnul}       {{\boldsymbol{0}}}

\usepackage{fullpage}

\DeclareSymbolFont{bbold}{U}{bbold}{m}{n}
\DeclareSymbolFontAlphabet{\mathbbold}{bbold}

\newcommand{\calD}{{\mathcal{D}}}
\newcommand{\calE}{{\mathcal{E}}}

\newcommand{\calG}{{\mathcal{G}}}

\newcommand{\calV}{{\mathcal{V}}}

\newcommand{\calX}{{\mathcal{X}}}
\newcommand{\calY}{{\mathcal{Y}}}

\providecommand{\argmin}{\operatorname*{argmin}}

\newboolean{final}
\setboolean{final}{false}
\newcommand{\cQ}{\mathcal{Q}}
\newcommand{\cR}{\mathcal{R}}
\newcommand{\cX}{\mathcal{X}}
\newcommand{\cZ}{\mathcal{Z}}
\newcommand{\cY}{\mathcal{Y}}
\newcommand{\cA}{\mathcal{A}}
\newcommand{\cC}{\mathcal{C}}
\newcommand{\cD}{\mathcal{D}}
\newcommand{\cE}{\mathcal{E}}
\newcommand{\cB}{\mathcal{B}}

\newcommand{\cG}{\mathcal{G}}

\newcommand{\cL}{{\mathcal L}} %
\newcommand{\cK}{{\mathcal K}} 
\newcommand{\cM}{{\mathcal M}} %
\newcommand{\FNN}{{\mathrm F_{NN}}}
\newcommand{\Flayer}{{\mathrm F_{LayerNorm}}}
\newcommand{\Amult}{{\mathrm A_{multihead}}}
\newcommand{\Wmult}{{\mathrm W_{multihead}}}
\newcommand{\Tout}{{\mathrm T_{out}}}
\newcommand{\Tin}{{\mathrm T_{in}}}
\newcommand{\EL}{{\mathrm E_L}}
\newcommand{\Nsa}{{\mathrm N_{sa}}} %
\newcommand{\IE}{{\mathbb E}}

\newcommand{\bsnul}{{\boldsymbol{0}}}

\newcommand{\bmf}{{\boldsymbol{f}}}

\definecolor{dkgreen}{rgb}{0.0, 0.5, 0.0}
\newcommand{\Tsch}{{Cheby\v{s}ev}} %
\usepackage{cancel}

\DeclareSymbolFont{bbold}{U}{bbold}{m}{n}
\DeclareSymbolFontAlphabet{\mathbbold}{bbold}
\newcommand{\domain}{{\mathrm D}}

\newcommand{\be}{\begin{equation}}
\newcommand{\ee}{\end{equation}}

\newcommand{\covar}{\cC}
\newcommand{\coeffmap}{g}
\newcommand{\coeff}[1][]{\ifx&#1&\Bar{\Bx}\else\bar{x}_{#1}\fi}
\newcommand{\spatialvar}{\eta}
\newcommand{\spatialvarother}{\zeta}
\newcommand{\muV}{\mu_V} %
\newcommand{\muVperp}{\mu^v} %

\usepackage{tcolorbox}
\usepackage{fontawesome5}
\newtcolorbox{codebox}{
    colback  = gray!12,
    colframe = gray!12,
    boxrule  = 0pt,
    arc      = 4pt,
    left=6pt, right=6pt, top=5pt, bottom=5pt,
    before upper = {\faLaptopCode \enspace\bgroup},
    after upper  = {\egroup},
  }
\begin{document}
\title{Operator Surrogates}
\author[1,a]{Christoph Schwab}
\author[2,1,b]{Josephine Westermann}
\author[2,c]{Jakob Zech}

\affil[1]{\footnotesize
  SAM, ETH Z\"urich, R\"amistrasse 101, 8092 Z\"urich, Switzerland}
\affil[2]{\footnotesize IWR, Uni Heidelberg, INF 205, 69120
  Heidelberg, Germany}
\affil[a]{\footnotesize \href{mailto:schwab@math.ethz.ch}{schwab@math.ethz.ch}}
\affil[b]{\footnotesize \href{mailto:josephine.westermann@math.ethz.ch}{josephine.westermann@math.ethz.ch}}
\affil[c]{\footnotesize \href{mailto:jakob.zech@uni-heidelberg.de}{jakob.zech@uni-heidelberg.de}}
\date{July 15, 2026}

\maketitle
\begin{abstract}
\begin{sloppypar}
 This chapter provides a mathematical introduction to operator surrogates, which are finite-parametric maps %
{approximating} complex (typically nonlinear) operators between infinite-dimensional (function) spaces. 
\end{sloppypar}

{A typical approach to constructing operator surrogates is}
the encoder-approximator-decoder paradigm, 
whereby infinite-dimensional inputs are mapped to a finite-dimensional latent space, 
processed, and then mapped back to the infinite-dimensional output space.
We {discuss} %
various methods for creating efficient representation systems, 
including stable linear frameworks like Banach frames and data-driven techniques such as 
Principal Component Analysis (PCA), Active Subspaces, and nonlinear autoencoders. 
Subsequently, we survey key surrogate architectures, 
contrasting spectral methods, 
based on polynomial chaos expansions and related technologies, 
with more recent neural network approaches. 
{We further discuss architectures that do not rely on representation systems, such as }
Fourier Neural Operators and transformers.
{We elaborate on their relationship and how they relate}
to other neural network operator surrogates, 
particularly those obtained by so-called `algorithm-unrolling'.

Additionally, we provide an introduction to the theoretical analysis of approximation and generalization errors. 
We explain how underlying low-dimensional structures, provided for example by holomorphy, are essential to overcome the 
{so-called ``curse of dimensionality'': 
this refers to the fact that the complexity of numerical algorithms tends to increase exponentially 
in the problem dimensionality. 
Since operator surrogates deal with the approximation of mappings between infinite dimensional spaces, 
traditional approaches suitable for low-dimensional problems typically fail in these situations.}
We present convergence rate bounds for holomorphic operators 
that depend on the number of trainable parameters and the regularity of the input and output data. 
Throughout, we prioritize intuition over {detailed} technical proofs 
{and
provide precise pointers to the literature for detailed arguments.
}
\end{abstract}
  
\numberwithin{equation}{section}

\tableofcontents

\section{Introduction}\label{sec:Intr}
Partial (Integro) Differential Equations (P(I)DEs for short)
are the most widely employed 
formalism to express continuum models in engineering and in the sciences. 
P(I)DEs do not admit explicit solutions, as a rule,
and \emph{numerical approximations} need to be employed for predictive 
simulations. 
Accordingly, during the past decades, 
numerical solution technology for large classes of PDEs has matured,
with broad classes of solution methodologies
(e.g.\ finite difference, finite element, finite volume, spectral methods ...)
today being theoretically founded and practically deployed.
The focus of this development was on 
\emph{efficient numerical approximation of PDE solutions for one instance of input data}.
In the course of this development,
numerical PDE solvers have been optimized to 
in many cases near-linear complexity, 
in particular for elliptic and parabolic {P(I)DEs}.%

Increasingly, however,
the availability of \emph{ubiquitous observation and measurement data}
has generated new computational scenarios
which mandate \emph{efficient numerical many-query PDE solutions}.
Examples include
data assimilation \cite{AschDatAss},
uncertainty quantification \cite{UQSmith},
digital twins (see \cite{KWDigTwin24,AschDT,pash2025predictivedigitaltwinsquantified,liu2024deepneuraloperatorenabled} 
and references there),
efficient sampling of statistical solutions of evolution equations 
(e.g.\ \cite{ChaeFoias94,FMRT01,LMS16} and the references there) etc., 
all of which mandate \emph{efficient many query solvers}.

\emph{Parallel execution} 
of the mentioned existing single query 
solution methodologies has been developed in the past decades
but will, usually, not exploit common structure of input queries. 
This is well-known, for example, in 
Single-Level Monte-Carlo (SLMC) simulations in {Uncertainty Quantification (UQ);}
as is by now well-known, 
\emph{Multi-Level discretizations of parametric PDEs} 
such as Multi-Level Monte-Carlo (MLMC) 
simulation strategies offer significant error-vs-work gains
for PDEs with uncertain input data (e.g.\ \cite{LMS16}) 
albeit limited by the low MC convergence rate.

The efficient numerical realization of operators, 
that is generally nonlinear maps between infinite-dimensional function spaces, 
has become a key research topic at the interface of scientific computing (SC) and machine learning (ML). 
A core task in this field, referred to as Scientific Machine Learning {(SciML),}
is to provide efficient numerical surrogates $\hat{\calG}$ 
for the data-to-solution maps
\begin{equation}\label{eq:G}
  \calG: \calX \to \calY
\end{equation}
of operator equations, particularly PDEs; here $\cX$ and $\cY$ are
typically Banach or Hilbert spaces.  {\emph{Operator surrogates}
  are approximations of such mappings, where $\calG$ is truly
  infinite-dimensional; by this we mean that the problem cannot be
  reduced to that of approximating a function whose domain (or codomain) is finite
  dimensional. We will make this more precise in Section
  \ref{sec:Ck}.}  {Such surrogates can be building blocks in
  algorithms that involve many-query evaluations of PDE models.
  Examples are inverse problems and algorithms such as MCMC or
  filtering for data assimilation into PDE models.}

Mathematically, 
the goal is to construct parsimonious, 
finite-parametric approximations of these maps that come with rigorous a-priori performance guarantees. 
Ideally, these approximations should also achieve some form of 
optimality in terms of accuracy versus computational work for a given set of input data. 
This area has been significantly advanced by deep learning methodologies 
that assimilate data-solution pairs into operator surrogates. 
For recent overviews, we refer for example to the surveys 
\cite{kovachki2021neural,kovachki2024operator,MishraActa24} and the references therein.

With the currently emerging availability 
of hierarchic classes of finite-parametric maps between function spaces, 
which this chapter discusses, 
the area of \emph{data assimilation and data-driven {updates}
in hierarchies of operator surrogates} is emerging, 
in particular in connection with 
digital twins \cite{KWDigTwin24,AschDT,AschDatAss}.
This line of reasoning tacitly assumes that 
the operator equations which describe the physics of interest
are known in closed form.

{
\emph{Operator Learning} refers to computational methodologies that 
associate, for an input-output relation given e.g.\ as solution map of an operator equation 
in function spaces (hence, with infinite-dimensional domain and codomain), 
\emph{finite-parametric operator surrogates}
take the form of numerical approximations of such relations, ideally with performance guarantees.
In this process, operator surrogates from a given set of \emph{surrogate architectures} 
are numerically calibrated to the input-output relation, usually subject to data.
}

Within the context of AI and natural language processing, 
the process of operator learning and model updating itself 
can be transferred to computational methods \cite{POSEIDON2024,molinaro2025generativeaifastaccurate}. 
{
This is of particular interest for multi-modal data.
}
The mathematical formalization of operator-learning methodologies 
in this context is still in its infancy and beyond the scope of this (introductory)
chapter. 
However, the mathematical tools presented here will almost certainly 
be required in {some form towards} formalizing such algorithmic approaches.

Besides neural networks, 
operator surrogates can be built from other approximation architectures,
and we provide a first example for dynamical systems 
in Section~\ref{sec:IVP}.
Owing to the infinite dimension of the input and output function spaces, 
in the design and mathematical analysis of numerical operator approximation 
architectures, in particular in the establishment of convergence rate
bounds, the so-called \emph{curse of dimensionality} (CoD) 
takes a prominent role: 
approximation architectures which are successful in operator learning
must realize some mechanism of overcoming the CoD.
To date, several broad classes of approximation architectures have
emerged {with} good performance on practical operator learning tasks {and}
with mathematical foundations in various stages of completeness.
A rough taxonomy could be:
\begin{enumerate}
\item Spectral Operators based on \emph{(Generalized) Polynomial Chaos
    (g)PC Approximations} and related methods: 
  the CoD in the parametric approximations is 
  overcome by \emph{sparsity of gPC expansions} \cite{HSZ24} in holomorphic operators.
  {Operator classes of finite regularity (such as $k$-times 
       {Fr\'{e}chet-}differentiable maps {with finite $k$}) 
       generally suffer from the curse of dimensionality, see e.g.
       \cite{kovachki2024datacomplexityestimatesoperator} and the references there.}
\item \emph{Neural Operators, Neural Operator Networks}:
the CoD in the parametric approximations is overcome by
sparsity in the operator and strong expressivity
of deep neural networks, which allow {us} to detect such sparsity.
For input and outputs from (infinite-dimensional) 
function spaces, this can be realized in the 
so-called encoder-approximator-decoder framework
\cite{SLSMGEM22,kovachki2021neural,li2020fourier,brivio2023error,FRANCO2023129,Franco_2022}.
\item \emph{Random Feature Neural Networks, Extreme Learning Machines and Reservoirs}:
the CoD is overcome by probabilistic architectures and error measures, 
and mathematical analysis furnishes 
statistical convergence guarantees without the CoD, 
via suitable central limit theorems
\cite{lanthaler2023errorboundslearningvectorvalued,GononRndFNN4BS-23,LocEML}.

\item \emph{Algorithm Unrolling}: 
A further approach towards construction of operator networks 
can be termed ``algorithm unrolling''.
{Going back to \cite{monga2020algorithm},}
it consists in re-writing numerical algorithms
with mathematical convergence guarantees in terms of 
error vs. work
for numerical solution approximation 
of nonlinear (operator) equations 
in terms of deep neural network architectures. 
One execution of an iterative solver (e.g.\ \cite{MGNetJCP22})
respectively
of a constructive proof of existence of solutions (e.g.\ \cite{FSZ25_1134}),
will give rise to a map of the solution (approxim{atio}n) space into itself,
and 
will entail recurrent or so-called ``U-Net'' structures 
within the Operator Network architecture. 
{
\item \emph{Kernel Methods}: 
Kernel regression has in recent years been proposed for operator learning e.g. 
in 
\cite{BATLLE2024112549,GPRegr21,lowery2026kernelneuraloperatorsknos}.
}
\end{enumerate}
{
We shall focus in the present review on details for methodologies i)  and ii),
but do not develop technical details of approaches iii) - v).
Any of the above can be combined by composition.
}
\subsection{Existing Results and Algorithms}
\label{sec:ExRes}
Fundamental \emph{mathematical questions} {about} %
classes of {n}eural {o}perators
are roughly grouped around %
approximation properties {on the one hand,} and %
computational feasibility and efficiency
{(}and, related to this,
generalization, optimization, 
robustness/stability and sample complexity, to name but a few{) on the other}.

Approximation properties %
{describe whether}
a given class of neural operators
can, as the number of trainable parameters tends to infinity, 
in principle approximate a given map between function spaces to any prescribed 
tolerance, in a given metric and for a given class of inputs.
This is commonly referred to as \emph{universal approximation}.
Universal approximation of  
{neural operator surrogates}
has been established for a wide range of architectures. 
We mention only the seminal work \cite{ChenChen1993}, 
which focused on uniform (``worst case'') 
universality on compact subsets of inputs, and 
subsequent work \cite{SLSMGEM22,HSZ24,schwab2023deep} 
where compactness of admissible inputs was relaxed, 
and $L^p(\calX,\calY;\mu)$ - universality was shown.
We also mention the recent 
\cite{furuya2024transformersuniversalincontextlearners,
      benth2022neuralnetworksfrechetspaces,
      godeke2025newuniversaloperatorapproximation}
where universal approximation was shown for a wide range of 
architectures comprising those considered here, in particular transformers,
acting between Banach-, Metric- and Fr\'{e}chet-spaces. 

A more refined viewpoint %
asks 
at which rates can a given class of
{neural operator surrogates}
approximate 
suitable maps $\calG$ between function spaces, given 
conditions on the set of admissible input data. 
Mathematical results %
on operator approximation rates 
are, at this time, less developed than {results on universal approximation.} %
They require specific assumptions on admissible input sets, 
typically related to regularity in scales $\{ \calX^s \}_{s>0}$ 
of subspaces $\calX^s \subset \calX$ of the input space,
and stronger structural assumptions. 

{Mathematical results on approximation rates typically assume 
\emph{exact representation} (i.e. infinite-precision) of network weights and biases.}
{In this idealized scenario, 
the Kolmogorov-Arnold representation theorem \cite{kolmogorov1957representations} enables}
universality and approximation rates with shallow
NNs (see, e.g., \cite{schwab2023deep} for concrete results).
{When implemented on digital computers with finite-precision, however, the Kolmogorov construction
faces stability obstructions. }
{For a mathematical analysis of \emph{quantization} in the context of function approximation see \cite{ou2024quantizationregimesrelunetworks} and references there.
For leveraging quantization in neural operators see \cite{tu2024guaranteedapproximationboundsmixedprecision}.
Further,} %
in
\cite{Lanthaler2023,lanthaler2024operatorlearninglipschitzoperators}
operator network expressivity is studied 
with focus on information-theoretic aspects,
where the finite bitstring length of floating point numbers
is taken into account.

\subsection{Goals and Audience}
\label{sec:Goals}
This chapter introduces key mathematical concepts and algorithmic foundations of the growing field of operator surrogates. Our primary goal is to build intuition through a selective overview of the field. Nonetheless, we provide mathematically precise statements, and give proofs or proof sketches of some key results. For more intricate and technical arguments, readers are directed to the relevant literature. The material is aimed at a Master's-level audience and assumes familiarity with basic concepts in finite-dimensional deep learning, as, e.g., presented in \cite{petersen2025mathematicaltheorydeeplearning}.

\subsection{{Outline}}
\label{sec:Strct}
We begin in Section~\ref{sec:ModProb} by introducing two
  prototypical model problems to illustrate basic concepts and
  applications of operator learning.
  In Section \ref{sec:representation} we introduce tools
  for efficient dimension reduction and representation of elements in
  the input and output spaces $\cX$ and $\cY$.  This is a crucial
  prerequisite for constructing any numerical operator surrogate of
  $\cG$ in \eqref{eq:G}.

With these fundamentals in place, we move to the constructive and analytical core: Section~\ref{sec:Arch} discusses and compares a selection of surrogate architectures, with a focus on spectral and neural operators. The practical task of training such models is explai{ned} in Section \ref{sec:training}, where we detail a typical training setup for an encoder-decoder architecture. Next, in Section \ref{sec:approximation} we transition to a theoretical analysis, 
and discuss approximation properties of these surrogates. 
We start by reviewing the relevance of compactness and its implications for universal approximation. Afterwards we explain the necessary occurrence of the curse of dimension for operator classes of finite regularity. While there are different scenarios under which it can be overcome, we focus on the class of holomorphic operators. Building on these results, {in} Section \ref{sec:generalization} we briefly discuss the topic of generalization, without going into details on the proofs.

Finally, Section~\ref{sec:Concl} concludes the chapter by providing an outlook on extensions of these results, 
discussing methods not covered, and pointing to further relevant literature.

{
\subsection{Code}

This work is accompanied by a repository of Python Jupyter notebooks,
available at \url{github.com/JoWestermann/operator-surrogates/tree/release/v1},
implementing core concepts from the text alongside visualizations and numerical experiments.
Throughout the manuscript, text boxes marked with the icon~\faLaptopCode\ 
point the reader to specific notebooks related to the discussed concepts.
}

\subsection{Notation}
\label{sec:Notation}
\paragraph{Numbers and constants.}
We write $\N = \{1,2,...\}$ and $\N_0 = \{0,1,2,...\}$. 
Throughout, $C \lesssim D$ means that $C$
can be bounded by a multiple of $D$, independently of
parameters which $C$ and $D$ may depend on.
$C \gtrsim D$ is defined as $D \lesssim C$, and  
$C\simeq D$ as $C\lesssim D$ and $C \gtrsim D$.

\paragraph{Banach and Hilbert spaces.}
For a Hilbert or Banach space $\cX$ we denote its norm by $\norm[\cX]{\cdot}$.
For a %
Hilbert space $\calX$, the inner product of $v$ and $w\in\calX$ is
denoted by $\dup{v}{w}$. The (unique, cf., e.g., \cite{MunozCplxB}) complexification of a real Hilbert space $\cX$ 
{is} denoted by $\cX_\C$. 
Similarly, for Banach spaces $\cX$, we let $\cX_\C$ be a complexification with a natural norm \cite{MunozCplxB,JZDiss}.

For Banach spaces $\calX$ and $\calY$, we denote by $L(\calX,\calY)$ 
the bounded, linear maps from $\calX$ to $\calY$, by $C(\calX,\calY)$ 
the continuous maps, and for $k\in \bbN_0$
by $C^{k}(\calX,\calY)$ the $k$-times continuously Fr\'echet
differentiable maps. The $k$th Fr\'echet derivative of $\cG\in C^k(\cX,\cY)$
at $x\in\cX$ is denoted by $D^k\cG(x)\in L(\cX,\cY)$, and if $\cX$ is a Hilbert space and $\cY=\R$, then $\nabla \cG(x)\in\cX$ is the gradient.
Moreover, $\cX' = L(\cX,\R)$ is the dual space of $\cX$.

\paragraph{Lebesgue and Sobolev spaces.}
For %
$p\in (0,\infty)$, 
the space of real-valued, $p$-summable sequences indexed over $\N$
is denoted by $\ell^p(\N)$.
All Lebesgue spaces with summability exponent $p$ are assumed to 
have the usual modification for $p=\infty$.

For a Lipschitz domain $\domain\subseteq\R^n$, we use the usual function spaces: 
$L^1(\domain)$ denotes the space of absolutely (Lebesgue) integrable, 
real-valued functions in $\domain$.
For $1< p< \infty$, we denote 
$L^p(\domain) := \{ u:\domain\to \bbR,~ |u|^p \in L^1(\domain) \}$ 
the space of $p$-Lebesgue integrable, real-valued functions,
with corresponding %
norm denoted by $\| u \|_{L^p(\domain)}$. Again, for $p=\infty$ the usual modification. For $k\in\N$ we let
$H^k(\domain)=\set{f:\domain\to\R}{\partial^\Bnu f\in L^2(\domain)~\forall \Bnu \in \N_0^d : |\Bnu|\le k}$.
Finally, $\tr \in L(H^1(\domain), H^{1/2}(\partial\domain))$ is the trace map, $H_0^1(\domain):=\set{u\in H^1(\domain)}{\tr(u)=0}$, and $H^{-1}(\domain)$ is the dual space of $H_0^1(\domain)$.

\section{Model Problems}
\label{sec:ModProb}
We illustrate basic concepts of operator learning for
two model problems. The first is the standard parametric
elliptic PDE, and we review known results on well-posedness
and parameter regularity. The second is the flow
of an ODE in finite dimensions; while technically not operator learning
(due to the finite dimensionality), it allows us to
showcase key features of evolution problems.

\subsection{{Stationary} Diffusion}
\label{sec:RD}
Let $n\ge 2$, and let $\domain\subseteq\R^n$ be a bounded Lipschitz domain.
Let $a_->0$, let a diffusion coefficient $a$ be given in %
\be\label{eq:a+}
L^\infty_{a_-}(\domain) := \setc{a \in L^\infty(\domain)}{ {\rm ess} {\inf_{\spatialvar \in \domain} a(\spatialvar)} \ge a_- } 
.
\ee
{
With this coefficient and 
a source term $f\in L^{2}(\domain)$,
}
consider the boundary value problem (BVP): 
find ${y}\in H^1_0(\domain)$ such that 
\be\label{eq:BVP}
-\nabla_\spatialvar \cdot(a{(\spatialvar)} \nabla_\spatialvar {y(\spatialvar)}) = f 
\;\mbox{in}\; \domain,
\quad \tr({y}) = 0 \;\mbox{on}\; \partial\domain \;.
\ee

To define weak solutions and discretizations, 
we adopt a variational formulation of \eqref{eq:BVP}.
There are, generally, many possibilities to re-formulate \eqref{eq:BVP}.
The most straightforward formulation seeks the solution 
${y}\in \calY := H^1_0(\domain)$ and 
is obtained by multiplying \eqref{eq:BVP}
with a test-function $v\in H^1_0(\domain)$ 
and integrating by parts. It reads:
find  ${y}\in H^1_0(\domain)$ such that
\be\label{eq:RDVar1}
b(a;{y},v):= \int_D a {(\spatialvar)} \nabla {y(\spatialvar)} \cdot \nabla v{(\spatialvar)} d \spatialvar 
= 
\int_\domain f{(\spatialvar)} v{(\spatialvar)} d \spatialvar
\quad 
\forall v\in H^1_0(\domain).
\ee
{
We choose as data-space of problem~\eqref{eq:RDVar1} 
$\calX = L^\infty(\domain) \times H^{-1}(\domain)$, 
and endow it with the 
sum norm 
$\| (f,a) \|_{\calX} := \| f \|_{H^{-1}(\domain)} + \| a \|_{L^\infty(\domain)}$.
}

For every $a \in {L^\infty_{a_-}(\domain)}$ %
the bilinear form $b(a;\cdot,\cdot):\calY\times \calY\to \bbR$
is continuous and coercive, i.e.\ by \eqref{eq:a+}
\be\label{eq:ConCoe}
\forall v,w\in \calY: \quad 
| b(a;v,w){|} \leq \| a \|_{L^\infty(\domain)} \| \nabla v \|_{L^2(\domain)} \| \nabla w \|_{L^2(\domain)} \;,
\quad 
b(a;v,v) \geq a_- \| \nabla v \|^2_{L^2(\domain)}
\;.
\ee
By Poincar\'e's inequality,
the expression $\| \nabla v \|_{L^2(\domain)}$ 
is a norm on $\calY = H^1_0(\domain)$, 
the so-called ``energy norm'', 
{
i.e.\ $\| v \|^2_\calY \simeq \int_\domain a |\nabla v |^2 dx = \| \sqrt{a} \nabla v \|^2_{L^2(\domain)}$.
Here, $\simeq$ denotes equivalences with constants depending 
on $a_-$ in \eqref{eq:a+} and on an upper bound $a_+$ of $\| a \|_{L^\infty(\domain)}$.
}
Then, 
\eqref{eq:ConCoe} implies with the \emph{Lax-Milgram Lemma}, e.g.,
\cite{erm_2004_finiteElement}, that for every $f\in L^2(\domain)$
the BVP \eqref{eq:BVP} admits a unique solution. 
Denote the solution operator by
\begin{equation}\label{eq:Gelliptic}
  \cG:=\begin{cases}
         \set{(a,f)\in\cX}{a\in L_{a_-}^\infty(\domain)}\to\cY\\
         (a,f)\mapsto {y}.
         \end{cases}
\end{equation}
The map $\calG(.,.)$ thus defined is Lipschitz-continuous:
for fixed data $f_1,f_2\in L^2(\domain)$ 
and for $a_1,a_2\in L^\infty_{a_-}(\domain)$, %
there exist constants $c_a, c_f > 0$ 
such that (see, e.g.\ \cite{BachmNguy})
\be\label{eq:diffLip}
\| \cG(a_1,f_1) - \calG(a_2,f_2) \|_{\calY}  
\leq c_a \| a_1 - a_2 \|_{L^\infty(\domain)} + c_f \| f_1 - f_2 \|_{L^2(\domain)}
\;.
\ee

The solution locally even depends \emph{holomorphically} on {$(a,f)\in \calX$,}
see, e.g., \cite{CDS11,Cohen2016}:
\begin{remark} \label{rmk:HolSol}
  Holomorphic parameter dependence of ${y}$ on $a$ and $f$ means that
  the solution operator $\cG:(a,f)\mapsto {y}$ is complex Fr\'echet
  differentiable as a mapping between the Banach spaces 
  $\cX$ and $\cY$ in \eqref{eq:Gelliptic}. 
  It is a strong form of smoothness,
  that implies in particular infinite Fr\'echet (and G\^ateaux)
  differentiability. 
  For example, letting $a(t) = a_0+t a_1$,
  $a_0\in L_{a_-}^\infty(\domain)$ as in \eqref{eq:a+} and
  $\varphi\in \cY'$, it holds that
  $t\mapsto \varphi(\cG(a(t),f)):\C\to\C$ is complex differentiable in
  a vicinity of $0$, and thus is in particular $C^\infty$ in a neighbourhood
  of $0$. 
 This strong smoothness is a crucial property inherent
 to many (well-posed) PDEs, \cite{CDS11,Cohen2016}.
\end{remark}
{
\begin{remark} \label{rmk:RegSol}
Proving Neural Operator approximation rate bounds require, 
as we shall see below,
mapping properties of $\calG$ restricted to subspaces.
Specifically, 
that on a suitable subspace $\tilde{\calX}\subset \calX$, 
$\calG\mid_{\tilde{\calX}}$ belongs to a corresponding subspace $\tilde{\calY} \subset \calY$.
Specific choices of $\tilde{\calX}, \tilde{\calY}$ are problem-dependent.
For \eqref{eq:BVP} assume that $\domain$ is convex, and 
consider data from 
$\tilde{\calX}:= W^{1,\infty}(\domain)\cap L^\infty_{a_-}(\domain) \times L^2(\domain) \subset \calX$
(i.e. Lipschitz-continuous in $\domain$ diffusion coefficient $a$ and 
      square-integrable source-term $f$).
Then, 
by classical results on elliptic regularity, 
$$
\calG|_{\tilde{\calX}} \subset \tilde{\calY} := (H^2\cap H^1_0)(\domain).
$$
This regularity property will be in Section~\ref{sec:ConvRts} below 
a key ingredient in the convergence rate analysis of en- and decoders
in the architecture \eqref{eq:DgE} in Section~\ref{sec:representation}.
\end{remark}
}
\begin{theorem}
Let $a_0\in L_{a_-}^\infty(\domain)$ and $f_0\in H^{-1}(\domain)$. 
Then $\cG$ in \eqref{eq:Gelliptic} 
is Fr\'echet differentiable in a neighbourhood of $(a_0,f_0)$ %
in the complexified spaces $L^\infty(\domain)_\C\times L^2(\domain)_\C$ 
(see \cite{MunozCplxB} for complexification of real Banach spaces).
\end{theorem}

We only sketch the argument following \cite[Example 1.2.38]{JZDiss}, 
and refer to this reference for further details: 
Since $\cG$ is linear in $f$, we focus on $a\mapsto \cG(a,f)$. 
With the differential operator
\begin{equation*}
  A(a):=\begin{cases}
          H_0^1(\domain)\to H^{-1}(\domain)\\          
          {y}\mapsto -\nabla\cdot (a\nabla {y})
        \end{cases}
\end{equation*}
We have $A(a_0+a)=A(a_0)+A(a)$ and ${y}(a) = A(a)^{-1}(f)=(A(a_0)+A(a))^{-1}f$. 
By the stated unique weak solvability of \eqref{eq:BVP}, 
the map $A(a_0)\in L(H_0^1(\domain),H^{-1}(\domain))$ is an isomorphism. 
Using a Neumann series we thus have
\begin{equation*}
  {y}(a_0+a) = (A(a_0)+A(a))^{-1}f = \sum_{k\in\N_0}(-1)^k(A(a_0)^{-1}A(a))^k A(a_0)^{-1}f.
\end{equation*}
The $k$th term in this expansion is a $k$-multilinear map in $a$. 
It thus represents a Taylor expansion in Banach spaces of the solution ${y}$ around $a_0$. 
In particular the (symmetrized) $k$th term can, up to the constant $k!$, 
be identified with the $k$th (complex) Fr\'echet derivative of ${y}(a_0+a)$ with respect to $a$.

\begin{codebox}
For an illustration of the parametric diffusion problem for $n=1$, 
see notebook \texttt{0\_Diffusion.ipynb}. The PDE solution is approximated via the finite difference method.
\end{codebox}

\subsection{Finite-dimensional Initial Value Problems}
\label{sec:IVP}
Operator Networks have been proposed recently for learning solution
maps of evolution (initial-boundary value) problems. 
See, e.g.\  \cite{BKMM21,DeRSM24} and the reference there.
In these works, {numerical ``neural propagators'', i.e.,} 
operator networks which map e.g.\ initial data {at time $t=0$} to the
solution at time $t>0$ have been constructed by transferring a time-stepping PDE scheme
into a suitable neural operator architecture, a procedure which has
been also referred to as ``algorithm-unrolling'' (see Section \ref{sec:Unrol} ahead)
a terminology coined in \cite{monga2020algorithm}.
The operator-network concept allows, however, also fundamentally different
approaches to evolution problems.

We illustrate this for an initial-value ODE system in 
finite-dimensional, Euclidean state space $\mathbb{R}^n$.

Let the vector field
$\Phi: \mathbb{R}^n \to \mathbb{R}^n$ 
be globally Lipschitz-continuous, 
i.e.\ there exists a constant $L_\Phi>0$ such that 
\be\label{eq:GlLip}
\forall u,v \in \mathbb{R}^n: \;\; \| \Phi(u) - \Phi(v) \|_2 \leq L_\Phi \| u - v \|_2 \;.
\ee
Given $u_0\in \mathbb{R}^n$ and any time-horizon $0<T \leq \infty$, 
consider the initial-value problem (IVP): 
find $u \in C^1([0,T];\mathbb{R}^n)$ such that for $0<t\leq T$ holds
\be\label{eq:IVP} 
\dot{u}(t) = \Phi(u(t))\;,\quad u(0) = u_0 \;.
\ee
It is classical (see, e.g., \cite{WalterODE}) that 
under the global Lipschitz assumption \eqref{eq:GlLip},
the IVP \eqref{eq:IVP} admits a unique solution $(u(t))_{t\geq 0}$,
defined for all $t\geq 0$. 
We remark that under weaker, local Lipschitz conditions on $\Phi$, 
generally existence and uniqueness hold only up to 
a finite, maximal time-horizon $T$. 
All that follows holds verbatim for $0<t<T$ in that case.

\noindent
For fixed $t>0$, the \emph{initial-data-to-solution-map} 
\be\label{eq:D2S}
S(t): \mathbb{R}^n \to \mathbb{R}^n : u_0 \mapsto u(t) 
\ee
is Lipschitz-continuous in the uniform norm \cite[Chap.III.13.II]{WalterODE}:
for any compact set $K{{\subset}} \mathbb{R}^n $ of initial data, 
and for every $0<T<\infty$, there exists a constant $C(K,T,L_\Phi)$ 
such that 
$$
\forall u_0, \tilde{u}_0 \in K:    
\sup_{0<t<T} \| u(t) - \tilde{u}(t) \|_2 
=
\sup_{0<t<T} \| S(t)u_0 - S(t)\tilde{u}_0 \|_2 
\leq 
C(K,T) \| u_0 - \tilde{u}_0  \|_2 
\;.
$$
The global unique solvability of \eqref{eq:IVP} 
gives that
the nonlinear maps $S(t)$ exhibit the (compositional) 
\emph{semigroup property}
\be \label{eq:SG}
S(0) = I, \quad \mbox{and} \quad 
\forall t,t' \geq 0: \quad S(t+t') = S(t) \circ S(t') \;.
\ee
 
For vector fields $\Phi$ in \eqref{eq:IVP} which have, 
in addition to \eqref{eq:GlLip}, extra regularity, the 
map $S(t)$ in \eqref{eq:D2S} is more regular.
\begin{theorem}\label{thm:D2SODE}
Consider the ODE-IVP \eqref{eq:IVP} in $\mathbb{R}^n$.
Assume that the vector field $\Phi$ is globally Lipschitz \eqref{eq:GlLip}
and, in addition, that the vector field $\Phi$ is real-analytic 
on compact sets $K{{\subset}} \mathbb{R}^n$.

Then, for every $0<t<T$, 
the initial-data to solution map $S(t): u_0\mapsto u(t)$ for \eqref{eq:IVP}
is real-analytic, as a map from
$K{{\subset}} \mathbb{R}^n$ to $S(t)(K) {{\subset}} \mathbb{R}^n$.
\end{theorem}
\begin{proof}
The proof of this result follows from standard properties of 
ODE-IVPs \eqref{eq:IVP} with \eqref{eq:GlLip}: 
the Lipschitz-continuity of $S(t)$ shown in  
\cite[Chap.III.13.II]{WalterODE} implies that $S(t)(K)$ is compact.
We next consider the notation in \cite[Chap.III.12 \& 13]{WalterODE}: 
there, non-autonomous ODE-IVPs 
$\dot{y}(t): = \bmf(t,y,\lambda)$
are considered, with 
$\lambda \in K {{\subset}} \mathbb{R}^n$, i.e.\ in a compact set of parameters.
It is shown in \cite[Chap.III.13, Thm.~XI]{WalterODE} that
in the case that $\bmf$ is holomorphic in the variables $(t,y, \lambda)$, 
so is the solution's dependence on $\lambda$. 

To reduce the proof of Theorem~\ref{thm:D2SODE} to this result, 
{we observe that our assumption of real analyticity of $f$ implies 
existence of a complex analytic extension of $f$ to a 
(possibly small) open neighborhood $U \subset \mathbb{C}^n$ which contains $K$, 
see, e.g., \cite[Prop.~3.1.1]{LeblNotes}.
}
We set for $u_0 \in K$ and $t\in [0,T] =: J$, 
$ y(t) := u(t) - u_0 $, 
so that for all $\lambda = u_0 \in K$ holds 
$y(0) = \bsnul$, 
and $u(t) = y(t) + u_0$.

Also, 
$\dot{y}(t) = \dot{u}(t) = \Phi(u(t)) = \Phi(y(t) + u_0) =: \bmf(t,y,u_0)$.
The unique solvability of \eqref{eq:IVP} and 
the continuous dependence of solutions on the initial data $u_0\in K$
{(see \cite{WalterODE})}
imply that $u(t) \in S(t)(K) {{\subset}} \mathbb{R}^n$ 
for every $t\in J$.
The assumption that $\Phi$ be real-analytic on compact subsets of $\mathbb{R}^n$
implies, via the usual covering and power-series argument, that 
$\Phi$ admits a unique complex-analytic continuation (again denoted by $\Phi$)
to a open, complex neighborhood
$K_\mathbb{C} \subset \mathbb{C}^{n}$ 
of 
$\bigcup_{0\leq t \leq T} S(t)(K) \subset \mathbb{R}^n$.
It also implies that for every $u_0\in K$, the solution $t\mapsto u(t)$ 
is analytic w.r. to $t\in J$.

This gives that 
$\bmf(t,y,u_0) = \Phi(y(t) + u_0)$ 
is complex-analytic w.r. to $(t,y,u_0)$. 
Now \cite[Chap.III.13, Thm.~XI]{WalterODE} 
yields that the solution $y(t)$ depends holomorphically on $\lambda \in K$.
Due to $u(t) =  y(t) + u_0 = y(t) + \lambda$,
we infer that for $u_0\in K_\mathbb{C}$, 
the solution $u(t) = S(t)u_0$ depends holomorphically on $u_0$.
Upon restricting to $u_0\in K {{\subset}} \mathbb{R}^n$, 
we obtain the claimed analytic dependence of $u(t) = S(t)u_0$ on $u_0 \in K$.
\end{proof}

The initial-data to solution map $S(t)$ being an analytic 
function of (the components of) $u_0 \in K$,
allows {us} to deduce statements on approximation rates of $S(t)$ 
by finite-parametric surrogates, in particular of polynomial 
and neural network type. 

To state a {concrete} result,
we assume 
$$ K = [-1,1]^n$$.
It is classical that functions %
which are real-analytic for $t\in [-1,1]$
can be approximated at {an} exponential rate 
by polynomials of degree $p \geq 1$ 
in the uniform norm on $[-1,1]$ 
(see, e.g., \cite[12.4.7]{Davis75}).
It follows by tensorization that, for $t>0$ and $u_0\in K=[-1,1]^n$, 
the component {$j \in \{1,\dots,n\}$} of the map $S(t):K\to \mathbb{R}^n$ 
can be approximated in the uniform norm on $K$ 
by a polynomial $S_{kj}(t)(u_0)$ of degree $k$ (interpolating, e.g.\  
in a tensor-product grid of {\Tsch} - nodes of order $k$ in $K$)
with respect to the ``input'' $u_0\in K = [-1,1]^n$  
according to
$$
\sup_{u_0\in K} | (S(t)(u_0))_j - S_{kj}(t)(u_0) | \leq C \exp(-b k )
$$
where the constants $b$, $C>0$ depend on $t>0$, $L_\Phi$ and 
the modulus of the analytic continuation of 
$S(t)(\cdot)$ to $K_\mathbb{C}\subset \mathbb{C}^n$.
The maps $S_{kj}(t)(\cdot)$ are tensor product polynomials
of degree $k\geq 1$ in $K$, hence require $O(k^n)$ many parameters.

Approximating in this way each component $(S(t)(\cdot))_{j=1}^n$
of the data-to-solution map $S(t)(\cdot)$
yields a \emph{spectral surrogate of the solution map}
of the ODE-IVP \eqref{eq:IVP} with altogether $O(k^{n+1})$ degrees of freedom.
\begin{theorem}[Spectral Surrogate]\label{thm:StIVP}
For the autonomous ODE-IVP \eqref{eq:IVP} with globally Lipschitz $\Phi$ 
which is, in addition, 
componentwise real analytic on compact subsets of $\mathbb{R}^n$, 
{for arbitrary, fixed $t>0$,} 
the solution map {$S(t):u_0 \mapsto u(t)$}
on the set of initial data $K = [-1,1]^n$ 
can be approximated 
by a finite-parametric, \emph{spectral surrogate map} $S_k(t)$ 
which is, component-wise, 
a tensor product polynomial of degree $k\geq 1$, 
with worst-case error bound
\be\label{eq:StSpcErr}
\sup_{u_0\in K} \| S(t)(u_0) - S_k(t)(u_0) \|_\infty 
\leq 
C \exp(-b k) \;,
\ee
with constants $b,C> 0$ independent of $k$, 
but depending (exponentially, in general) on $t, L_\Phi$.

In terms of the number of parameters $N$ of $S_k$, 
it holds $N(k) = O(nk^{n})$. 
This implies in \eqref{eq:StSpcErr} that there exist constants $b',C>0$ 
such that 
\be\label{eq:StSpcN}
\sup_{u_0\in K} \| S(t)(u_0) - S_k(t)(u_0) \|_\infty
\leq
C \exp(-b' N^{1/(n+1)}) \;, \quad 0\leq t \leq T\;.
\ee
Here, $b'=b\exp(-(n+1)\log(n+1))$ tends to zero as $n\to \infty$ at least exponentially.
\end{theorem}
\begin{remark} \label{rmk:ODEIVP}
\begin{enumerate}
\item {\bf Curse of Dimensionality}:
The shrinkage of the constant $b'$ in \eqref{eq:StSpcN} as $n\to \infty$
and the decrease in the exponential rate $\exp(-b'N^{1/(n+1)})$ as $n\to \infty$
is a form of the CoD. 
It precludes direct transfer of Thm.~\ref{thm:StIVP} to 
infinite-dimensional sets of initial data and state spaces.
\emph{A large part of construction of operator surrogates in SciML 
      revolves around the handling $n\to \infty$ without the CoD.}
\item {\bf Constructive Surrogate}:
The proof of existence of the surrogate map $S_k$ is constructive: 
it requires solving (numerically, in parallel) 
$(k+1)^n$ many instances of the ODE-IVP \eqref{eq:IVP} 
in the unisolvent, for the tensor-product polynomials of degree $k$, 
(isotropic) tensor-product {\Tsch-grid} in $K=[-1,1]^n$.
\item {\bf Uniform w.r.t.\ $t\in [0,T]$ {S}urrogate}:
So far, we considered $\{S_k(t)\}_{k\geq 1}$ for $t \in (0,T]$ fixed.
In Theorem~\ref{thm:StIVP}, the dependence of $u(t)$ on $t\in [0,T]$ is analytic. 
This may be leveraged to construct 
exponentially consistent solution operator surrogates
$S_k := \{ S_k(t) : 0\leq t \leq T \}$, $k=1,2,...$, e.g.\  by 
\Tsch-interpolation of the $S_k(t)$ in the \Tsch-points of order $k$ in $[0,T]$. 
\item  {\bf Work Bound}:
The total work for determining, 
{for arbitrary, fixed $t \in (0,T] $, 
the surrogate map $S_k(t)$ with accuracy $\eps\in (0,1]$, i.e.
$$
\sup_{u_0\in K} \| u(t) - S_k(t)u_0 \|_{\infty} \leq \eps\;,
$$
}
is bounded a-priori by $O(1+\log(\eps)^b)$ for some constant {$b(n)>0$.
Here, $b(n) = O(n)$ as $n$ increases.}
This follows upon training with approximate solutions of the IVP \eqref{eq:IVP}.
Accuracy of training data computed, e.g., by an initial-value ODE solver, 
in $[0,T]$ is to be consistent with the target accuracy \eqref{eq:StSpcN}.
{The exponential surrogate-consistency shown e.g. in Theorem~\ref{thm:NNStIVP} 
needs to be ensured for the data used in the ``training'' process. 
Generating training trajectories with standard, 
fixed order ODE solvers such as Runge-Kutta methods,
which converge at finite, algebraic order in terms of work vs.\ accuracy,
even for analytic $\Phi$, is possible to this end. However, ensuring 
possibly high fidelity $\eps$ of training data will entail 
possibly very large, costly ``offline'' computations. 
Once the (small, in terms of NN size vs. error $\eps$) 
surrogate maps $S_k(t)$ have been determined 
the predictive evaluation of $S_k(t)$ for specific initial data is considerably faster
than ``simulation from scratch'': 
this qualitative property of operator surrogates is one principal ``raison d'\^{e}tre'' 
of these surrogates. 
Similar comments apply to MOR approaches (see, e.g., \cite{RegaPagaAQ22} and 
the references there).
}

Spectral approaches to nonlinear IVPs such as those 
developed in \cite{SS2000}, will furnish such 
training data, however, 
in polylog w.r.\ to $\eps$ work vs.\ accuracy.
\end{enumerate}
\end{remark}
With the spectral surrogates $S_k(t)$ of the nonlinear 
initial-data to solution map $S(t)$ at hand, there arises the
question on other ways to approximate $S(t)$. For SciML, 
particular interest is in \emph{deep neural network surrogates}.

A well-known result, due to Telgarsky and Yarotsky
\cite{telgarsky_depth,Yarotsky16} states that ReLU NNs can emulate
polynomials of degree $k$ in $[-1,1]$ at exponential rate in the
uniform norm in terms of the number of %
trainable network parameters\footnote{{By ``\emph{trainable
        parameters}'' of a parametrized function class such as neural
      networks of a fixed given architecture, we mean the number of
      parameters (e.g., network weights and biases) that are adapted
      during the training process. This number can be interpreted as a
      measure of the complexity of the parametric function class.}}
  See e.g.\ \cite{OpChSII} and the references there for details.  From this
  result, with a tensor-product construction (suitably modified, as
  ReLU NNs do not allow a mathematically exact multiplication of two
  real numbers), with the triangle inequality
  and Theorem~\ref{thm:StIVP} 
follows a corresponding existence result for ReLU NN surrogate maps.
{As is common, we refer in the following to the size of the network
as the number of its trainable parameters, 
e.g., \cite[Chapter 2]{petersen2025mathematicaltheorydeeplearning}.}
\begin{theorem}\label{thm:NNStIVP}
Under the assumptions of Theorem~\ref{thm:StIVP},
for every $t>0$ and for every $\eps \in (0,1/2 ] $,
there exists a fully connected, feedforward ReLU NN 
surrogate map $\tilde{S}_k(t)$ 
with $n$ input and output channels such that 
$$
\| S(t) - \tilde{S}_k(t) \|_{W^{1,\infty}([-1,1]^n,\mathbb{R}^n)}
\leq C\eps \;.
$$
The size of the NN $\tilde{S}_k(t)$ 
measured in terms of the number $N(\eps)$ 
of neurons equals $O(|\log(\eps)|^{n+1})$.
\end{theorem}
\begin{proof}
We fix $t>0$ and recall that 
the map $S(t): K\to \mathbb{R}^n$ 
is real analytic, component-wise. 
It hence admits a holomorphic extension to a neighborhood 
$K_{\mathbb C}$ of $K$ in $\mathbb{C}^n$. 
Therefore, 
there exists a poly-ellipse $\calE_\rho \subset {\mathbb C}^n$ 
such that 
$$
K= [-1,1]^n \subset \calE_\rho \subset K_{\mathbb C} \subset {\mathbb C}^n 
$$
with strict inclusions and %
$\rho>1$  such that  the map $S(t)$ is holomorphic in $\overline{\calE_\rho}$, 
component-wise. 

Applying \cite[Theorem~3.6]{OSZ22} to emulate each component of $S(t)$, 
we obtain that 
$$
\| S(t) - \tilde{S}_k(t) \|_{W^{1,\infty}([-1,1]^n,\mathbb{R}^n)}
\leq 
C \exp(-b N^{1/(n+1)})\;,
$$
{with constants $b = b(\rho, n)>0$ and $C = C(S, \rho, n)$.}
{By  \cite[Theorem~3.6]{OSZ22} it further holds that}
$$
{\rm size}(\tilde{S}_k(t)) \leq %
N \;,
\quad 
{\rm depth}(\tilde{S}_k(t)) \leq C N^{\frac{1}{n+1}} \log(N),
$$
{where the size refers again to the number of trainable parameters, and depth refers to the depth of the neural network $\tilde S_k$.}
{We then obtain the result by choosing $N$ such that $C \exp(-b N^{1/(n+1)}) \le \eps$.}
\end{proof}
 
The preceding example featured key ideas in the construction of neural operators:
a mathematically well-defined map ($S(t)$ in \eqref{eq:D2S}) is replaced
by a hierarchical sequence of finite-parametric, computable surrogate 
maps ($S_k(t)$ in Theorem~\ref{thm:StIVP}) 
with convergence guarantees (the error bounds \eqref{eq:StSpcErr}, \eqref{eq:StSpcN}).
They also indicate several so-called 
\emph{approximation architectures} (spectral- and neural) for the construction
of finite-parametric surrogates.
The surrogates $S_k$ are constructed for a fixed time-hori{zon} $T$,
but embody in principle full knowledge on the evolution of solutions,
for all times $t>0$, on ground of a) their convergence, as $k\to \infty$, to the 
exact evolution map, and b) the semigroup property \eqref{eq:SG} which 
could be mimicked in order to 
\emph{obtain evolution maps for arbitrary multiples of $T$ by composition}
of $S_k$.
{Naturally, the error of such compositions may accumulate,
and finding stable frameworks for long-time simulation remains an active research area.} 
In this connection, it is important that both, polynomials
and deep neural networks, are closed under composition, if we allow
increases in polynomial degree {and} neural network depth{, respectively}.

In the preceding example, both inputs and output were from possibly 
high-, but finite-dimensional Euclidean space.
Once partial differential equations (PDEs) are considered,
infinite-dimensional in- and outputs must be admitted,
as input and output spaces of PDEs are function spaces.
This mandates extra building blocks in the construction of surrogates:
on the input side, \emph{encoders} which sample function space input
data, and reduce it to finite-parametric signals (typically, finite 
sequences of expansion coefficie{nts}, or ``signals'') which, subsequently,
are processed by a parametric surrogate map of the type of $S_k$ 
considered above. Finally, each output vector provided by $S_k$ must,
in turn, be decoded into a corresponding element of the output (function)
space, by a decoder block. This naturally leads to the so-called 
\emph{encoder-approximator-decoder} architectures, which are shared by 
a large number of neural operator which have been proposed recently.
We conclude discussion of this simple example with the remark, 
that the ODE initial-data-to-solution surrogate can be used 
as a building block in operator networks for the solution of 
linear transport equations {with non-constant advection field} 
as considered in \cite{LaakPhP21}.
For analytic transport fields, Theorem~\ref{thm:NNStIVP} again
implies exponential expressivity.
{
Alternative operator network constructions for linear, hyperbolic evolution PDEs
leverage NN approximations in mathematical solution representations,
e.g. in linear PDEs via fundamental solutions or superpositions of waves.
See, e.g., \cite{cho2025learninglowrankneural,RimWelperLRNR} and the references there.
}

\begin{codebox}
As an illustrative example of an IVP we consider the Lotka-Volterra equations in notebook \texttt{6\_IVP.ipynb}. 
These equations model the dynamics of a population of two species in a predator-prey relationship. 
The notebook further provides experiments in which the mapping of an initial state 
to the state at a time $T$ is approximated with a polynomial interpolant and with a feedforward network.
\end{codebox}
{
\subsection{Parabolic Initial-Boundary Value Problems}
\label{sec:ParIVP}
Operator Learning within the present, abstract framework is also
possible for evolution PDEs. Consider the nonstationary version
of the diffusion equation \eqref{eq:BVP}. 
With the notation of Section~\ref{sec:RD}, we are additionally 
given a finite time-horizon $0<T<\infty$ and a space-time 
cylinder $(0,T)\times \domain$.
We assume given a (time-independent, for simplicity) diffusion
coefficient $a\in L^\infty_{a_-}(\domain)$ as in \eqref{eq:a+},
initial data ${y}_0 \in H:= L^2(\domain)$, 
and a time-dependent source $f(\spatialvar,t)$ with $f\in L^2(0,T;H)$.

Then the 
\emph{linear, parabolic initial-boundary value problem}
reads: find $y(\eta,t)$ such that 
\begin{equation}\label{eq:IBVP}
  \begin{aligned}
    \partial_t {y} - \nabla_\spatialvar \cdot(a(\spatialvar)\nabla_\spatialvar {y}) &= f %
                                                                                                                 &&\qquad\text{in } (0,T)\times \domain\;,\\
    {y}(\cdot,t)|_{\partial\domain} & = 0 %
                                                                                                   &&\qquad\text{for }\;\; 0<t\leq T \;,\\
    {y}(\cdot,0) & =  y_0 %
                                                                                                   &&\qquad\text{in the sense of }
                   \; H \;.
                       \end{aligned}
\end{equation}
For fixed, sufficiently regular source term $f$ and for fixed ${y}_0\in H$,
the evolution equation \eqref{eq:IBVP} is well-posed, for 
$a\in  L^\infty_{a_-}(\domain)\subset \calX = L^\infty(\domain)$,
with the solution space $\calY$ now given by the 
\emph{intersection-type Bochner space}
\begin{equation}\label{eq:xtSpace}
\calY = L^2(0,T;H^1_0(\domain)) \cap H^1(0,T;H^{-1}(\domain)) \;.
\end{equation}
We refer to \cite[Sec.~5]{ChSRobSt2009} and the references there 
for the proof of unique solvability and 
to \cite[Sec.~2.3]{harbrecht2026neuralshapeoperatorsurrogates}
for a proof of analyticity of the coefficient-to-solution map 
$\calG:  L^\infty_{a_-}(\domain) \to \calY: a\mapsto {y}$.

Corresponding space-time formulations for nonlinear evolution equations
are found in \cite{ChSRSt17}.
}
\section{Representation Systems}
\label{sec:representation}
In practice, any computable\footnote{
We assume here \emph{exact numerical computation}. 
However,
practical DL algorithms realize NN weights and biases in float{ing} point arithmetic, 
often in mixed or reduced precision, e.g.\  in connection with \emph{NN quantization}.
Low floating point precision in NN parameters can be compensated by increased 
depth, see \cite[Section~4]{ou2024quantizationregimesrelunetworks}.
In quantum DL, NN weights and biases are represented via {qubits} (see, e.g.\  \cite{Ghosh_2019,GonJacQDL25} 
and the references there).
}
approximation $\hat \cG$ of $\cG:\cX\to\cY$ requires discretization, 
as a digital computer cannot directly handle infinite dimensional objects,
such as elements in function spaces.
The first discretization concerns the operator input in $\cX$ and performs a dimension reduction.
Input discretization is referred to as the {\bf encoder}, $\cE$. 
The encoder is a mapping from
\begin{equation*}
  \cE:\cX\to\R^{d_\calX},
\end{equation*}
for some $d_\calX \in \N$. 
The encoded input is subsequently transformed by an {\bf approximator} or {\bf coefficient map}
\begin{equation*}
  g:\R^{d_\calX}\to \R^{d_\calY},
\end{equation*}
for some $d_\calY \in\N$. 
To obtain an element of the output space $\cY$, we finally apply a {\bf decoder}
\begin{equation*}
  \cD:\R^{d_\calY} \to \cY.
\end{equation*}

Any operator approximation scheme which we consider here is 
then of the generic ``Encoder-Approximator-Decoder'' form
\begin{equation}\label{eq:DgE}
  \hat \cG = \  \cD \circ g \circ \cE\simeq \cG.
\end{equation}
For particular architectures, 
the choice of encoder $\cE$ and decoder $\cD$ plays a central role, 
in particular with respect to 
how the inputs (from function space) are accessed (via the encoder $\cE$), 
and 
how the outputs (into function space) are represented (via the decoder $\cD$).
For other architectures, 
$\cE$ and $\cD$ 
are only given implicitly and may not even be explicitly mentioned in the description of the architecture.
Moreover, the precise boundary between what is considered part of the encoder or decoder 
and what is part of $g$ is generally not canonical and to some extent subject to interpretation.

If $\cX$ and $\cY$ are function spaces, 
the encoder $\cE$ typically evaluates a function at a finite set of sampling points, 
possibly followed by some form of transformation {such as, e.g.,}
{a linear or nonlinear dimension reduction.} %
The decoder $\cD$ usually either maps a finite set of function values 
to some form of interpolant (e.g., a piecewise multilinear interpolant), 
or takes coefficients in a chosen representation system, and maps them via 
a linear or nonlinear transformation to an element of $\cY$. 
The encoder and decoder may be fixed a priori or learned from data. 
The coefficient map $g$ is typically learned.
In this and the following section we discuss some concrete examples.

\subsection{Stable linear representations}\label{sec:stable}
A popular class of encoder/decoder maps in scientific machine learning 
is made up by bounded linear maps. 
In this case
\begin{equation*}
  \cE:=\begin{cases}
         \cX\to\R^{d_\cX}\\
         a\mapsto (\varphi_j(a))_{j=1}^{d_\cX}
       \end{cases}
       \qquad\text{and}\qquad
  \cD:=\begin{cases}
         \R^{d_\cY}\to \cY\\
         (c_j)_{j=1}^{d_\cY}\mapsto \sum_{j=1}^{d_\cY} c_j {\psi}_j,
       \end{cases}
\end{equation*}
with $\varphi_j\in \cX'$ and $\psi_j\in \cY$ for all $j$.

Encoder and decoder sequences 
must be chosen carefully to capture the essential features of the operator's input and output: 
any information in the kernel of $\cE$ is lost, and anything outside the range of $\cD$ 
cannot be approximated by~\eqref{eq:DgE}. 
It is thus clear that \eqref{eq:DgE} can only provide a good approximation if $\ker(\cE)$ 
spans a subspace of input features that are of little importance, 
while ${\rm im}(\cD)$ describes the most important features of the output. 
We therefore wish to understand which functionals $\varphi_j$ 
provide the most information to reconstruct $a$ from $(\varphi_j(a))_{j=1}^m$. 
To make this mathematically concrete, 
we introduce the concept of \emph{Banach frames}. 
We opt for the definition in \cite[Section 1]{KHGroech}, 
and \emph{assume the Banach spaces $\calX$ and $\calY$ 
to be isomorphic to (weighted) $\ell^p$ sequence space}. 
This includes a wide range of separable function spaces of 
Sobolev- and Besov-Triebel-Lizorkin type, see 
\cite{TriebelWavelets2008,TriebelBases2010} and the references there.

\begin{definition}\label{def:frame}
  Let $\cZ$ be a separable Banach space\footnote{We use the generic notation $\cZ$ 
  when introducing concepts relevant both for the input space $\cX$ and the output space $\cY$.} and let $p\in [1,\infty)$.

  A family $(\varphi_{\cZ,j})_{j\in\N}$ of functionals in $\cZ'$
  and elements $(\psi_{\cZ,j})_{j\in\N}$ in $\cZ$
  is called a {\bf $p$-Banach frame} for $\cZ$, if
  \begin{enumerate}
  \item the {\bf analysis operator}
    \begin{align*}
      A_\cZ:&\cZ\to\ell^p(\N):\;
      z \mapsto (\varphi_{\cZ,j}(z))_{j\in\N}
    \end{align*}
    is continuous,
  \item\label{item:frconst} the norms $\norm[\cZ]{z}$ and $\norm[\ell^p]{(\varphi_{\cZ,j}(z))_{j\in\N}}$ 
   are equivalent on $\cZ$,
  \item\label{item:synthesis} the {\bf synthesis operator}
    \begin{align*}
      S_\cZ:&\ell^p(\N)\to \cZ : \;
      (c_j)_{j\in\N}\mapsto \sum_{j\in\N}c_j\psi_{\cZ,j}
    \end{align*}
    is a bounded linear map satisfying
    \begin{equation}\label{eq:SAid}
      S_\cZ(A_\cZ(z))=z
    \end{equation} 
      for all $z\in\cZ$.
  \end{enumerate}
\end{definition}
\emph{Linear Encoders} $\calE_\cZ$ can be obtained 
by combining $A_\cZ$ with $d_\cZ$-term truncation to the first
$d_\cZ$ coefficients of the output of $A_\cZ$. 
\emph{Linear Decoders} $\calD_\cZ$ can be obtained 
by padding a finite sequence in $\R^{d_\cZ}$ 
with zeros to an element of $\ell^p(\N)$, 
and then applying $S_\cZ$.
\begin{remark}\label{rmk:hilbert}
  If $p=2$ and $\cZ$ is a separable Hilbert space, then
  (i) and (ii) imply the existence of $(\psi_{\cZ,j})_{j\in\N}$ satisfying (iii), 
  see e.g.\  \cite{Heil}.
\end{remark}

For a $p$-Banach frame $(\varphi_{\cX,j},\psi_{\cX,j})_{j\in\N}$ on $\cX$, 
{one can estimate the information lost when working with the 
  $d$-term truncations of (sequences of) coefficients $(\varphi_{\cX,j}(a))_{j=1}^d$
  as follows: Motivated by \eqref{eq:SAid}, we define a \emph{reconstruction map} $R$ via
\begin{equation*}
  R:=\begin{cases}
                 \R^d\to \cZ\\
                 (c_j)_{j=1}^d\mapsto \sum_{j=1}^d c_j\psi_{\cZ,j}.
               \end{cases}
\end{equation*}
Then Definition \ref{def:frame} \ref{item:synthesis} implies the existence of a constant $C$ such that
for all $a\in\cX$
\begin{align*}
  \norm[\cX]{a-R(\varphi_{\cX,j}(a))}&=
  \norm[\cX]{S_{\cX}(A_{\cX}(a))-S_{\cX}((\varphi_{\cX,1}(a),\dots,\varphi_{\cX,d}(a),0,0,\dots)}\\
  &\le C \norm[\ell^p]{(\varphi_{\cX,j}(a))_{j>d}}.
\end{align*}
}

{This provides a rough, first} guiding criterion for determining how many coefficients $d$
are needed\footnote{This also depends on the particular way in which the basis/frame elements are enumerated.}
to retain sufficient information about the operator input $a\in\cX$. 
Analogous considerations hold for the representation of the operator output $y$ via a $p$-Banach frame $(\varphi_{\cY,j},\psi_{\cY,j})_{j\in\N}$ in $\cY$. 
We will discuss this in more detail in Section~\ref{sec:Arch}. 
This motivates the en- and decoder choice
\begin{equation}\label{eq:EDframe}
  \cE(a)=(\varphi_{\cX,j}(a))_{j=1}^{d_\calX} \qquad\text{and}\qquad
  \cD((c_j)_{j=1}^{d_\calY})=\sum_{j=1}^{d_\calY} c_j\psi_{\cY,j}.
\end{equation}
The appeal of using Banach frames in operator learning lies in the continuity of both the coefficient-to-function map and its inverse, which ensures stability in either direction.

\begin{example}\label{ex:frames}
  Typical examples include the following:
  \begin{itemize}
  \item \textbf{Orthonormal bases in Hilbert spaces:}  
    Suppose that $(\psi_j)_{j\in\N}$ is an orthonormal basis of the separable Hilbert space
    $\cX$. Using Parseval's identity, it is straightforward to check that
    \begin{equation*}
      \varphi_j(a) \dfn \langle a, \psi_j \rangle_{\cX} \qquad \forall j\in\N
    \end{equation*}
    forms a $2$-Banach frame for $\cX$.

  \item \textbf{Frames in Hilbert spaces:}
More generally, let $(\psi_j)_{j\in\N}$ be a \textbf{frame for a Hilbert space} $\cX$ \cite{Heil}, 
meaning there exist \emph{frame bounds} $\lambda$, $\Lambda > 0$ such that
 \begin{equation}\label{eq:lambdaLambda}
     \lambda \norm[\cX]{a}^2 \le \sum_{j\in\N} |\dup{a}{\psi_j}_\cX|^2 \le \Lambda \norm[\cX]{a}^2\qquad \forall a\in\cX.
 \end{equation}
Setting $\varphi_j(a) \dfn \dup{a}{\psi_j}_\cX$, 
it is again easy to see that \eqref{eq:lambdaLambda} 
guarantees the conditions of Definition~\ref{def:frame} for $p=2$. 
  \item \textbf{Besov-Triebel-Lizorkin spaces:}  
    Let $\cX = B^s_{p,q}(\domain)$ (Besov space) or $\cX = F^s_{p,q}(\domain)$ (Triebel-Lizorkin space) for suitable parameters $(s,p,q)$. 
    For such spaces there exists (on suitable domains $D$, 
    see \cite{TriebelWavelets2008,TriebelBases2010} for details) 
    a countable system of functions $(\psi_j)_{j\in\N}$ (e.g., wavelets) 
    such that any $a\in\cX$ can be represented as
    \begin{equation*}
        a = \sum_{j\in\N} c_j \psi_j,
    \end{equation*}
    and the norm $\norm[\cX]{a}$ is equivalent to a 
    (possibly weighted) $\ell^p$-norm of the coefficient sequence $(c_j)_{j\in\N}$. 
    This isomorphism between the function space 
    and a sequence space provides the structure of a Banach frame. 
    See \cite[Corollary 5.44]{TriebelWavelets2008} for more details.

\item \textbf{Biorthogonal spline-wavelet bases of hilbertian Sobolev spaces in domains:}
     These systems are Riesz-bases and, formally, special cases of frames. 
     Concrete constructions of such systems have been developed over the past decade, we 
     refer to \cite{RekRSt18,DKU} and references there for examples. 
     Figure~\ref{fig:wavelet} shows a simple biorthogonal spline wavelet system 
     in the univariate domain $D=(0,1)$, which constitutes a Riesz basis 
     for the space $\calX = \{v \in H^1(\domain)\mid v(0) = 0 \}$.
     {We refer to \cite[Sec.~6]{ChSRobSt2009} for a construction of space-time spline-wavelet
         systems for the Bochner space $\calY$ in \eqref{eq:xtSpace} for 
         parabolic initial-boundary value problems such as \eqref{eq:IBVP}.
        }
  \end{itemize}
\end{example}
\begin{figure}
\begin{center}
\includegraphics[width=0.75\textwidth]{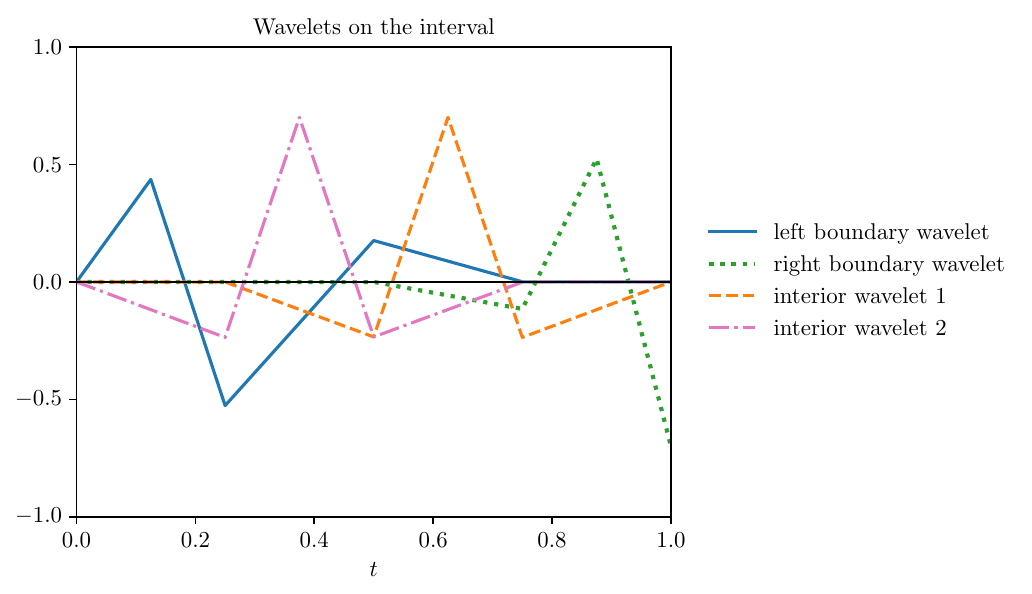}
\caption{Continuous, piecewise linear wavelet basis of 
$\calX = H^1_{\{0\}}(0,1)$ on subdivision level $j=2$ in $I=(0,1)$ 
with two vanishing moments and a homogeneous Dirichlet boundary condition in $t=0$
(imposed only on the left boundary wavelet).
See the appendix of \cite{HaScMZ24} for details and 
\cite{Urban09,DKU,HoangStevFEWav09} for other, {spline- and Finite-Element like} constructions.
\label{fig:wavelet}
}
\end{center}
\end{figure}

\begin{codebox}
We set up an encoder $\cE$ for the parametric diffusion problem 
based on a priori assumption on the input distribution in notebook \texttt{1\_Representation.ipynb}.
Specifically, we assume that the log-normal diffusion coefficient $\log a$ 
can be expressed  in the orthonormal sine basis $\psi_j(x) = \sqrt{2}\sin(j\pi x)$ of $L^2(0,1)$,
i.e., $\log a(x;\, y) = \sum_{j=1}^{d} y_j \, \psi_j(x),$ 
where $y = (y_1,\ldots,y_d) \in \R^d$ is the parameter vector.
The encoder then maps a diffusion coefficient $a$ to (a truncated version of) the parameter vector.
The log-normal form guarantees $a > 0$ pointwise, as required by \eqref{eq:a+},
and the basis $(\psi_j)_{j \geq 1}$ is an instance of the orthonormal basis frame in Example~\ref{ex:frames}.
\end{codebox}

\subsection{PCA}\label{sec:PCA}
The approach in Section~\ref{sec:stable} is sensible when Banach frames of $\cX$ and $\cY$ are {known a priori} and constructable. However, in many applications the inputs of $\cG$ 
 {are assumed to be i.i.d.\ realizations of a random variable taking values in $\cX$ 
     with i.i.d.\ understood with respect to a {probability} measure $\mu$ on $\cX$.}
{The support of a finite sample of input data for $\cG$} 
may occupy only a small subset of $\cX$. 
Similarly, the outputs follow the pushforward measure\footnote{{For a
  measure space $(\cX,\CB,\mu)$, a measurable space
  $(\Omega,\CA)$, and a measurable map $T:\cX\to\Omega$,
  we denote by $T_\sharp\mu$ the so-called \emph{pushforward measure} on $(\Omega,\CA)$;
  it is defined as $(T_\sharp\mu)(A):=\mu(\set{x\in \cX}{T(x)\in A})=\mu(T^{-1}(A))$ for all $A\in\CA$.}}
$\cG_\sharp\mu$, which, 
depending on $\cG$, may be supported on a small subset of $\cY$. 
In these cases, even if Banach frames of $\cX$ and $\cY$ are available, 
they may not yield the most efficient representation systems.

\subsubsection{Formulation}
{For efficient compression, we aim to determine a}
best approximating $d$-dimensional subspace of a data distribution $\mu$ %
on a space $\cZ$. 
{Mathematically, this can be formulated as an optimization problem of the type}
\begin{equation}\label{eq:pcaobj}
  \inf_{\substack{V\subseteq \cZ,~{\rm dim}(V)=d\\ z_0\in \cZ}} \bbE_{z\sim \mu}[\inf_{v\in V}\norm[\cZ]{z-(z_0+v)}^p],
\end{equation}
for a suitable {$p\ge 1$ (and assuming finite $p$th moment of $\mu$).} 
We are thus looking for the best approximating $d$-dimensional affine subspace $z_0+V$ of $\cZ$. 
In general, this optimization problem need not have a unique solution, and even if it does, 
it might be difficult to find in practice. 
In the following we thus only focus on 
{the case in which} $p=2$ and $\cZ$ {is} a Hilbert space;
{This is known as functional PCA, and} in this case the solution can be described explicitly.
{We comment on the usual PCA in finite dimensional spaces in Remark \ref{rmk:PCAfinite}.}

To explain how, we assume %
$\mu$ is a
probability measure on $\cZ$ with finite first and second moments, i.e.\
\begin{equation}\label{eq:muconditions}
    \int_{\cZ}\norm[\cZ]{z}\dd\mu(z)<\infty,\qquad \int_{\cZ}\norm[\cZ]{z}^2\dd\mu(z)<\infty.
\end{equation}
For a $d$-dimensional subspace $V\subseteq\cZ$ and with the orthogonal projection $P_V:\cZ\to V$ it holds for any $z\in\cZ$ \begin{equation*}
  \norm[\cZ]{z-P_Vz}=\inf_{v\in V}\norm[\cZ]{z-v}.
\end{equation*}
Therefore, the objective in \eqref{eq:pcaobj} corresponds to
\begin{equation*}
  \bbE_{z\sim \mu}\norm[\cZ]{z-z_0-P_V(z-z_0)}^2
  =
  \bbE_{z\sim \mu}\norm[\cZ]{(z-P_Vz) - (z_0-P_Vz_0)}^2.
\end{equation*}
The element in $\cZ$ best approximating the random variable 
$\tilde z:=z-P_Vz$ in the least-squares sense 
is $\bbE[\tilde z]=\bbE[z]-\bbE[P_Vz]=\bbE[z]-P_V\bbE[z]$ 
(see Lemma \ref{lemma:averagebest} ahead). 
For this reason, we fix %
\begin{equation}\label{eq:z0}
  z_0 = \bbE_{z\sim\mu}[z] = \int_{\cZ}z\dd\mu(z).
\end{equation}
By Pythagoras' theorem
\begin{equation*}
  \norm[\cZ]{z-P_Vz}^2 = %
  \norm[\cZ]{z}^2-\norm[\cZ]{P_Vz}^2.
\end{equation*}
The optimization problem \eqref{eq:pcaobj} for $p=2$ can now be simplified to
\begin{equation*}
  \inf_{V\subseteq \cZ,~{\rm dim}(V)=d} \bbE_{z\sim \mu}[\norm[\cZ]{(z-z_0)-P_V(z-z_0)}^2]
  =
  \inf_{V\subseteq \cZ,~{\rm dim}(V)=d} \Big(\bbE_{z\sim \mu}[\norm[\cZ]{z-z_0}^2]-
  \bbE_{z\sim \mu}[\norm[\cZ]{P_V(z-z_0)}^2]\Big).
\end{equation*}
Thus, $V$ is equivalently characterized as the subspace \emph{maximizing}
\begin{equation}\label{eq:pcamaxobj}
  \bbE_{z\sim \mu}[\norm[\cZ]{P_V(z-z_0)}^2].
\end{equation}
For this reason, PCA is often described as finding the linear subspace $V$ 
on which the \emph{data projection onto $V$ has maximal variance}.
\subsubsection{Characterization of $V$}
\label{sec:CharV}
Introduce the covariance operator $\covar:\cZ\to\cZ$ of $\mu$ via
\begin{equation}\label{eq:covariance}
\dup{\covar x}{y}_\cZ=\int_{\cZ}\dup{x}{z-z_0}_\cZ\dup{y}{z-z_0}_\cZ\dd\mu(z) \qquad \forall x,y \in \cZ.
\end{equation}
As is well-known, under assumption \eqref{eq:muconditions}, 
$\covar:\cZ\to\cZ$ as in \eqref{eq:covariance} is a well-defined, 
compact, and self-adjoint linear operator, 
see, e.g.\ \cite{StSc2012} and references there. 

\begin{proposition}\label{prop:PCA}
Let $\cZ$ be a Hilbert space {and} let $\mu$ be a probability measure on $\cZ$ 
with finite first and second moments. 
Then, for any $d\in\mathbb{N}$,
a $d$-dimensional subspace $V \subseteq \cZ$
maximizes \eqref{eq:pcamaxobj} 
if and only if $V$ is a span of eigenvectors of $\covar$ corresponding to the $d$ largest eigenvalues.
\end{proposition}
\begin{proof}
Since $\covar$ is self-adjoint and compact, 
there exists a null-sequence of monotonically decreasing, 
non-negative eigenvalues $(\lambda_j)_{j\in\N}$ 
and corresponding orthonormal eigenvectors $(\psi_j)_{j\in\N}$ in $\cZ$, 
such that $\covar$ has the spectral decomposition
  \begin{equation*}
    \covar z=\sum_{j\in\N}\lambda_j \dup{\psi_j}{z}_\cZ\psi_j\qquad\forall z\in\cZ.
  \end{equation*}

Let $V$ be an arbitrary $d$-dimensional subspace of $\cZ$,
and fix an orthonormal basis $e_1,\dots,e_d$ of $V$. Then
\begin{equation*}
    \norm[\cZ]{P_Vz}^2 = \sum_{i=1}^d \dup{z}{e_i}^2\qquad\forall z\in\cZ,
\end{equation*}
and thus
\begin{equation*}
    \bbE_{z\sim \mu}[\norm[\cZ]{P_V(z-z_0)}^2]
    =\int_{\cZ}\sum_{{i=1}}^d \dup{z-z_0}{e_i}^2\dd\mu(z)
    =\sum_{i=1}^d \dup{\covar e_i}{e_i}
    =\sum_{j\in\N}\lambda_j\underbrace{\sum_{i=1}^d \dup{\psi_j}{e_i}^2}_{=:c_j}.
\end{equation*}
Due to the orthonormality of the $(\psi_j)_{j\in\N}$ and the $(e_i)_{i=1}^d$, 
we have $0\le c_j\le 1$ for all $j\in\N$, and $\sum_{j\in\N}c_j\le d$. 
Assume for the moment that the $\lambda_j$ are \emph{strictly} monotonically decreasing, 
i.e.\ $\lambda_1>\lambda_2>\dots$. 
Then the sum $\sum_{j\in\N}\lambda_j c_j$ is maximized under the above constraints on the $(c_j)_{j\in\N}$ 
if and only if $c_1=\dots=c_d=1$ and $c_j=0$ for $j>d$. 
This is achieved if and only if $\psi_j\in{\rm span}\{e_1,\dots,e_d\}$ for all $j=1,\dots,d$.
In other words $V={\rm span}\{\psi_1,\dots,\psi_d\}$.

Finally, if the eigenvalues are not necessarily strictly decreasing, i.e.\ $\lambda_1\ge \lambda_2\ge\dots$, 
the optimal subspace is not necessarily unique. 
The characterization remains valid however: 
any subspace spanned by an orthonormal set of eigenvectors corresponding to the $d$ largest eigenvalues yields a {maximizer}.
\end{proof}

{
\begin{remark}\label{rmk:PCAfinite}
  Consider the special case of $\cZ=\R^n$ and
  $\mu=\frac{1}{m}\sum_{j=1}^m \delta_{z_j}$ being a discrete
  probability measure on $\R^n$ (here $\delta_{z_j}$ denotes the Dirac delta at
  $z_j\in\R^n$). In this case $\covar:\R^n\to\R^n$ is defined by its action
  on all $x$, $y\in\R^n$ via\footnote{If $z_1,\dots,z_m$, $m\ge 1$,
  are i.i.d.\ draws from some probability measure on $\R^n$, then the
  normalization $1/m$ in \eqref{eq:defCfinite} is typically replaced by
  $1/(m-1)$ to obtain an unbiased estimator.}
  \begin{equation}\label{eq:defCfinite}
    \dup{\covar x}{y}=\frac{1}{m}\sum_{j=1}^m x^\top(z_j-z_0) (z_j-z_0)^\top y,
  \end{equation}
  where $z_0=\frac{1}{m}\sum_{j=1}^m z_j\in\R^n$. Letting
  $A\in\R^{n\times m}$ be the matrix with columns $z_j-z_0$, the
  covariance operator corresponds to
  \begin{equation*}
    \frac{1}{m}\sum_{j=1}^m (z_j-z_0) (z_j-z_0)^\top = \frac{1}{m}A A^\top\in\R^{n\times n}.
  \end{equation*}
  The optimal $d$-dimensional subspace is then obtained by the span of
  the $d$ leading eigenvectors of $\covar$ (or equivalently of
  $AA^\top$). For numerical stability, in practice this can be done by computing
  the singular value decomposition of $A$.
\end{remark}
}

\begin{remark}[Error of PCA]\label{rmk:PCAerror}
Denote again by $\lambda_1\ge \lambda_2\ge\dots$ the eigenvalues of the covariance operator $\covar$. 
As in the preceding proof, 
it holds {that}
\begin{equation*}
  \bbE_{z\sim\mu}[\norm[\cZ]{z-z_0}^2] = \sum_{j\in\N}\lambda_j.
\end{equation*}
Together with the estimates shown in the proof this yields the error
\begin{equation*}
  \bbE_{z\sim\mu}[\norm[\cZ]{z-z_0-P_V(z-z_0)}^2] = \sum_{j>d}\lambda_j
\end{equation*}
for an optimal $d$-dimensional subspace $V$. %
\end{remark}

\subsubsection{Computation of Encoder and Decoder}
\label{sec:CompEncDec}
In the case of Hilbert spaces, 
``optimal'' (in the sense of minimizing the least squares objective \eqref{eq:pcaobj} for $p=2$) 
encoders and decoders can be found via PCA: 
let $x_0$ be the expected value of $\mu$ on $\cX$ and let $y_0$ the expected value of $\cG_\sharp\mu$ on $\cY$. 
Let the PCA yield the orthonormal vectors $(\psi_{\cX,j})_{j=1}^{d_\cX}$ and $(\psi_{\cY,j})_{j=1}^{d_\cY}$ in $\cX$, $\cY$, respectively. 
Then the corresponding encoder and decoder are given by
\begin{equation}\label{eq:EDPCA}
  \cE(x)=(\dup{\psi_{\cX,j}}{x-x_0})_{j=1}^{d_\cX}\qquad\text{and}\qquad
  \cD((c_j)_{j=1}^{d_\cY})=y_0+\sum_{j=1}^{d_\cY}c_j\psi_{\cY,j}.
\end{equation}

In practice, $\mu$ and $\cG_\sharp\mu$ might {be unknown} or intractable, but we have access to data sample{s} $(x_j,\cG(x_j))_{j=1}^m$ with $x_j\overset{{\rm iid}}{\sim}\mu$. 
In this case, we can still perform the empirical version of PCA. 
We outline the procedure for the input data:
\begin{enumerate}
\item {\bf Center the data.} Let $\bar x_j:=x_j-\bar x$, where $\bar x=\frac{1}{m}\sum_{j=1}^m x_j$ is the 
sample average.
\item {\bf Compute the empirical covariance operator.} Let
  \begin{equation*}
    \hat\covar x =\frac{1}{m-1}\sum_{j=1}^m\dup{x}{\bar x_j}_\cX \bar x_j.
  \end{equation*}
  Note that $\hat\covar:\cX\to {\rm span}\{\bar x_1,\dots,\bar x_m\}$. 
  For numerical computations, 
  $\hat\covar$ can be represented in a finite dimensional basis of ${\rm span}\{\bar x_1,\dots,\bar x_m\}$, yielding a matrix.
\item {\bf Solve the eigenvalue problem.}
  Compute the eigendecomposition of the symmetric matrix $\hat\covar$ {(cf.~Remark \ref{rmk:PCAfinite})}
\begin{equation*}
    \hat\covar \hat{\psi}_{\cX,j} = \hat{\lambda}_j \hat{\psi}_{\cX,j}.
\end{equation*}
\item \textbf{Select the subspace.} 
      Choose the $d$ (orthonormal) eigenvectors $(\hat{\psi}_{\cX,j})_{j=1}^d$ corresponding to the 
      $d$ largest eigenvalues $(\hat{\lambda}_j)_{j=1}^d$. 
      These vectors are used to define the encoder in \eqref{eq:EDPCA}.
\end{enumerate}

{
  \begin{remark}
It might seem easier to work with an orthonormal basis than a frame; in particular, frames can be redundant, meaning the expansion of an element of the Hilbert/Banach space in the frame is in general not unique. 
However, frames offer the crucial advantage of being computationally tractable. 
Finding an orthonormal basis of Sobolev spaces on standard domains like cubes is straightforward (e.g.,  Fourier modes); 
however, real-world problems are rarely posed on such simple domains. 
Finding an orthonormal basis for a Sobolev space on a complicated geometry is a nontrivial {numerical} 
computation in itself. 
Constructions for frames {on the other hand} are known and computable, based, for example, 
on triangulations of the domain, e.g., {\cite{DavydovStevenson2006,HSS08}}. 
They thus naturally integrate with existing numerical solvers and representations of numerical PDE solutions.
\end{remark}
  }

\begin{codebox}
We set up a PCA-based decoder $\cD$ for the parametric diffusion problem in notebook \texttt{1\_Representation.ipynb}.
\end{codebox}

\subsection{Active Subspaces}\label{sec:ActSubSpc}
Recall that our goal is to approximate $\cG:\cX\to\cY$, with inputs stemming from a measure $\mu$ on $\cX$. 
To perform a \emph{dimension reduction on the input space} $\cX$, 
PCA determines a best approximating $d$-dimensional subspace $V$ of $\cX$ with respect to $\mu$. 
This approach is entirely agnostic to $\cG$, however. Suppose for example that $\cG$ 
is constant along a direction $v\in \cX$, i.e.\ for every $x\in \cX$ the map $t\mapsto \cG(x+tv)$ 
is constant for $t\in\R$. 
Even if $v$ corresponds to an important eigendirection of the covariance of $\mu$, 
it should in this case not %
{be included in the reduced subspace $V$, as it provides no information about the output of $\cG$.}

In this section, we present basic results for the active subspace approach, see e.g.\  \cite{PC-ActSubspc}. 

\subsubsection{Formulation}
\label{sec:AcSpcForm}
Assume that $\cX$, $\cY$ are separable Hilbert spaces. 
To take into account the action of $\cG$,  
the active subspace approach replaces the objective \eqref{eq:pcaobj} by
\begin{equation}\label{eq:asobj}
  \inf_{\substack{V\subseteq \cX,~{\rm dim}(V)=d\\ \bar \cG_V:V\to\cY\text{ meas{urable}}}} \bbE_{x\sim \mu}[\norm[\cY]{\cG(x)-\bar\cG_V(P_Vx)}^2],
\end{equation}
where $P_V$ is again the orthogonal projection onto $V$. 
In words, we wish to determine a $d$-dimensional subspace $V\subseteq\cX$ 
which is \emph{best suited for the approximation of $\cG$ under the input measure $\mu$}.
\begin{remark}
For the objective \eqref{eq:asobj} it suffices to consider linear subspaces $V$ 
(rather than affine subspaces $z_0+V$ as in Section \ref{sec:PCA}): 
for a shift $z_0$, 
the expression $\bar \cG(z_0+P_V(z-z_0))$ corresponds to some $\tilde\cG(P_Vz)$. 
The shift is thus absorbed by the minimization over $\bar\cG_V$ in \eqref{eq:asobj}.
\end{remark}
The next lemma will allow to explicitly construct {the optimal $\bar\cG_V$ in \eqref{eq:asobj}.}
{The actual construction of $\bar\cG_V$ is developed following its proof.}

\begin{lemma}\label{lemma:averagebest}
Let $\cY$ be a separable Hilbert space, 
let $\gamma$ be a probability measure on a measure space $\cZ$, 
and let $\cF:\cZ\to\cY$ be measurable. 

Then
  \begin{equation}\label{eq:averagebest}
    \argmin_{y\in \cY}\norm[L^2(\cZ,\gamma)]{\cF(z)-y}=\int_\cZ \cF(z)\dd\gamma(z).
  \end{equation}
\end{lemma}
\begin{proof}
 Denote the right-hand side in \eqref{eq:averagebest} by $y$,
 and fix an arbitrary $v\in {\cY}$ with $v\neq 0$.
 Then
  \begin{equation}\label{eq:averagebestproof}
    \int_\cZ\norm[\cY]{\cF(z)-(y+v)}^2\dd\gamma(z)
    =\int_\cZ\norm[\cY]{\cF(z)-y}^2-2\dup{\cF(z)-y}{v}+\norm[\cY]{v}^2\dd\gamma(z).
  \end{equation}
  Since $\int_\cZ\dup{\cF(z)-y}{v}\dd\gamma(z)=0$ by definition of $y$, \eqref{eq:averagebestproof}
  is strictly larger than $\int_{\cZ}\norm[\cY]{\cF(z)-y}^2\dd\gamma(z)$.
\end{proof}

Fix a subspace $V\subseteq\cX$ of finite dimension $d$, 
and denote by $V^\perp$ its orthogonal complement. 
In the following, we let
\begin{equation*}
  x\sim \mu,\qquad
  v=P_Vx,\qquad
  w=P_{V^\perp}x,
\end{equation*}
so that in particular $x=v+w$. 
We denote by $\muV := (P_V)_\sharp\mu$ the {marginal of $\mu$ on $V$}
and by $\muVperp$ the distribution of $w$ conditioned on $v$; 
this is (for almost every $v$) well-defined, see, e.g., \cite[Vol.\ II, Section 10.4]{Bog07}.
With this convention, 
the objective in \eqref{eq:asobj} can be written as
\begin{equation}\label{eq:asobj2}
  \int_{V}\int_{V^\perp}\norm[\cY]{\cG(v+w)-\bar\cG_V(v)}^2\dd\muVperp{(w)} \dd\muV(v).
\end{equation}

By Lemma \ref{lemma:averagebest}, the mapping $\bar\cG_V:V\to\cY$ minimizing this objective can be identified as
the conditional mean of $\cG$ along $v+V^\perp$, i.e.
\begin{equation}\label{eq:barcG}
  \bar\cG_V(v) = \int_{V^\perp} \cG(v+w)\dd\muVperp{(w)}.
\end{equation}
With this observation, \eqref{eq:asobj} reduces to a minimization only over the $d$-dimensional subspace $V\subseteq\cX$. The objective \eqref{eq:asobj2} can further be rewritten to
\begin{align*}
  &\int_{V}\int_{V^\perp}\norm[\cY]{\cG(v+w)-\bbE_\mu[\cG]-(\bar\cG_V(v)-\bbE_\mu[\cG])}^2\dd\muVperp{(w)}\dd\muV(v)\\
  &\qquad=\int_V\int_{V^\perp}\norm[\cY]{\cG(v+w)-\bbE_\mu[\cG]}^2
    -2\dup{\cG(v+w)-\bbE_{\mu}[\cG]}{\bar\cG_V(v)-\bbE_{\mu}[\cG]}\\
  &\qquad\quad+ \norm[\cY]{\bar\cG_V(v)-\bbE_\mu[\cG]}^2\dd\muVperp{(w)}\dd\muV(v)\\
  &\qquad=\int_{\cX}\norm[\cY]{\cG(x)-\bbE_\mu[\cG]}^2\dd\mu(x)-
    \int_V\norm[\cY]{\bar\cG_V(v)-\bbE_\mu[\cG]}^2\dd\muV(v),
\end{align*}
where the second equality follows from the definition of $\bar\cG_V$ in \eqref{eq:barcG}. Since $\bbE_\mu[\cG]=\bbE_{\muV}[\bar\cG_V]$, the last integral corresponds to the variance of $\bar\cG_V(v)$ with respect to $v\sim \muV$. Thus, the optimization problem \eqref{eq:asobj} is equivalent to 
\begin{equation}\label{eq:asobj3}
  \sup_{\substack{V\subseteq \cX,~{\rm dim}(V)=d}} \int_{V} \norm[\cY]{\bar\cG_V(v)-\bbE[\bar\cG_V]}^2\dd\muV(v).
\end{equation}
In words, a solution to this problem is the $d$-dimensional subspace $V$ along which the \emph{variance of the conditional expectation of $\cG$, obtained by averaging over $V^\perp$, is maximized}.

\subsubsection{Characterization of $V$}
Unfortunately, due to the presence of the possibly nonlinear map
$\cG$, \eqref{eq:asobj3} is usually a very
nonconvex and hard to solve optimization problem. The idea is to
instead maximize an upper bound of {\eqref{eq:asobj3}}.

Denote in the following
\begin{align*}
  C^1\cap H^1(\cX,\mu;\cY):=& \Big\{\cG:\cX\to\cY:D\cG\in C^0(\cX,HS(\cX,\cY)),\\
  &\int_{\cX}\norm[\cY]{\cG(x)}^2\dd\mu(x)<\infty,~\int_{\cX}\norm[HS(\cX,\cY)]{D\cG(x)}^2\dd\mu(x)<\infty\Big\},
\end{align*}
where $D\cG(x)\in L(\cX,\cY)$ is the Fr\'echet derivative, which we
assume to be Hilbert-Schmidt at every $x\in\cX$. Suppose that there
exists $C_P>0$ such that for every $d$-dimensional subspace
$V\subseteq\cX$, and every $f\in C^1\cap H^1(V,(P_V)_\sharp \mu;\cY)$
\begin{equation*}
  \int_{V}\norm[\cY]{f(v)-\bbE[f]}^2\dd\mu(v)
  \le
  C_P \int_{V}\norm[HS(V,\cY)]{Df(v)}^2\dd\mu(v).
\end{equation*}
This is referred to as a Poincar\'e inequality. Then %
\begin{align}\label{eq:asupper}
  C_P\int_V\norm[HS(V,\cY)]{D\bar\cG_V(v)}^2\dd\mu(v)\nonumber
  &\le C_P \int_V\norm[HS(V,\cY)]{\int_{V^\perp}D_v\cG(v+w)\dd\mu(w|v)}^2\dd\mu(v)\nonumber\\
  &\le C_P \int_V\int_{V^\perp}\norm[HS(V,\cY)]{D_v\cG(v+w)}^2\dd\mu(w|v)\dd\mu(v)\nonumber\\
  &= C_P \int_{\cX}\norm[HS(V,\cY)]{D\cG(x)|_V}^2\dd\mu(x)
\end{align}
provides an upper bound to the objective in \eqref{eq:asobj3}.

To show how to maximize \eqref{eq:asupper}, we assume in the following
that $\cG\in C^1\cap H^1(\cX,\mu;\cY)$ and define the active subspace covariance operator
\begin{equation}\label{eq:ASCO}
  \covar :=\int_{\cX}D\cG(x)^*D\cG(x) {\dd\mu(x)} \in L(\cX,\cX),
\end{equation}
where $D\cG(x)^*\in L(\cY,\cX)$ is the adjoint of $D\cG(x)$.  Since at
every $x\in\cX$ the operator $D\cG(x)$ is Hilbert-Schmidt (and
therefore compact), and $D\cG(x)^*\in L(\cY,\cX)$, the operator
$D\cG(x)^*D\cG(x)\in L(\cX,\cX)$ is compact and self-adjoint with
$L(\cX,\cX)$-norm bounded by $\norm[HS(\cX,\cY)]{D\cG(x)}^2$. Due to
the square integrability of $\norm[HS(\cX,\cY)]{D\cG(x)}^2$, the
operator $\covar$ is well-defined, compact and self-adjoint.

\begin{proposition}\label{prop:AS}
  Let $V$ be a $d$-dimensional subspace of $\cX$. Then $V$ maximizes
  \eqref{eq:asupper} if and only if $V$ is a span of eigenvectors of
  $\covar$ corresponding to the $d$ largest eigenvalues.
\end{proposition}
\begin{proof}
  Denote by $e_1,\dots,e_d$ an orthonormal basis of $V\subseteq\cX$. 
  Then, by definition of the 
  Hilbert-Schmidt norm
  \begin{align*}
    \int_{\cX}\norm[HS(V,\cY)]{D\cG(x)|_V}^2\dd\mu(x) &= \int_{\cX}\sum_{j=1}^d\norm[\cY]{D\cG(x)e_j}^2\dd\mu(x)\\
                                                   &=\int_{\cX}\sum_{j=1}^d\dup{D\cG(x)^*D\cG(x)e_j}{e_j}\dd\mu(x)\\
    &=\sum_{j=1}^d\dup{\covar e_j}{e_j}.
  \end{align*}
  As in the proof of Proposition \ref{prop:PCA}, 
  one can show that this
  quantity is maximized if and only if $\{e_1,\dots,e_d\}$ is a set of
  eigenvectors corresponding to the $d$ largest eigenvalues of $\covar$.
\end{proof}

\begin{remark}[Error of active subspace approach]\label{rmk:aserror}
  Denote by $\lambda_1\ge \lambda_2\ge\cdots$ the eigenvalues of the 
  active subspace covariance operator $\covar$ in \eqref{eq:ASCO}. 
  Under suitable additional assumptions, 
  one can show that for $V$ as in Proposition \ref{prop:AS} and $\bar\cG_V$ as in \eqref{eq:barcG}
  \begin{equation*}
  \bbE_{z\sim \mu}[\norm[\cY]{\cG(z)-\bar\cG_V(P_Vz)}^2] \le C_P\sum_{j>d}\lambda_j.
  \end{equation*}
  We refer to \cite[Theorem 3.4]{PC-ActSubspc} for the proof of such a statement. 
  This again motivates the maximization of \eqref{eq:asupper}.
\end{remark} 

\subsubsection{Computation of Encoder}
The active subspace method can be used to determine efficient linear \emph{encoders}. 
Let $(\psi_{\cX,j})_{j=1}^d$ be $d$ orthonormal eigenvectors corresponding to $d$ 
largest eigenvalues of the active subspace covariance operator 
$\covar = \int_{\cX} D\cG(x)^* D\cG(x) \dd\mu(x)$. 
An encoder can then be constructed by
\begin{equation*}
\cE(x) = (\dup{\psi_{\cX,j}}{x})_{j=1}^d.
\end{equation*}

{The} operator $\covar$ is unknown, and we have to approximate it from a data sample.
A key difference {from} PCA is that we require information on the operator derivatives. 
Assume {we are} given $(D\cG(x_j))_{j=1}^m$ with $x_j\overset{\text{iid}}{\sim}\mu$. 
The procedure is then as follows:

\begin{enumerate}
\item {\bf Compute the empirical active subspace covariance operator}
\begin{equation*}
\hat\covar x = \frac{1}{m}\sum_{j=1}^m D\cG(x_j)^* D\cG(x_j)x\qquad\forall x\in \cX.
\end{equation*}
In a typical computational setting, 
$\cG$ is a function between two high- but finite-dimensional spaces $\R^k\to\R^n$.
In this case $D\cG(x)$ is represented by a $n\times k$ matrix, and
$D\cG(x)^*D\cG(x)=D\cG(x)^\top D\cG(x)\in\R^{k\times k}$.
\item \textbf{Solve the eigenvalue problem.} 
Compute the eigendecomposition of the symmetric matrix $\hat\covar \in\R^{k\times k}$
\begin{equation*}
    \hat\covar \hat{\psi}_{\cX,j} = \hat{\lambda}_j \hat{\psi}_{\cX,j}.
\end{equation*}
\item \textbf{Select the active directions.} 
Choose $d$ orthonormal eigenvectors $(\hat{\psi}_{\cX,j})_{j=1}^d$ 
corresponding to the $d$ largest eigenvalues $(\hat{\lambda}_j)_{j=1}^d$. 
These vectors form the basis for the empirical active subspace and are used to construct the encoder.
\end{enumerate}

\subsection{Autoencoders}\label{sec:Autoenc}
Autoencoders try to find latent representations of a data distribution $\mu$ on a Banach space $\cZ$. 
The ansatz is given via the optimization problem
\begin{equation}\label{eq:autoencoder}
  \inf_{\substack{\cA:\cZ\to\R^d\\ \cB:\R^d\to\cZ}} \bbE_{z\sim\mu}[\norm[\cZ]{z-\cB\circ\cA(z)}^p],
\end{equation}
for some $p\in [1,\infty)$. Further constraints on $\cA$ and $\cB$
discussed below, but importantly, they are allowed to be nonlinear;
{as we shall see, restricting $\cA$ and $\cB$ to be affine linear,
  gives a method that is equivalent to PCA.}
The interpretation of $\cA$ is that it finds a latent
variable representing the most important \emph{features} of the
data. The map $\cB$ reconstructs the input $z$ from its features
$\cA(z)\in\R^d$.
Again we only focus on $p=2$ and $\cZ$ a Hilbert space.

\subsubsection{Affine Autoencoder}
Suppose that we restrict $\cA$, $\cB$ to be affine linear maps. In this case $\cB\circ\cA(z)\in {\rm im}(\cB)$, which is an at most $d$-dimensional affine subspace of $\cZ$. Denote this subspace by $z_0+V$. The best approximation of $z$ within $z_0+V$ is given by $z_0+P_V(z-z_0)$. Therefore, if $\cA$, $\cB$ solve \eqref{eq:autoencoder}, it must hold
\begin{equation*}
  \cB\circ\cA(z) = z_0+P_V(z-z_0).
\end{equation*}
Assuming $\mu$ to have finite first and second moment, the optimal $z_0$ and $V$ are given by \eqref{eq:z0} and Proposition \eqref{prop:PCA}. Thus, an optimal solution $\cA$ and $\cB$ is given by
\begin{equation*}
  \cA(z) = (\dup{z-z_0}{\psi_j})_{j=1}^d,\qquad\cB((c)_{j=1}^d)= z_0+\sum_{j=1}^d c_j\psi_{j}
\end{equation*}
where $(\psi_{j})_{j=1}^d$ is an orthonormal eigenbasis corresponding to $d$ largest eigenvalues of the covariance of $\mu$ in \eqref{eq:covariance}. In this sense (affine) linear autoencoders are equivalent to PCA. 
We emphasize however that such $\cA$ and $\cB$ are not unique: For any invertible affine transformation $\cQ:\R^d\to\R^d$, we have that $(\cB\circ\cQ)\circ (\cQ^{-1}\circ\cA)$ provides another optimal solution to \eqref{eq:autoencoder}.

\subsubsection{Nonlinear Autoencoder}
\label{sec:NonlAEc}
While affine autoencoders can detect optimal linear subspace representations, they fail to capture nonlinear structure in the data. Autoencoders are therefore typically used for $\cA$, $\cB$ in nonlinear function classes such as neural networks.

Suppose that $\mu$ is a distribution whose support lies on a compact $d$-dimensional smooth manifold $M\subseteq\cZ$. If $M$ can be parametrized with a single chart $\varphi:[0,1]^d\to M$, then any extension $\cA:\cZ\to\R^d$ of $\varphi^{-1}:M\to\R^d$, and any extension $\cB:\R^d\to M$ of $\varphi$ satisfy that $\cB\circ\cA(z)=z$ for every $z$ in the support of $\mu$. Thus the latent dimension $d$ is sufficient to perfectly represent the data. In general this is not the case, if $M$ requires multiple charts. However, one can show that the latent dimension $2d+1$ still suffices in this case, e.g.\ \cite{dlrom}.

By contrast, linear subspaces may approximate manifolds only poorly, often requiring a much larger latent dimension. In such situations, nonlinear autoencoders can achieve far more efficient representations.

\subsubsection{Active Subspace Autoencoder}\label{sec:asae}
The idea of active subspaces can also be generalized to autoencoders.
In this case we replace the objective \eqref{eq:autoencoder} by
\begin{equation}\label{eq:asautoencoder}
  \inf_{\substack{\cE:\cX\to\R^{d_\cX}\\ \bar\cG:\R^{d_{\cX}}\to\cX}} 
  \bbE_{x\sim\mu}\!\left[\norm[\cX]{\cG(x)-\bar\cG\circ\cE(x)}^2\right].
\end{equation}
For linear $\cE$, this corresponds to the active subspace approach. 
For nonlinear $\cE$, the map may learn nonlinear features or a manifold 
capturing the directions of largest variation of $\cG$. This
then yields an encoder map $\cE$ for our surrogate.

Note that the objective \eqref{eq:asautoencoder} is very generic:
training it does not just yield an encoder $\cE$, but in principle it
can be used to obtain a complete surrogate $\bar\cG\circ\cE$ for
$\cG$. In fact, the {so-called ``DeepONet''} architecture can be interpreted as
an example of that type with a specific structure for $\bar\cG$ 
(see Section \ref{sec:DeepONet} and especially \eqref{eq:deepONet}).

\subsubsection{Computation of Encoder and Decoder}
Autoencoders can be used to determine nonlinear encoders and decoders 
$\cE$, $\cD$ from data as in Section~\ref{sec:representation}. 
Assume given training data $(x_j,\cG(x_j))_{j=1}^N$ with 
$x_j \overset{\text{iid}}{\sim}\mu$. 
To train an encoder $\cE:\cX\to\R^d$, we approximate the objective 
\eqref{eq:autoencoder} by its Monte-Carlo estimate
\begin{equation}\label{eq:autoencoder-encoder}
  \inf_{\substack{\cA_\cX:\cX\to\R^d\\ \cB_\cX:\R^d\to\cX}} 
  \frac{1}{N}\sum_{j=1}^N \norm[\cX]{x_j - \cB_\cX\circ\cA_\cX(x_j)}^2.
\end{equation}
The minimization can for example be carried out over neural network 
classes for $\cA_\cX$ and $\cB_\cX$. 
After training, $\cB_\cX$ is discarded, and we obtain an encoder $\cE:=\cA_\cX$. 
Alternatively, one may also employ a Monte-Carlo
approximation to the active-subspace autoencoder objective \eqref{eq:asautoencoder}.

To train a decoder, we analogously optimize
\begin{equation}\label{eq:autoencoder-decoder}
  \inf_{\substack{\cA_\cY:\cY\to\R^m\\ \cB_\cY:\R^m\to\cY}} 
  \frac{1}{N}\sum_{j=1}^N \norm[\cY]{\cG(x_j) - \cB_\cY\circ\cA_\cY(\cG(x_j))}^p.
\end{equation}
In this case $\cA_\cY$ is discarded, and we set $\cD:=\cB_\cY$.
\section{Architectures}
\label{sec:Arch}
With en- and decoders at hand, we turn to the actual architecture 
of (families of) finite-parametric operators 
$\{ \hat{ {\calG^N}} \}_{N\geq 1}:\calX\to \calY$ 
acting on infinite-dimensional function spaces.
Several classes of architectures of operator surrogates
have emerged in the past five years. 
Most, 
if not all operator surrogates considered fit into the 
``Encoder-Approximator-Decoder'' paradigm \eqref{eq:DgE}
i.e.\ they are of the generic form 
\begin{equation}\label{eq:DgEN}
{\hat\calG}^N = \calD_\calY \circ {g^N} \circ \calE_\calX 
\;.
\end{equation}
Encoders and decoders as discussed above 
relate function space in- and output
to parameter domains in sequ{enc}e spaces.
{In addition to finite-parametric} 
surrogate maps $g^N$ in \eqref{eq:DgEN}, on the input side, 
\emph{dimension-truncation} to a finite input dimension $d_{\calX}$ is performed 
on the coefficient sequence representing the encoded input (signal),
and the $d_{\calY}$-dimensional output of $g^N$ 
is understood to be padded by zeros 
prior to synthesis operator application of $\calD_\calY$ in \eqref{eq:DgEN}. 
Early examples of mathematical results for \eqref{eq:DgEN} are 
\cite{SLSMGEM22}, and \cite{HSZ24}.

The \emph{approximator} $g^N$ depends on ``latent'' 
parameters which are to be found from data 
by suitable training procedures.
Increasing the number $N$ of latent parameters, 
one minimally aims
at some form of \emph{universality} of the sequence $g^N$
{(see the discussion in Sec.~\ref{sec:ExRes})}.
Accordingly, 
the so-called \emph{curse of dimensionality} must be addressed and,
if at all possible, overcome. 
Several approaches to obtain suitable $g^N$ in \eqref{eq:DgEN}
have emerged in the past years:
we mention only 
$n$-term truncated polynomial chaos expansions and 
stochastic Galerkin resp. stochastic collocation {\cite{DNSZ23_2957}.} %
neural networks \cite{HSZ18_2635},
tensor-structured approximation \cite{lebedev2015speedingupconvolutionalneuralnetworks}, 
probabilistic ``random feature'', 
kernel surrogates \cite{NN+AMS24,nguyen2026randomfeaturesoperatorvaluedkernels},
and lattice point interpolation and approximation.
We detail here polynomial chaos approximators $g^N$ in \eqref{eq:DgEN}
giving rise to Spectral Operator Surrogates (with $N$ depending on $n$).
We also discuss several neural network based approximator 
constructions
giving rise {to} the ``DeepONet'' and ``Neural Operator'' surrogates;
\emph{their structure is distinct {from that} of \eqref{eq:DgEN}}.
\subsection{Spectral Operators}
\label{sec:SpecOp}
Consider stable linear representations of the operator in- and output
as described in Section \ref{sec:stable}:
\begin{equation}\label{eq:iorepresentation}
  x = \sum_{j\in\N}\varphi_{\cX,j}(x)\psi_{\cX,j}\qquad\text{and}\qquad
  y = \sum_{i\in\N}\varphi_{\cY,i}(y)\psi_{\cY,i}\qquad\forall x\in\cX,~y\in\cY.
\end{equation}
An operator $\cG:\cX\to\cY$ can then be decomposed via
\begin{equation}\label{eq:opdecomp}
  \cG(x) = \sum_{i\in\N} g_i((\varphi_{\cX,j}(x))_{j\in\N}) \psi_{\cY,i},
\end{equation}
for the functions $g_i:\R^\N\to\R$. 
Since the $\varphi_{\cX,j}$ and the $\psi_{\cY,i}$ are known, 
{the task of operator learning now}
boils down to approximating the $g_i$.
Using \eqref{eq:SAid} we have an explicit formula 
\begin{align}\label{eq:giexplicit}
  g_i((\varphi_{\cX,j}(x))_{j\in\N}) &= \varphi_{\cY,i}(\cG(x))
  = \varphi_{\cY,i}\Big(\cG\Big(\sum_{i\in\N}\varphi_{\cX,j}(x) \psi_{\cX,j}\Big) \Big).
\end{align}
If $(\psi_{\cY,j})_{j\in\N}$ is a basis of $\cY$, then \eqref{eq:giexplicit} are the unique functions such that \eqref{eq:opdecomp} holds.
Due to the linearity of the $\varphi_{\cX,j}$, the coefficient functions $g_i$ inherit smoothness properties, such as differentiability, from the operator $\cG$. Since many practically relevant operators are highly smooth, this motivates the use of spectral methods to approximate the $g_i$.

\subsubsection{Polynomial Chaos {Expansions}}\label{sec:pc}
Each coefficient function $g_i$ is a function of the infinitely many variables $\coeff[j]:=\varphi_{\cX,j}(x)\in\R$, $j\in\N$. We also write
\begin{equation}\label{eq:Bx}
  \coeff =(\coeff[j])_{j\in\N}=(\varphi_{\cX,j}(x))_{j\in\N}\in\R^\N.
\end{equation}
Under certain conditions, such infinite-parametric functions allow for multivariate polynomial expansions (known as ``polynomial chaos''), see, e.g., \cite{Xiu2002,Ernst2012} or \cite[Section 2.1]{JZDiss}. A simple example is Taylor polynomial chaos: Denote %
\begin{equation*}
  \cF:=\setc{\Bnu=(\nu_j)_{j\in\N}}{\nu_j\in\N_0,~\sum_{j\in\N}\nu_j<\infty}.
\end{equation*}
The Taylor polynomial chaos expansion of $g_i$ (assuming it exists) then reads
\begin{equation*}
  g_i((\varphi_{\cX,j}(x))_{j\in\N})=g_i(\coeff) = \sum_{\Bnu\in\cF} c_{i,\Bnu} \coeff^{\Bnu},
\end{equation*}
for certain coefficients $c_{i,\Bnu}\in\R$. 
Under suitable assumptions %
it holds {that} $c_{i,\Bnu}=\frac{1}{\Bnu!}\frac{\partial^{\Bnu}}{\partial{\Bx}^{\Bnu}} g_i(\Bx)|_{\Bx=\Bnul}$, e.g.\ \cite[Lemma 2.1.2]{JZDiss}; here $\Bnu!=\prod_{j\in{\rm supp}(\Bnu)}\nu_j!$ and $\Bx^{\Bnu}=\prod_{j\in\supp(\Bnu)}x_j^{\nu_j}$ where ${\rm supp}(\Bnu)=\set{i}{\nu_i\neq 0}$ and empty products equal $1$ by convention.

For a finite set of multiindices $\Lambda_i\subseteq\cF$, the $i$th coefficient function $g_i$ can then be approximated by a truncated polynomial chaos expansion
\begin{equation*}
  g_i((\varphi_{\cX,j}(x))_{j\in\N})\simeq \sum_{\Bnu\in\Lambda_i} c_{i,\Bnu} \coeff^{\Bnu}.
\end{equation*}
This yields an approximation to the operator via
\begin{equation}\label{eq:pcoperator}
  \cG(x) \simeq \sum_{i\in\N} \psi_{\cY,i} \sum_{\Bnu\in\Lambda_i} c_{i,\Bnu}\coeff^\Bnu.
\end{equation}
In order for this formula to be computable, we require $\Lambda_i=\emptyset$ for all but finitely many $i$. Since the right-hand side in \eqref{eq:pcoperator} is (affine) linear in the unknown coefficients $c_{i,\Bnu}$, they can be fitted by least-squares. 
In practice, for stability reasons other expansions such as multivariate Legendre or Hermite polynomials are preferred.

A key difficulty of this approach is the choice of efficient multiindex sets $\Lambda_i$. A concrete construction is provided in \cite[Sections 5 and 6]{HSZ24}. Here we only give some intuition:
\begin{itemize}
\item $g_i$ is the coefficient function corresponding to $\psi_{\cY,i}\in\cY$ on the output side. The representation system on the output side is typically ordered from coarse to fine scales. 
  This is for instance the case if $\psi_{\cY,i}$ is a Fourier type basis or a wavelet, e.g.\ 
{$\psi_{\cY,i}(\spatialvar)=\sin(i \spatialvar)$, $\spatialvar\in\R$}. 
Assuming smoothness of the operator output, low-index (small $i$) functions capture the dominant features of the output, 
while high-index functions correspond to finer details. 
{As smoothness assumptions imply faster decay of the coefficient sequence, in this
      case one} expects the cardinality of the $\Lambda_i$ to decrease, i.e.
  \begin{equation*}
    |\Lambda_1|\ge |\Lambda_2|\ge\dots.
  \end{equation*}
\item Similar considerations hold for the representation of the input with respect to $\psi_{\cX,j}$, cf.~\eqref{eq:iorepresentation}. 
      Since lower $j$ corresponds to more dominant features of the input, for fixed $i$, 
      the multiindex set $\Lambda_i\subseteq\cF$ should be anisotropic, with higher importance given to lower $j$.
\end{itemize}
Without recalling the full details of \cite{HSZ24}, 
the multiindex sets constructed there are of the following type, 
which reflects the above intuition: 
For two null sequences $\Brho=(\rho_j)_{j\in\N}$, $\Bxi=(\xi_i)_{i\in\N}$ with values in $(0,1)$,
\begin{equation*}
  \Lambda_i = \set{\Bnu\in\cF}{\xi_i\Brho^{\Bnu}\ge \eps},
\end{equation*}
with $\eps>0$ being a parameter associated to the overall accuracy of the scheme. 
This construction ensures in particular $\Lambda_i=\emptyset$ for all but finitely many $i$.

\subsubsection{Sparse-Grid {Interpolation}}\label{sec:sg}
A closely related approach is to use sparse-grid interpolants instead of polynomial chaos expansions. 
With the notation in \eqref{eq:Bx}, the surrogate is constructed as
\begin{equation}\label{eq:sgsurrogate}
  \cG(x) \simeq \sum_{i\in\N} I_{\Lambda_i}[g_i](\coeff) \psi_{\cY,i},
\end{equation}
where $I_{\Lambda_i}[g_i]$ is a sparse-grid interpolant of $g_i$; 
we recall its construction below. 
The $\Lambda_i\subseteq\cF$ are again multiindex sets, for which the remarks in Section \ref{sec:pc} hold verbatim. 
One advantage of this approach is that it has deterministic convergence guarantees, see \cite{HSZ24}. 

Computing \eqref{eq:sgsurrogate} requires computation of the interpolation operators. 
We assume in the following that each $\Lambda_i$ is \emph{downward closed}, 
meaning that if $\Bnu\in\Lambda_i$, then $\Bmu\in\Lambda_i$ for all $\Bmu\in\cF$ such that $\mu_i\le\nu_i$ for all $i\in\N$. 
For $n\in\N$ denote by $I_{n}:C^0(\R)\to \bbP_n$ a univariate polynomial interpolant in $n+1$ points; 
typical examples are %
{\Tsch} or Hermite interpolants. 
With the multivariate tensorized interpolants
\begin{equation*}
  I_{\Bnu} = \otimes_{j\in\supp{\Bnu}} I_{\nu_j}
\end{equation*}
the corresponding sparse-grid interpolant of coefficient map $g_i$ is defined as
\begin{equation*}
  I_{\Lambda_i}[g_i] = \sum_{\Bnu\in\Lambda_i} a_{\Bnu,i} I_\Bnu[g_i]\qquad\text{where}\qquad
  a_{\Bnu,i}=\sum_{\substack{\Be\in\{0,1\}^\N\\ \text{s.t. }\Bnu+\Be\in\Lambda_i}}(-1)^{|\Be|},
\end{equation*}
where $|\Be|=\sum_{j\in\N}e_j\in\N_0$. 
For more details and explicit formulas we refer for example to \cite{Bungartz_Griebel_2004,klimke2006uncertainty,JZDiss}.

We emphasize that the computation of \eqref{eq:sgsurrogate} requires {us} to evaluate the coefficient maps $g_i$ in \eqref{eq:giexplicit} at the interpolation points, which are determined by the one-dimensional interpolation operators $I_n$, $n\in\N$, as well as the sets $\Lambda_i$, $i\in\N$. As long as the $\Lambda_i$ are fixed, so is the set of required evaluation points of the operator. These operator evaluations can then be done a priori and in parallel. Subsequently, no further form of training is required, as the surrogate has an explicit formula.

\begin{codebox}
We apply operator surrogates based on sparse-grid interpolation to the parametric diffusion problem in notebook \texttt{2\_Interpolation.ipynb}.
\end{codebox}

\subsubsection{Low-rank Tensor Approximations}
\label{sec:LwRkTT}
Another possibility is to represent the $g_i$ via a low-rank representation of its values at some grid. 
Specifically, fix a set of univariate interpolation points $\xi_0,\dots,\xi_n\in\R$; these could for example again be %
{\Tsch} or Gauss-Hermite points. %
{Setting the variables after the first $d$ dimensions to $0$, $g_i$ can be approximated by the tensor interpolant}
\begin{equation}\label{eq:tensorinterpolant}
  I_n\otimes I_n\otimes\cdots \otimes I_n[g_i(\cdot,\dots,\cdot,0,0,\dots)].
\end{equation}
This representation requires {us} to store the function values at the interpolation points, i.e.\ the $d$ tensor
\begin{equation}\label{eq:gitensor}
  g_i(\xi_{j_1},\dots,\xi_{j_d})_{j_1,\dots,j_d=0}^n\in\R^{(n+1)\times\cdots\times (n+1)}.
\end{equation}
The idea is to compress this tensor via a low rank representation.

A popular low-rank tensor format is the tensor train (TT) format \cite{oseledets11}. 
It tries to approximate \eqref{eq:gitensor} via
\begin{equation}\label{eq:separation}
  g_i(\xi_{j_1},\dots,\xi_{j_d}) \simeq \Bv^{(1)}(j_1)\cdots\Bv^{(d)}(j_d)\qquad\forall j_1,\dots,j_d=0,\dots,n,
\end{equation}
with the {\bf TT cores}
\begin{equation*}
  \Bv^{(1)}\in\R^{1\times (n+1)\times r},~
  \Bv^{(d)}\in\R^{r\times (n+1)\times 1}\qquad\text{and}\qquad
  \Bv^{(j)}\in\R^{r\times (n+1)\times r}~\forall j=2,\dots,d-1.
\end{equation*}
The memory cost of this format is $O(r^2nd)$, compared to $O(n^d)$ for the direct storage of \eqref{eq:gitensor}. 
For moderate ranks $r$, this is a significant improvement. 
Moreover, given a representation in the TT format, the interpolant in \eqref{eq:tensorinterpolant} 
is very efficient to evaluate: Due to the separation of the cores \eqref{eq:separation}, 
the interpolation can be done individually in each dimension. 
Evaluating the interpolant thus scales only linearly in the dimension.

A key question is how to efficiently obtain a TT approximation of the tensor in \eqref{eq:gitensor}. 
A standard algorithm is so-called TT-cross \cite{oseledets10}. A drawback is that it requires many sequential evaluations of $\cG$, 
which can make the process slow in practice. 
For this reason the ALS-cross algorithm was developed in the context of parametric PDEs \cite{dolgov18}; 
while it also relies on sequential evaluations of $\cG$, it can be significantly more efficient. 
The downside is that it requires a specific structure for the operator $\cG$, thereby limiting its applicability.
\subsection{{Nonlocal Operator Networks}}
\label{sec:NeurOp}
A unifying mathematical framework covering a wide class of 
operator surrogate architectures \eqref{eq:DgEN}
was recently promulgated in \cite{kovachki2021neural} 
under the label ``Neural Operator''.
{In order to distinguish it from other surrogate architectures based on neural networks 
     we here refer to this architecture as \emph{Nonlocal Operator Network} (Nonlocal ONet).}
{Nonlocal ONets} mimic the architectures of traditional 
     finite-parametric, deep feedforward neural networks in finite dimensions, 
while  
a) admitting infinite-dimensional input and output spaces $\cX$ and $\cY$, 
b) \emph{replacing the affine weight-bias maps by linear operators}, 
and, importantly,
by application of a 
\emph{nonlocal (integral) operator with prescribed kernel} 
(``linear kernel layer'' \cite{kovachki2021neural},
``continuum attention'' \cite{calvello2025continuumattentionneuraloperators} ) 
and 
c) by injecting nonlinearity, as in the finite-dimensional case, 
via \emph{activation functions}. Acting on function spaces,
these are interpreted either in a pointwise sense,
or as Nemy{tskii} operators in the (Sobolev or Besov) function spaces of interest.

The general structure of a {nonlocal ONet} of finite depth $L\in \N$ acting between 
Banach input space $\cX$ and output space $\cY$ is \cite{kovachki2021neural}
\be\label{eq:nonlocalONet}
{\hat\calG} : \cX \to \cY: u\mapsto Q\circ \cL_L\circ ... \circ \cL_2\circ \cL_1 \circ \cR(u) \;.
\ee
The 
layers $\cL_\ell:\cX_\ell \to \cX_{\ell+1}$, $\ell = 1,...,L$, 
are \emph{nonlinear maps between suitable function spaces}
$\{\cX_\ell\}_{\ell = 1,...L+1}$
of the general form
\be \label{eq:NeurLay}
\cL_\ell: \cX_\ell \to \cX_{\ell+1}: v \mapsto \sigma(W_\ell v + b_\ell + \cK_\ell(v;\gamma_\ell) )\;,
\quad \ell = 1,2,...,L\;.
\ee
Here, 
the activation function $\sigma(\cdot)$ 
is understood to act in a pointwise sense in \cite{kovachki2021neural},
which mandates that $\cX_\ell \subset C^0(\domain_\ell)$ 
{for some domain $\domain_\ell \subset \R^{d_\ell}, d_\ell \in \N$. Further,  
$W_\ell \in L(\cX_\ell,\cX_{\ell+1})$ is a {bounded,} linear operator, %
$b_\ell \in \cX_{\ell+1}$ a bias term,
and $\cK_\ell$ a nonlocal integral operator
of} the generic form 
\[
{\cK_\ell}(v;\gamma)(\spatialvar)
:= 
\int_{\spatialvarother\in \domain_\ell} \kappa_\ell(\spatialvar,\spatialvarother;\gamma)v(\spatialvarother) d\spatialvarother, \quad {{\spatialvar \in \domain_\ell,}}
\]
with suitable kernel functions $\kappa_\ell(\cdot,\cdot;\gamma)$ 
which depend on \emph{(trainable) latent parameter vectors $\gamma$}. 
In \eqref{eq:nonlocalONet},
$\cR: \cX \to \cX_1$ and $\cQ: \cX_{L+1} \to \cY$.
All layers $\cL_\ell$ in \eqref{eq:nonlocalONet} may share the same kernel, 
in which the subscript $\ell$ is dropped in the definition of $\cL_\ell$,
i.e.\ $\domain_\ell = \domain$, 
$\cX_\ell = \cX$ and $\kappa_\ell = \kappa$ for $\ell = 1,...,L$.

In \eqref{eq:nonlocalONet},
{the kernel parameters are (a subset of)} trainable parameters
of the {neural operator $\hat\cG$,}
i.e.\ 
$\{\gamma_\ell\}_{\ell = 0,1,...,L}\subset {\theta}$.
Definition \eqref{eq:nonlocalONet} as such is not finite-parametric, due to
it involving operators acting on the function spaces $\cX_\ell$. 
{A nonlocal ONet as in} \eqref{eq:nonlocalONet} could be viewed as 
``infinite-width closure'' or continuum limit 
of finite-parametric, fully connected feedforward networks 
(subject to suitable scaling hypotheses).

Computationally viable architectures arise from 
{nonlocal ONets} \eqref{eq:nonlocalONet} via \emph{discretization}, 
where any discretization method (Nystr\"om, Galerkin, quadrature,...) could 
be employed.
The conceptual abstraction \eqref{eq:nonlocalONet} 
in the form of {nonlocal ONets} %
allows, however, by judicious choices of discretizations, 
to unify existing constructions and to 
generate new, finite-parametric architectures. 
In any case, a discretization error analysis is required 
to quantify the difference between a {nonlocal ONet} %
and
its computational version,
see, e.g.\  \cite{Kovachki2021,lanthaler2024discretizationerrorfourierneural} 
for a study for Fourier Neural operators {(FNO)}.

Particular choices of the kernels $\kappa_\ell$ 
in the operators $\cK_\ell$
will give rise to particular neural operators: 
translation-invariant kernels
$\kappa(\spatialvar,\spatialvarother;\gamma) = \bar{\kappa}(z;\gamma)\mid_{z=\spatialvar-\spatialvarother}$ 
yield the convolutional neural operator \cite{NEURIPS2023_f3c1951b}.
It {can} %
be discretized and numerically realized via 
discrete Fourier- and inverse Fourier transforms
\cite{Kovachki2021,lanthaler2024discretizationerrorfourierneural},
Fast Multipole Methods (FMMs) \cite{li2020multipole},
${\mathcal H}$-Matrix Techniques,
Tensor-structured Formats \cite{lebedev2015speedingupconvolutionalneuralnetworks}, 
to name but a few.
In periodic settings, where $\cX_\ell$ are function spaces on 
the torus ${\mathbb T}^d$ with periodicity, \eqref{eq:nonlocalONet} 
takes the form of the %
Fourier Neural Operator %
\cite{Kovachki2021,kovachki2021neural}, i.e.\ 
\be\label{eq:FNO}
{\cK}_\ell(v;\gamma)(x)
:= 
\cF^{-1}_{\xi \to \spatialvar} ( \hat{\kappa}_\ell(\xi;\gamma) \hat{v}(\xi) ) 
\ee
with $\hat{\kappa}_\ell$ denoting the Fourier-transform (``Fourier-symbol'') 
of translation-invariant kernel $\bar{\kappa}_\ell(z;\gamma)$, 
i.e.
$\hat{\kappa}_\ell(\xi;\gamma) = \cF_{z\to \xi} \bar{\kappa}_\ell(z;\gamma) $.

\begin{codebox}
An implementation of the discretized FNO for one-dimensional input and output functions is provided and applied to the parametric diffusion problem in notebook \texttt{4\_FNO.ipynb}.
\end{codebox}

\subsection{PCA-Net and FrameNet}
\label{sec:PCAFrm}
The PCA-Net architecture was among the first neural network-based approaches for learning operator surrogates \cite{HESTHAVEN201855, SMAI-JCM_2021__7__121_0}. It fits the abstract encoder-approximator-decoder framework \eqref{eq:DgE}. Given data $(x_i, \cG(x_i))_{i=1}^m$ with $x_i \overset{\text{iid}}{\sim} \mu$, 
the training is in two steps: first, the encoder $\cE:\cX\to\R^{d_\cX}$ and decoder $\cD:\R^{d_\cY}\to\cY$ are constructed via empirical PCA (see Section~\ref{sec:PCA}). Second, the approximator $g:\R^{d_\cX}\to\R^{d_\cY}$ in \eqref{eq:DgE} is set up as a feedforward neural network, and is trained on the resulting latent variables. The error analysis for empirical PCA is well-developed (see, e.g., \cite{Reiss2020Nonasymptotic,MILBRADT2020108741}). In the specific context of PCA-Net, an analysis was done in \cite{Lanthaler2023}: due to data randomness, the error bounds from the PCA reconstruction errors scale in general only with the inverse square root of the number of data points, i.e., the Monte-Carlo rate; 
{this dimension reduction error is independent of the choice of coefficient map, 
     and is thus a general limitation valid also for other PCA-based architectures.}

Closely related is the FrameNet architecture \cite{HSZ24, RWZ24}, where the encoder $\cE$ and decoder $\cD$ are defined by a priori (data-independent) Banach frames on $\cX$ and $\cY$, see Section~\ref{sec:stable}. For PDE-based solution operators, this is a feasible approach, as suitable frame constructions for typical solution spaces are well-known in the literature, e.g., Example~\ref{ex:frames}. While this allows for fast algebraic convergence rates for the approximation error \cite{HSZ24}, this advantage does not necessarily carry over to the generalization error, see Section~\ref{sec:generalization} and also \cite{grohs2025theorytopracticegapneuralnetworks}.

In practice, while FrameNet has desirable theoretical properties, PCA-Net is often simpler to implement. Moreover, by using data-dependent representations for its input and output spaces, PCA-Net may be better able to capture latent structures within the operator $\cG$ and the input measure $\mu$.

\begin{codebox}
An implementation of PCA-Net is provided and applied to the parametric diffusion problem in notebook \texttt{3\_PCANet.ipynb}.
\end{codebox}

\subsection{DeepONet}
\label{sec:DeepONet}
The ``DeepONet'' class of neural operators is one of the earliest approaches. 
Its ``branch-and-trunk'' architecture differs at first sight 
from \eqref{eq:DgE} considered elsewhere in this text, 
{but can also be interpreted in this sense, see below.}
Universality of the DeepONet architectures has first been established in 
\cite{ChenChen1993}, on compact subsets of admissible input signals, 
and in 
\cite{SLSMGEM22} also in the non-compact case. 
There, also first error bounds which are free from the CoD were obtained, 
in the case of holomorphic maps, and with PCA decoding.

We refer to \cite{GEK-Fair22} for a succinct presentation of {DeepONet} architecture,
comparison with Fourier-Neural Operator architecture, and quantitative comparisons.
A key feature of DeepONet is the so-called ``branch Net'' part of the DeepONet 
architecture (see \cite[Fig.~1]{GEK-Fair22} for an illustration), which performs
decoding of signals in a feature space which is data-adaptive, 
i.e. learned from training data.
{Importantly, and contrary to the PCA approach, this decoder is learned
simultaneously with the approximator.}

The trunk-net part of the DeepONet 
takes the coordinates $\spatialvar$ in the physical domain $\domain$ 
which supports the output,
{and} the branch net takes the discretized function $x$ as input.
The output of the DeepONet takes the generic form
\begin{equation}\label{eq:deepONet}
{\hat \calG} (x)(\spatialvar) = \sum_{k=1}^{d_{\cX}} b_k(x) t_k(\spatialvar) + b_0, 
\quad 
\spatialvar \in \domain
\end{equation}
where $b_0\in \R$ is a bias, $\{b_1, b_2, . . . , b_{d_{\cX}} \}$ 
are  $d_{\cX}$ outputs of the branch net, 
and $\{t_1, t_2, . . . , t_{d_{\cX}} \}$ 
representing the DeepONet output are the $d_{\cX}$ output features of the trunk net,
which can be interpreted as an encoder $\cE:\cX\to\R^{d_{\cX}}$ in \eqref{eq:DgE}; 
also see Section \ref{sec:asae}. 
The approximator and decoder $\cD\circ g:\R^{d_\cX}\to\cY$ then 
correspond to the map 
$(c_j)_{j=1}^{d_{\cX}}\mapsto b_0 + \sum_{k=1}^{d_{\cX}}c_k t_k(\spatialvar)$.
\begin{remark}\label{rmk:DONEAD}
We may interpret the DeepONet's components 
within an encoder-decoder framework 
by defining the branch net as the encoder. 
While not canonical, this choice aligns with our presentation of 
active subspace autoencoders in Section \ref{sec:asae}.
\end{remark}

DeepONet is a neural operator architecture framework 
which, a-priori, does not 
restrict the branch net and the trunk net yet to any specific architecture. 
As $\xi \in \domain$ is the physical domain of the input function, 
a {feedforward} NN is often used as the trunk net. 
The architecture of the branch net which defined the output feature space
can be chosen as a {convolutional neural network (}CNN{)}, ResNet, or a graph neural network (GNN), etc. 
For example, 
if the discretization of $x$ is equispaced and cartesian, a CNN or a {feedforward} NN offers advantages.
If the discretization of $x$ is on an unstructured mesh, then a 
GNN is preferable.
\subsection{Transformers}
\label{sec:Transf}
Starting with the seminal \cite{vaswani2023attentionneed} introduction of 
transformers for finite-dimensional maps,
transformer architectures have emerged in recent years as key
ingredie{nts} for large language models, such as Chat-GPT. 
They are also at the heart of integrating \emph{multimodal data} 
(such as images, text and sound) into large language models 
{\cite{calvello2025continuumattentionneuraloperators, LLMmathematical}.}
Their current relevance
for SciML and SciAI lies in transformer ability 
to learn from and generalize across multimodal PDE input data 
(constitutive equations, domain shapes, loadings, initial and boundary data,
 solution snapshots, form of the PDE, etc.).
For operator learning in a PDE context, 
this capability is key for 
\emph{Scientific AI}: 
the generalization of data-to-solution maps across a range of partial differential equations
{\cite{POSEIDON2024}, e.g.,}
from a simple diffusion equation in a square to linear elastostatics in polygonal domains.
Multimodal inputs can also arise in the challenging task %
{of generalizing} across physical domains, or 
{of identifying} a PDE from data and snapshot solutions,
{or of the inverse problems to identify a domain of a PDE from solution snapshots.}

An early contribution to transformer-architectures for neural PDE solvers in SciML
is
\cite{cao2021choosetransformerfouriergalerkin}.

More recently, on the side of mathematical theory,
universal approximation of transformer-based neural operators has been shown e.g. 
for stochastic processes in path space in 
\cite{acciaio2022metric} and in the finite-dimensional setting
in \cite{furuya2024transformersuniversalincontextlearners}, and references there.
See also \cite{wen2025geometryawareoperatortransformer} for Transformer-generalization of PDE
solution operators across classes of physical domains. 
The recent work \cite{molinaro2025generativeaifastaccurate} leverages transformers for
prediction of statistical properties of (turbulent) flows.
Approximation rate bounds for transformer-based neural architectures
have been established in \cite{furuya2024transformersuniversalincontextlearners,
shih2024transformersneuraloperatorssolutions}.

To present the transformer architectures ${\hat \calG_T}$ in function space, 
we follow \cite[Sec.~4]{calvello2025continuumattentionneuraloperators} 
and place ourselves again in an encoder-transformer-decoder architecture, i.e.\ 
\begin{equation}\label{eq:Trns}
{\hat \calG_T} : \calX\to \calY : x \mapsto \Tout \circ \EL \circ \Tin(x) \;.
\end{equation}
Here, $\Tin: \calX \to \calV$ maps input data ${x} \in \calX$ to an embedding function
space $\calV$ which could be viewed as \emph{latent function space}.
Importantly, 
$\Tin$ which is key for the transformers' ability to deal with multimodal data
is \emph{not a linear operator} as {e.g., in PCA-Net and FrameNet.}
The operator $\EL: \calV\to \calV$ is the nonlinear, function space analog
to the transformer encoder block in \cite{vaswani2023attentionneed},
and $\Tout: \calV\to \calY$ is a linear decoder.
To describe the structure of the blocks $\Tin, \EL, \Tout$ 
in \eqref{eq:Trns}
in more detail,
we denote $\domain\subset \R^{{d_\spatialvar}}$ and 
$\domain'\subset \R^{{d_\spatialvarother}}$ bounded domains
on which the input and output function spaces, 
$\calX = \calX(\domain;\R^{{d_x}})$ and $\calY = \calY(\domain';\R^{{d_y}})$
are modelled, respectively (think of Besov or Sobolev spaces on $\domain$ resp. $\domain'$,
of vector-valued functions), and the map to be emulated being  
$\calG:\calX\to \calY: a\mapsto u$.

In constructing $\Tin$, 
an input ${x} \in \calX(\domain;\R^{{d_x}})$ is \emph{positionally encoded} 
{by appending the spatial coordinate $\spatialvar$ to each function value:} %
$$
{x}_{in}(\spatialvar) := ({x}(\spatialvar),\spatialvar)\;,\quad \spatialvar\in \domain\;.
$$
{This approach breaks the translation equivariance of the learned weights and biases, 
     and is also common in Fourier Neural Operators.}

Subsequently, 
a linear transformation 
$W_{in} \in \R^{d_{model}\times (d_x + {d_\spatialvar})}$
is applied to ${x}_{in}(\spatialvar)$ for $\spatialvar \in \domain$ to lift ${x}_{in}$ to $\calV$, 
so that 
$$
(\Tin ({x}))(\spatialvar) := W_{in} {x}_{in}(\spatialvar)\;,\quad \spatialvar \in \domain \;.
$$
The transformation $\Tin$ acts locally, 
pointwise (w.r.t. $\spatialvar\in \domain$) on ${x}_{in}$,
with output in the latent embedding space $\calV(\domain,\R^{d_{model}})$.

The transformer block $\EL: \calV\to\calV$ in \eqref{eq:Trns}
acts on inputs $v^{(0)} := \Tin(x) \in \calV $ by an iteration
(reminiscent of recurrent blocks in operator unrolling)
which is defined, for $l=1,2,...,L$, by the formal assignments
\begin{equation}\label{eq:EL}
\begin{array}{rcl}
v^{(l-1)} & {\leftarrow} & W^{(l-1)}_1 v^{(l-1)} + \Amult^{(l-1)}( v^{(l-1)} ) \;,
\\
v^{(l-1)} & {\leftarrow} & \Flayer^{(l-1)}( v^{(l-1)} ) \;,
\\
v^{(l-1)} & {\leftarrow} & W^{(l-1)}_2 v^{(l-1)} + \FNN^{(l-1)}(  v^{(l-1)} ) \;,
\\
v^{(l)} & {\leftarrow} & \Flayer^{(l-1)}( v^{(l-1)} ) \;.
\end{array}
\end{equation}
with $\EL(v^{(0)}) := v^{(L)} \in \calV$.
{
Here, 
the nonlinear $\Flayer$ performs layer normalization,
and $\FNN: \calV\to\calV$ is a standard feedforward neural network layer.

In \eqref{eq:EL},
}
the multi-head, self-attention operator $\Amult: \calV \to \calV$ is 
a composition of a (multi-) linear transformation $\Wmult$ of 
$\Nsa\in \N$ many self-attention operators, 
$$
(\Amult(v))(\spatialvar) = \Wmult(A_*(v;\theta_1)(\spatialvar),...,A_*(v;\theta_{\Nsa})(\spatialvar)), \spatialvar\in \domain, 
$$
with $A_*$ a trainable (with training parameters $\theta_j$)
\emph{continuum self-attention operator} 
acting pointwise on inputs $v\in \calV(\domain,\R^{d_{model}})$.

The continuum self-attention operator is obtained in 
\cite[Appendix~1]{calvello2025continuumattentionneuraloperators}
by a limiting argument from the originally proposed 
discrete self-attention in \cite{vaswani2023attentionneed}.
With the probability density function $p(\cdot;v,\spatialvar):\domain\to \R_+$
$$
p(\spatialvarother;v,\spatialvar) 
= 
\frac{\exp(\langle Qv(\spatialvar), Kv(\spatialvarother) \rangle)}{\int_\domain \exp(\langle Qv(\spatialvar), K{v}(s) \rangle) ds}\;,
\quad \spatialvar\in \domain, \spatialvarother\in \domain, 
$$
where $Q,K$ are query and key matrices, respectively, 
it takes the form
\begin{equation} \label{eq:ContSA}
A_*(v)(\spatialvar) = \IE_{\spatialvarother\sim p(\spatialvarother;v,\spatialvar)}[Vv(\spatialvarother)] \;, \quad \spatialvar\in \domain \;.
\end{equation}
Here, $V$ is a (trainable) value matrix.

The weights $W^{(l-1)}_1,W^{(l-1)}_2$ in \eqref{eq:EL} are linear (matrix-multiplication) 
operations that act pointwise: $v(\spatialvar)\mapsto W^{(l-1)}_i v(\spatialvar)$. 
Similar integral-operator representations can be obtained 
for a ``continuum cross-attention'' operator, 
see \cite[Appendix~2]{calvello2025continuumattentionneuraloperators} for details.

An important observation is that for $Q,K = 0$ and $V=I$, 
the continuum self-attention map
reduces to a simple averaging, which can be viewed a continuous analog to the pooling operation:
$$
A(v;Q,K,V)(\spatialvar) = \IE_{\spatialvarother\sim p(\spatialvarother;v,\spatialvar)}[Vv(\spatialvarother)] = \frac{1}{|\domain|} \int_{\domain} v(\spatialvarother)d\spatialvarother.
$$
This implies \emph{universal approximation} of transformers 
\cite[Theorems~22,23]{calvello2025continuumattentionneuraloperators}, 
as a consequence of the general result in \cite{lanthaler2023nonlocal}.
\section{Training}\label{sec:training}
Constructing a surrogate model for a function $\cG$ is typically framed as a supervised learning problem. 
This process relies on a training dataset consisting of $m\in\N$ input-output pairs of the form
\begin{equation}\label{eq:data}
(x_i, y_i)_{i=1}^m\qquad \text{where}\qquad y_i=\cG(x_i) + \eps_i\in\cY.
\end{equation}
The goal is to learn an approximation of $\cG$ from this data. 
Again, $x_i$ represent the input points, 
and the $\eps_i$ may account for (measurement) noise or numerical errors, depending on the source of the data. 
The training points $x_i$ {are assumed to be} %
i.i.d.\ draws from a {probability} measure $\mu$ on $\cX$. 
Alternatively, they can be a priori selected based on a fixed design strategy, 
or chosen adaptively to maximize the information gained about $\cG$, 
which is known as active learning.
\begin{remark}\label{rmk:whitenoise}
A common modelling assumption is that $\eps_i$ is white noise. 
In this case $y_i$ is formally not an element of $\cY$. 
We omit discussing this in more detail, and refer to \cite{RWZ24}.
\end{remark}

We distinguish between three cases:
\begin{itemize}
\item {\bf direct construction}: In some cases, the surrogate can be
  constructed directly without the need for optimization. This applies
  when there is a known, computable map from the data to the surrogate
  function, as in classical interpolation methods (e.g., the
  sparse-grid spectral operator in Section \ref{sec:sg}, or {kernel
    based approaches \cite{BATLLE2024112549} computed by directly
    solving a linear system}). {These methods typically require the
    ability to evaluate $\cG$ at specific, deterministically chosen
    points.}
\item {\bf convex objective}: For some models, the training objective is a convex function of the parameters $\theta$. This typically occurs when the surrogate $\cG_\theta$ depends affine linearly on $\theta$ and a standard loss function is used. The resulting optimization problem usually has at least one global minimum, which can be found efficiently. Examples include polynomial chaos surrogates (Section \ref{sec:pc}), compressed sensing based approaches \cite{ADCOCK2025106761}, 
  random feature models \cite{NN+AMS24}, and kernel-based methods \cite{BATLLE2024112549}.
\item {\bf non-convex objective}: 
This is the most common scenario in operator learning, particularly when deep neural networks are part of the architecture. Here, the surrogate $\cG_\theta$ depends nonlinearly on its parameters $\theta$, yielding a non-convex optimization landscape with multiple local minima. This includes the majority of deep learning-based architectures.
\end{itemize}

For simplicity, we let $\cX$, $\cY$ be Hilbert spaces in the rest of this section.

\subsection{Empirical Risk Minimization}
A typical loss function used in operator learning is the square loss. 
For a class of operator {surrogates} ${\hat \cG}_\theta$ 
depending on a finite dimensional parameter $\theta$, 
given data \eqref{eq:data} we then wish to minimize the empirical risk
\begin{equation}\label{eq:erm}
    \hat{\cR}_m(\theta):=\frac{1}{m}\sum_{j=1}^m \norm[\cY]{{\hat \cG}_\theta(x_j)- y_j}^2
\end{equation}
in the parameter $\theta$. 
In practice the {$\cY$}-norm is not computable, but must be approximated.

  Let us specify how to do this for specific encoder-decoder architectures
  \begin{equation}\label{eq:ermcg}
    {\hat \cG}_\theta = \cD\circ {g}_\theta \circ \cE.
  \end{equation}
  We assume $\cE:\cX\to\R^{d_\cX}$, $\cD:\R^{d_\cY}\to\cY$ to be linear as defined in \eqref{eq:EDframe} via orthonormal bases $(\psi_{\cX,i})_{i\in\N}$, $(\psi_{\cY,j})_{j\in\N}$, of $\cX$, $\cY$, i.e.\
  \begin{equation}\label{eq:ermed}
    \cE(x)=(\dup{x}{\psi_{\cX,i}}_\cX)_{i=1}^{d_\cX}\in\R^{d_\cX},\qquad
    \cD((c_j)_{j=1}^{d_{\cY}}) = \sum_{j=1}^{d_\cY} c_j\psi_{\cY,j}.
  \end{equation}
  Due to Parseval's identity it is natural to work directly on the coefficient sequences. 
We thus let
  \begin{align*}
    \bar{\cR}_m(\theta)&:=\frac{1}{m}\sum_{j=1}^m \norm[2]{\underbrace{{g}_\theta\circ \cE(x_j)-%
        (\dup{y_j}{\psi_{\cY,i}})_{i=1}^{d_\cY}
      }_{\in\R^{d_{\cY}}}}^2 \ {=\hat\cR_m(\theta)-\underbrace{\frac{1}{m}\sum_{j=1}^m\norm[2]{(\dup{y_j}{\psi_{\cY,i}})_{i>d_{\cY}}}^2}_{\text{truncation error}},}
  \end{align*}
  {where the last equality follows by the fact that
  $(\psi_{\cY,i})_{i\in\N}$ is a sequence of orthonormal vectors.} 
  Up to the additive and $\theta$ independent truncation
  error, %
  this objective is thus equal to the empirical risk $\hat{\cR}_m$ and
  is easily computable. For simplicity we assumed orthonormal bases in
  Hilbert spaces; when working with $p$-Banach frames, the coefficient
  based losses additionally differ up to the frame constants.
  
\subsection{Adding Derivatives}
\label{sec:LrnDeriv}
Let ${\hat \cG}_\theta$, $\cE$ and $\cD$ be as in the previous subsection. 
As proposed and analyzed in 
\cite{10.5555/3294996.3295182,OLEARYROSEBERRY2022114199,OLEARYROSEBERRY2024112555,
      luo2025dimensionreductionderivativeinformedoperator,yao2025derivativeinformedfourierneuraloperator}, 
one can add gradient information to the objective, {yielding so-called \emph{Sobolev training}.}
{For smooth classes of operators such as e.g.\ holomorphic maps $\calG$,
information on first derivatives of the operators to be emulated is often available.
For example, in shape optimization, shape-gradients are in any case numerically evaluated
(see, e.g., \cite{OLEARYROSEBERRY2024112555}).
In practice, inclusion of differential or gradient information 
in operator learning}
can significantly boost accuracy of the surrogate \cite{OLEARYROSEBERRY2024112555}. 
{Derivative-informed operator learning requires training data of the form}
\begin{equation*} 
  (x_j, y_j,Z_j)_{j=1}^m\qquad \text{where}\qquad y_j=\cG(x_j) + \eps_j\in\cY,\qquad
  Z_j = D\cG(x_j)+E_j\in L(\cX,\cY).
\end{equation*}
As before $x_j\in\cX$, and $\eps_j$ as well as $E_j$ 
are interpreted as noise or error terms. 

Using again the square loss, the risk becomes
\begin{equation}\label{eq:derloss}
\hat{\cR}_m^1(\theta)
:=\hat{\cR}_m(\theta) + \frac{1}{m}\sum_{j=1}^m\norm[]{\underbrace{\calD \circ D \coeffmap_\theta(\cE(x_j))\circ \calE(x_j)}_{= D{\hat \cG}_\theta(x_j)}- Z_j}^2,
\end{equation}
{where $\norm[]{\cdot}$ is a suitable norm (see below),}
$D\coeffmap_\theta\in \R^{d_{\cY}\times d_{\cX}}$ is the Jacobian of $\coeffmap_\theta:\R^{d_\cX}\to\R^{d_\cY}$, 
and we used that $\cE$, $\cD$ are linear maps.

Since $D{\hat \cG}_\theta(x_j)$ and $Z_j$ belong to $L(\cX,\cY)$, 
it would be natural to take the operator norm in \eqref{eq:derloss}. 
In practice the Hilbert-Schmidt norm is preferred, as it simplifies the calculation. 
Recall that for a linear map $A:\cX\to\cY$, 
its squared Hilbert Schmidt norm is 
{
$\norm[{\rm HS}(\cX,\cY)]{A}^2 =
\sum_{i\in\N}\norm[\cY]{Ae_i}^2=\sum_{i,k\in\N}\dup{Ae_i}{f_k}^2$, 
}
where $(e_i)_{i\in\N}$, $(f_k)_{k\in\N}$, are (arbitrary) orthonormal bases of $\cX$ and $\cY$. 
Choosing our orthonormal encoding and decoding bases $\psi_{\cX,i}$, $\psi_{\cY,k}$, 
we thus have
\begin{equation*}
  \norm[{\rm HS}(\cX,\cY)]{A}^2 = \sum_{i\le d_{\cX},k\le d_{\cY}}\dup{A\psi_{\cX,i}}{\psi_{\cY,k}}^2 + \underbrace{\sum_{i>d_{\cX}\text{ or }k>d_{\cY}}\dup{A\psi_{\cX,i}}{\psi_{\cY,k}}^2}_{\text{truncation error}}.
\end{equation*}

Since $\dup{D{\hat \cG}_\theta\psi_{\cX,i}}{\psi_{\cY,k}}=0$ 
whenever either $i>d_{\cX}$ or $k>d_{\cY}$ (cf.~\eqref{eq:ermcg} and \eqref{eq:ermed}), 
we obtain {(using the Hilbert-Schmidt norm in \eqref{eq:derloss})}
  \begin{align} \label{eq:barcrm}
    \bar{\cR}_m^1(\theta)&:=\bar{\cR}_m(\theta)+\frac{1}{m}\sum_{j=1}^m\norm[F]{\underbrace{D \coeffmap_\theta(\cE(x_j))- (\dup{Z_j\psi_{\cX,i}}{\psi_{\cY,k}})_{i\le d_{\cX},k\le d_{\cY}}}_{\in\R^{d_\cY\times d_\cX}}}^2\nonumber\\
                         &{=\hat\cR_m^1(\theta) -\underbrace{\frac{1}{m}\sum_{j=1}^m\Big(
                           \sum_{i>d_{\cY}}\dup{y_j}{\psi_{\cY,i}}^2 +
                           \sum_{i>d_{\cX}\text{ or }k>d_{\cY}}\dup{Z_j\psi_{\cX,i}}{\psi_{\cY,k}}^2\Big)}_{\text{truncation error}},}
  \end{align}
  where $\norm[F]{\cdot}$ is the Frobenius norm of this matrix in $\R^{d_\cY\times d_\cX}$. 
  {Since the truncation error in \eqref{eq:barcrm} is independent of $\theta$, minimizing
    $\bar{\cR}_m^1(\theta)$ is equivalent to the minimization of
    $\hat{\cR}_m^1(\theta)$. The term $\bar{\cR}_m^1(\theta)$ is easy
    to compute, and merely requires the Jacobians of $\coeffmap_\theta$. 
    In particular, all computations can be done at the level of coefficients.
    }
\section{Approximation Errors}\label{sec:approximation}
\label{sec:AppErr}
For ease of presentation, throughout this section we focus once more on $\cX$, $\cY$ being separable 
Hilbert spaces and stable linear representations as described in Section \ref{sec:stable}. 
However, many arguments extend straightforwardly to the case of {separable} Banach spaces 
which admit a {so-called Schauder basis (a countable basis with certain stability properties
analogous to Example~\ref{ex:frames}, 3rd bullet point). 
We refer to \cite{TriebelWavelets2008,TriebelBases2010} for constructions of such bases
for a wide range of separable Sobolev and Besov spaces which arise 
as data and solution spaces for PDEs.
}
\subsection{Universal Approximation}\label{sec:compactness}
We first address the question of what kind of operators can be learned {by finite-parametric surrogate families}.
For functions with finite dimensional in- and output, 
\emph{universal approximation with neural networks} 
(and polynomials, wavelets etc.) is classical, e.g., \cite{LESHNO1993861}. 
In infinite dimensions, {in the context of operator learning},
the situation becomes a bit more subtle. 
The next statement highlights the importance of compactness of the input set.

\begin{proposition}\label{prop:compactness}
  Let $\cX$ be a Banach space.
  Let $K\subseteq\cX$ be bounded and such that $\bar K$ is not compact. 
  Then there exists $\eps>0$, such that 
  for any $d\in\N$, any bounded linear $\cE:\cX\to\R^d$, 
  and for any continuous map $\cF:\R^d\to\cX$
  \begin{equation*}
    \sup_{x\in K} \norm[\cX]{x-\cF\circ \cE(x)}\ge \eps.
  \end{equation*}
\end{proposition}
\begin{proof}
Since $\bar K$ is closed but not compact, 
$\bar K$ is not totally bounded, meaning there exists $\eps>0$ 
such that $\bigcup_{x\in K} B_{2\eps}(x)\supseteq \bar K$ has no finite subcover. 
From this we can conclude that there exists a sequence $(x_j)_{j\in\N}\subseteq K$ 
such that
\begin{equation}\label{eq:xixjeps}
  \norm[\cX]{x_i-x_j}\ge 2\eps\qquad\forall i\neq j.
\end{equation}
On the other hand $\cE(x_j)_{j\in\N}$ is a bounded subset of $\R^d$, 
and therefore has a subsequence that is a Cauchy sequence. 
Without loss of generality, let this sequence be $\cE(x_j)_{j\in\N}$. 
Due to the continuity of $\cF$ 
  \begin{equation*}
    \lim_{N\to\infty}\sup_{i,j\ge N} \norm[\cX]{\cF(\cE(x_i))-\cF(\cE(x_j))} = 0.
  \end{equation*}
  Together with \eqref{eq:xixjeps} we conclude that
  \begin{equation*}
    2\sup_{i\in\N}\norm[\cX]{x_i-\cF(\cE(x_i))}\ge
    \sup_{i\neq j\in\N}|\norm[\cX]{x_i-x_j} -\norm[\cX]{\cF(\cE(x_i))-\cF(\cE(x_j))}|\ge
    2\eps.\qedhere
  \end{equation*}
\end{proof}
The proposition states that if $\bar K$ is not compact, 
and we use a bounded linear encoder $\cE$, 
then no architecture that processes the encoded input $\cE(x)$ continuously 
(which covers all common architectures) is capable of 
uniformly approximating the identity operator $\cG={\rm Id}$ on $K$. %
The statement does not extend to nonlinear encoders $\cE$. 
However, in most practical situations $\cX$ is a function space continuously 
embedded into the continuous functions, and the first step is 
to evaluate the input on a grid {in the physical (space-time) domain;}
this corresponds to a linear encoding. 
For any such method, uniform approximation of operators on non-compact input sets in $\cX$
is therefore in general impossible. 
We emphasize that this limitation does not apply to approximation 
with respect to weaker norms, 
such as $L^p$ norms, see, e.g., \cite{SLSMGEM22,HSZ24}.

A similar argument implies that for linear decoders, 
the operators' response manifolds $\cG(x)_{x\in K}\subseteq\cY$ 
must also be relatively compact; again this is not required for nonlinear decoders. 
We do not discuss this in further detail and 
instead turn to a positive result on universal approximation in Hilbert spaces from \cite{schwab2023deep}. 
To state it we first introduce the set of admissible activation functions from \cite{LESHNO1993861}
\begin{align*}
  \cM:=\{&\sigma\in L_{\rm loc}^\infty(\R)\text{ is not a polynomial{,} and the closure of}\\
  &\text{the points of discontinuity form a Lebesgue null set}\},
\end{align*}
{where $L_{\rm loc}^\infty(\R)$ is the set of (possibly unbounded) measurable 
functions on $\R$ which are (essentially) bounded on each bounded subinterval of $\R$.}
We consider surrogates of the form
  \begin{align}\label{eq:UAarchitecture}
    \{\cD\circ \coeffmap \circ \cE\,:\,&d_{\cX},d_{\cY}\in\N;~\cE:\cX\to\R^{d_\cX},~\cD:\R^{d_\cY}\to\cY\text{ are linear encoders/decoders }\nonumber\\
                              &\text{based on stable representations as in Section \ref{sec:stable}};\\
    &\coeffmap:\R^{d_\cX}\to\R^{d_\cY}\text{ is a feedforward NN with activation function }\sigma\}.
  \end{align}
  The next theorem corresponds to \cite[Theorem 3.1]{schwab2023deep}.

\begin{theorem}\label{thm:UA}
  Let $\cX$, $\cY$ be separable Hilbert spaces. 
  Then for every continuous $\cG:\cX\to\cY$ 
  there exists a sequence %
  $\cG^k=\cD^k\circ \coeffmap^k\circ\cE^k$, $k\in\N$, 
  with architecture \eqref{eq:UAarchitecture} 
  such that
  \begin{equation}\label{eq:UA}
    \forall x\in\cX:\qquad\lim_{k\to\infty} \cG^k(x)=\cG(x).
  \end{equation}
  The convergence is uniform on every compact subset $K$ of $\cX$.
\end{theorem}

{
  \begin{remark}\label{rmk:UnivAppr}
    Given a class of approximating mappings $\mathcal{H}$, 
    universal approximation theorems are often stated in the following form:
\begin{equation}\label{eq:approximation_quote}
  \begin{minipage}{0.7\textwidth}
    For every $\epsilon>0$ and every compact $K\subseteq\cX$ 
    there exists $\hat{\cG}_{\epsilon,K}\in\mathcal{H}$ 
    such that $\sup_{x\in K}\lVert\cG(x)-\hat{\cG}_{\epsilon,K}(x)\rVert_{\cX}<\epsilon$.
  \end{minipage}
\end{equation} 

If $\cX$ is finite dimensional, this is equivalent to Theorem \ref{thm:UA} 
as is easy to see (e.g., \cite[Prop.~3.4]{petersen2025mathematicaltheorydeeplearning}). 
In the infinite dimensional case, Theorem \ref{thm:UA} gives a strictly stronger statement.
\end{remark}
}

\begin{remark}
A typical setting where compactness holds is the following: 
let $\domain\subset \R^n$ be a bounded domain, with $\partial \domain$ Lipschitz.
Then for $s>s' \geq 0$, {and any $1\leq p,q \leq \infty$, it holds}
$B^{s}_{pq}(\domain) \subset B^{s'}_{pq}(\domain)$ with compact
embedding (Rellich's theorem) \cite{TriebelWavelets2008}.
\end{remark}
Universal approximation theorems for other architectures can for instance be found in 
\cite{SLSMGEM22,Kovachki2021,godeke2025newuniversaloperatorapproximation,kratsios2023approximation}.
For the recently popular transformer architectures, 
universal approximation has been shown
e.g.\ in \cite{calvello2025continuumattentionneuraloperators}.
\subsection{$C^k$ Operators}
\label{sec:Ck}
After having established universal approximation, we next discuss convergence rates over classes of operators. Classical approximation theory is concerned with smoothness spaces, such as the $k$ times differentiable functions. This motivates {us} to consider the unit ball of the $C^k$ operators from $\cX\to\cY$, i.e.\
\begin{equation*}
  C_1^{k}(\cX,\cY):= \set{\cG:\cX\to\cY}{\max_{0\le j\le k}\sup_{x\in \cX}\norm{D^j\cG(x)}\le 1}.
\end{equation*}
As we have seen, compactness of the input set is crucial, and therefore we focus on the uniform approximation of these operators on some compact set $K\subseteq\cX$. In order to qualify as true ``operator learning'', 
$K$ should not be contained in a finite {dimensional} subspace or manifold; 
this would render the problem effectively finite dimensional. 
To formalize this, we assume that $K$ contains a full ball within a compactly embedded subspace:
\begin{assumption}\label{ass:K}
  Let $\tilde\cX$ be an infinite dimensional Hilbert space that embeds compactly in $\cX$. There exists $r>0$ such that $\set{x\in\tilde\cX}{\norm[\tilde \cX]{x}\le r}\subseteq K$.
\end{assumption}
We start by recalling a by now classical result regarding 
continuous nonlinear $N$-widths of unit balls in $C^k([0,1]^d)$ from \cite{DeVore1989}.

{
\begin{theorem}\label{thm:devore}
  Let $d$, $k\in\N$. Then there exists $C>0$ such that
  for every map $\bar\cA:\R^N\to C^0([0,1]^d)$
  and every continuous
  map $\bar\Theta:C^0([0,1]^d)\to\R^N$
  there exists $g \in C^k([0,1]^d)$
  with $\norm[{C^k([0,1]^d)}]{g}\le 1$
  such that
\begin{equation}\label{eq:nonlnwidth}
  \norm[{C^0([0,1]^d)}]{g-\bar\cA(\bar\Theta(g))}\ge C N^{-k/d}\qquad\forall N\in\N.
\end{equation}
\end{theorem}
}

{An} interpretation of the theorem is as follows: 
the mapping $\bar\cA$ takes $N$ parameters to a function in $C^0([0,1]^d)$. 
We can think of it as some neural network (or other) framework parametrizing a function class with a finite number of parameters $\theta\in\R^N$. 
The mapping $\bar\Theta$ takes a $g\in C^0([0,1]^d)$ and 
returns a corresponding parameter vector $\bar\Theta(g)\in\R^N$. 
The composition $\bar\cA\circ\bar\Theta$ should then yield an approximation of the identity. 
It thus gives a lower bound on how well the function class 
$\set{g\in C^k([0,1]^d)}{\norm[C^k]{g}\le 1}$ can be uniformly approximated with $N$ parameters. 
Importantly, \emph{the parameter assignment $\bar\Theta$ is assumed continuous},
i.e.\ functions that are close have similar parameters.  This is a
form of stability, that is reasonable to request for a practical
algorithm.  Without such a restriction, the quantity
\eqref{eq:nonlnwidth} becomes zero, due to the existence of
space-filling curves \cite{DeVore1989}.
\begin{remark}
{Typical neural network training does not satisfy continuity 
  of the parameter
  assignment $\bar\Theta$: Random initialization and the use of
  stochastic optimization algorithms produce different weights on
  different runs, even for the same data. Nevertheless, some form of
  stability is necessary for any practical algorithm to avoid the
  pathology of space-filling curves. As such, while a simplification,
  the continuity hypothesis serves as a mathematical idealization, and
  the lower bounds remain meaningful to indicate the limits of stable
  approximation methods.}
\end{remark}

\begin{proposition}\label{prop:parameters}
Let $\cX$, $\cY$ be two separable Hilbert spaces. 
Let $K\subseteq\cX$ be compact and satisfy Assumption \ref{ass:K}. 
Then for any $\alpha>0$ there exists $C_\alpha>0$ such that
  \begin{equation*}
  \inf_{\substack{\Theta:C^0(\cX,\cY)\to\R^N~\text{cont.}\\
      \cA:\R^N\to C^0(K,\cY)}}~~\sup_{\cG\in C_1^k(\cX,\cY)} ~~\sup_{x\in K}\norm[\cY]{\cG(x)-\cA(\Theta(\cG))(x)}\ge C_\alpha N^{-\alpha}\qquad\forall N\in\N.
  \end{equation*}
\end{proposition}
\begin{proof}
It suffices to consider the case $\cY=\R$: If $\cY$ is any other Hilbert space, then the statement follows by lower bounding the supremum over $C_1^k(\cX,{\rm span}(e))\subseteq C_1^k(\cX,\cY)$, where $e\in\cY$ is an arbitrary norm one vector; the space $C_1^k(\cX,{\rm span}(e))$ 
can be identified with $C_1^k(\cX,\R)$.

{\bf Step 1.}
  Fix $d\in\N$ so large that $k/d<\alpha$. Let $(\tilde e_i)_{i\in\N}$ be an ONB of $\tilde \cX$, with $\tilde \cX$ as in Assumption \ref{ass:K}. By Gram-Schmidt orthogonalization, we can create an orthonormal system $(e_i)_{i\in\N}$ in $\cX$ from this sequence. By Assumption \ref{ass:K}
  \begin{equation*}
    \setc{\sum_{i=1}^dy_i\tilde e_i}{\sum_{i=1}^d y_i^2\le r}\subseteq K.
  \end{equation*}
  Hence, for some fixed $\lambda>0$ sufficiently small, 
  \begin{equation*}
    \setc{\lambda\sum_{i=1}^d y_i e_i}{0\le y_i\le 1}\subseteq K.
  \end{equation*}

{\bf Step 2.}
  Now let $\cA:\R^N\to C^0(K,\R)$ {be} arbitrary, and let $\Theta:C^0(\cX,\R)\to\R^N$ be continuous. Set
  \begin{equation*}
    \bar\cA:=\begin{cases}
                \R^N\to C^0([0,1]^d)\\
                \theta\mapsto \Big((y_1,\dots,y_d)\mapsto \cA(\theta)(\lambda\sum_{i=1}^d y_i e_i)\Big).
             \end{cases}
\end{equation*}
Next let $E:C^0([0,1]^d)\to C^0(\R^d)$ be a linear extension operator 
satisfying for all $f\in C^0([0,1]^d)$ and $g\in C^k([0,1]^d)$
\begin{equation*}
  \norm[C^0(\R^d)]{E(f)}\le C_E
  \norm[{C^0([0,1]^d)}]{E(f)}\qquad\text{and}\qquad
  \norm[C^k(\R^d)]{E(g)}\le C_E
  \norm[{C^k([0,1]^d)}]{E(g)},
\end{equation*}
for some fixed constant $C_E>0$. 
Existence of such (Whitney) extension operators is well-known {(e.g. \cite[Chap.~VI]{stein70})}.
For arbitrary $f\in C^j([0,1]^d)$, $j\in\{0,k\}$,
\begin{equation*}
  \cG_g(x):= E(g)(\lambda^{-1}\dup{x}{e_1},\dots,\lambda^{-1}\dup{x}{e_d})
\end{equation*}
defines an operator in $C^j(\cX,\R)$. 
Moreover, there exists a constant $C_\cG$ depending on $\lambda$ and $C_E$ such that for $j\in\{0,k\}$
\begin{equation*}
  \norm[C^j(\cX,\R)]{\cG_g}\le C_{\cG} \norm[{C^j([0,1]^d)}]{g}.
\end{equation*}
With this we define       
\begin{equation*}
  \bar\Theta:=\begin{cases}
                C^0([0,1]^d)\to\R^N\\
                g\mapsto \Theta(\cG_g).
              \end{cases}
\end{equation*}
Note that continuity of $\Theta$ implies continuity of $\bar\Theta$.

{\bf Step 3.}
The definitions in Step 2 imply for all $g\in C^k([0,1]^d)$
\begin{equation*}
  \bar\cA(\bar\Theta(g))(\By)=\cA(\Theta(\cG_g))\Big(\lambda\sum_{i=1}^dy_ie_i\Big)\qquad\forall\By\in [0,1]^d.
\end{equation*}
Therefore
\begin{align*}
  \sup_{\cG\in C_1^k}~~\sup_{x\in K}|\cG(x)-\cA(\Theta(\cG))(x)|
  &\ge 
  \sup_{\substack{g\in C^k([0,1]^d)\\ \norm[C^k]{g}\le C_\cG^{-1}}}~~\sup_{x\in K}|\cG_g(x)-\cA(\Theta(\cG_g))(x)|  \\
  &\ge\sup_{\substack{g\in C^k([0,1]^d)\\ \norm[C^k]{g}\le C_\cG^{-1}}}~~\sup_{x\in \set{\lambda\sum_{i=1}^d y_i e_i}{0\le y_i\le 1}}|\cG_g(x)-\cA(\Theta(\cG_g))(x)|  \\
  &=\sup_{\substack{g\in C^k([0,1]^d)\\ \norm[C^k]{g}\le C_\cG^{-1}}}~~\sup_{\By\in [0,1]^d}|g(\By)-\bar\cA(\bar\Theta(g))(x)|.
\end{align*}
By Theorem \ref{thm:devore}, the last quantity is lower bounded by $CN^{-k/d}$. Since $\cA$ and $\Theta$ were arbitrary, and $\alpha\le k/d$, this concludes the proof.
\end{proof}

According to the above theorem, under the assumption of continuous parameter dependence, there cannot exist an architecture capable of uniformly approximating $C^k$ operators at an algebraic rate. This is a fundamental restriction related to this function class. We can draw a similar conclusion for the amount of data that is required to learn $C^k$ operators.
\begin{proposition}\label{prop:data}
  Let $\cX$, $\cY$ be two separable Hilbert spaces. Let $K\subseteq\cX$ be compact and satisfy Assumption \ref{ass:K}, and let $x_1,\dots,x_N\in K$. Then for every $\alpha>0$ there exists $C_\alpha>0$ such that
  \begin{equation*}
    \inf_{\cA:\cY^N\to C^0(K,\cY)}\sup_{\cG\in C_1^k(\cX,\cY)}\sup_{x\in K}\norm[\cY]{\cG(x)-\cA(\cG(x_1),\dots,\cG(x_N))(x)}\ge C_\alpha N^{-\alpha}\qquad\forall N\in\N.
  \end{equation*}
\end{proposition}
Proposition \ref{prop:data} states that the amount of data $N$ required to reliably reconstruct an operator $\cG\in C^k$ needs to increase faster than any polynomial; also see \cite{kovachki2024datacomplexityestimatesoperator} for more precise statements and sharper (exponential) bounds. For $\cY=\R$, the statement is simply a special case of Proposition \ref{prop:parameters} with $\Theta(\cG)=\cG(x_1),\dots,\cG(x_N)$. The case of general $\cY$ can again be reduced to the real-valued one, similar as in Proposition \ref{prop:parameters}.

Related results on parametric and data complexity can found in \cite{Mhaskar1997-yq,MhaskarLocApprOp,kovachki2024datacomplexityestimatesoperator,lanthaler2024operatorlearninglipschitzoperators}. 
E.g., in \cite{lanthaler2024operatorlearninglipschitzoperators}, 
rather than assuming continuous parameter dependence, 
the perspective of quantization is taken. 
Specifically, the metric entropy is used to derive lower bounds on the number of bits required 
to approximate the class of Lipschitz operators. 
Moreover, \cite{schwab2023deep} provides a construction which \emph{does} 
achieve algebraic rates {without the CoD} 
for the approximation of Lipschitz operators. 
The proof therefore relies on discontinuous parameter dependence, however.

\subsection{Holomorphic Operators}
\label{sec:HolOp}
We have argued that in general a form of compactness of the input set is required, 
and that classical function spaces of finite smoothness are too big to be amenable for operator learning {(with uniform convergence)}. 
Thus, further restrictions are required. 
Motivated by well-known properties of parametric PDEs, a common class to consider 
are holomorphic operators \cite{CDS11,cohen_devore_2015,HSZ24,NEURIPS2024_3100d29d}. 
In this section we discuss a convergence result from \cite{HSZ24}.
\subsubsection{Smoothness Spaces}\label{sec:smoothness}
We introduce a scale of smoothness spaces. 
To this end we assume linear encoders and decoders defined over 
Hilbert frames $(\psi_{\cX,j},\varphi_{\cX,j})_{j\in\N}$ and $(\psi_{\cY,j},\varphi_{\cY,j})_{j\in\N}$ 
as introduced Section \ref{sec:stable}. 
For $s$, $t\ge 0$ we let
\begin{equation}\label{eq:XsYt}
  \cX^s:=\setc{x\in\cX}{\sum_{j\in\N}\varphi_{\cX,j}(x)^2 j^{2s}<\infty},\qquad
  \cY^t:=\setc{y\in\cY}{\sum_{j\in\N}\varphi_{\cY,j}(y)^2 j^{2t}<\infty},
\end{equation}
with the norms in those spaces given by the respective weighted sums. In the following we consider corresponding encoder, decoder maps
\begin{equation}\label{eq:EDrates}
  \cE:\cX\to\R^{d_\cX},\qquad \cD:\R^{d_{\cY}}\to\cY
\end{equation}
as in \eqref{eq:EDframe} with the above frames.

\begin{example}\label{ex:fourier}
  Let $\psi_{j}(z)=\sqrt{2}\sin(\pi j z)$ for $z\in [0,1]$ and $j\in\N$. Then
  $(\psi_{j}(z))_{j\in\N}$ together with $\varphi_j(f)=\int_0^1 f(z)\psi_j(z)\dd z$ 
is a (tight) Hilbert frame\footnote{{A tight Hilbert frame is a frame with frame bounds $\lambda = \Lambda$ in \eqref{eq:lambdaLambda}.}} on $\cX=L^2([0,1])$.
Moreover, $\cX^s=\set{f\in H^s([0,1])}{f(0)=f(1)=0}$.
\end{example}

As stated in \cite[Lemma 2]{HSZ24}, for $0\le s'<s$, the embedding $\cX^s\hookrightarrow \cX^{s'}$ is compact. In particular, balls in $\cX^s$ are compact in $\cX$; they thus satisfy the minimal criteria of input sets required for the uniform approximation of an operator $\cG:\cX\to\cY$, see Section \ref{sec:compactness}. 

\subsubsection{Convergence Rates}
\label{sec:ConvRts}
To state the convergence rate result, 
we first make the holomorphy assumption on $\cG$ more precise. 
To this end we denote by $\cX^s_\C$ the complexification of the Hilbert space $\cX^s$, 
and similar for $\cY$.

\begin{assumption}\label{ass:G}
  Let $r>0$. There exists $s>1$, $t>0$, $M<\infty$, and an open set $O_{\C}\subseteq\cX_\C$ containing $\set{r\sum_{j\in\N} j^{-s} \coeff[j] \psi_{\cX,j}}{\coeff \in [-1,1]^\N}$, such that $\sup_{x\in O_\C}\norm[\cY_\C^t]{\cG(x)}\le M$ and $\cG:O_\C\to\cY_\C$ is holomorphic.
\end{assumption}

The assumption states two things: First, $\cG$ allows a complex differentiable extension to an open subset $O_\C$ of $\cX_\C$. Second, its outputs belong to $\cY_\C^t$. The first part of the assumption implies that the operator response depends very smoothly on the input. In particular it is not just $k$ times differentiable, but it is even analytic. The set $O_\C$ of analyticity describes a region in the complexified space related to the ball of radius $r$ in $\cX^s$ centered at $0$.
The assumption essentially requires that $\cG$ is not just well-behaved on the real-valued ball, but also in a sufficiently large complex neighborhood around it. This is a common requirement for deriving approximation results. The second part of the assumption implies a certain decay of the coefficients of the output functions. 
{While the assumption looks rather technical, the next example demonstrates
its wide applicability in the context of learning PDE solutions.}

{
  \begin{example} %
    Let $\domain \subseteq\R^d$ be a bounded Lipschitz domain, and denote by
    $\cX=L^\infty(\domain)$ the Lebesgue space of essentially bounded
    functions $x:\domain \to\R$, 
    and by $\cY=H_0^1(\domain)$ the Sobolev space of
    square integrable functions $y:\domain \to\R$, 
    whose first weak derivative is also square integrable.

    Additionally fix $f\in \cY'$ and denote
    \begin{equation}\label{eq:darcy}
      \CN(x,y):=\nabla\cdot (x\nabla y)-f\in \cY'.
    \end{equation}
    For any $x_0 \in L^\infty_{a_-}(\domain)\subset\cX$
    (cp.~\eqref{eq:a+}) with strictly positive essential infimum on
    $\domain$, the Lax-Milgram Lemma guarantees that
    $\CN(x_0,\cdot):\cY\to\cY'$ is a bounded isomorphism. In
    particular, there exists a unique $y_0\in\cY$ such that
    $\CN(x_0,y_0)=0$, and the Fr\'echet derivative
    $D_y\CN(x_0,y_0)(v)=\nabla\cdot (x_0\nabla v)\in\cY'$, $v\in \cY$,
    is an isomorphism from $\cY\to\cY'$.  The implicit function
    theorem in Banach spaces, e.g., \cite[Theorem
    4.B]{zeidler1986nonlinear}, then implies the existence of a
    neighbourhood $O\subseteq \cX$ of $x_0$ and a differentiable
    function $Y:O\to\cY$ such that $\CN(x,Y(x))=0$ for all $x\in
    O$. Note that $Y$ is the mapping taking a diffusion coefficient
    $x$ to the solution $y$ of the Darcy flow PDE \eqref{eq:darcy}.
    This argument extends directly to the complexified spaces, and can
    thus be used to verify that the solution map is locally
    \emph{complex differentiable} in a neighbourhood of $x_0$ in
    $\cX_\C$. Together with the compactness of the parameter set, one
    can then derive holomorphy and uniform boundedness as stated in
    the assumption; for more details also see the concrete example in
    Section \ref{sec:RD}, and, e.g., \cite[Sec.~1.2]{JZDiss} for the
    general treatment. This corresponds to the first part of
    Assumption \ref{ass:G}.

    The second part of the assumption states that the solution map $Y$
    maps into a space $\cY_\C^t$, where {the parameter $t>0$ 
    quantifies the coefficient decay}. 
    For PDE models like \eqref{eq:darcy}, this can be
    understood as a regularity assumption on the PDE solution. For
    example, if $t\ge 0$, $f\in H^t(\domain)$, $x\in C^{t+1}(\domain)$, and
    $\domain$ is a $C^{t+2}$-domain, then classical regularity theory
    implies $Y(x)\in H^{t+2}(\domain)$, e.g., \cite[Sec.\ 6.3]{evans10}. 
    In suitable representation systems, 
    this translates to coefficient decay of the type
    \eqref{eq:XsYt}, cf.~Example \ref{ex:frames}.

    We emphasize that the above argument relies only on two crucial observations:
    the differential operator
    $\CN$ depends holomorphically (in the present case even linearly) 
    on the parameter $x$, and $D_y\CN(x_0,y_0):\cY_\C\to \cY_\C'$
    is an isomorphism; the latter usually follows by well-posedness of the PDE.
    For this reason the same argument extends to a large range of PDEs
    and related problems (under additional suitable assumptions),
    e.g., Maxwell's equations \cite{JSZ17}, 
    the Navier-Stokes equations \cite{CSZ18_2319}, 
    or Boundary integral operators \cite{Henriquez2021};
    also see \cite{Cohen2016} for a general treatment.
\end{example}
}

This will allow to approximate the output within finite dimensional spaces. 
There holds the next theorem \cite[Corollary 1]{HSZ24} 
{(see the discussion above Theorem \ref{thm:NNStIVP} for the notion of network size)}.

 \begin{theorem}\label{thm:HSZ}
   Let $\delta$, $r>0$. Let $\cG$ satisfy Assumption \ref{ass:G} with $s>1$, $t>0$,
   and let $\cE$, $\cD$ be as in Section \ref{sec:smoothness}. 

   Then there exists a constant $C>0$ such that for all $N\in\N$ there exist $d_\cX$, $d_\cY\in\N$ in \eqref{eq:EDrates}, and a 
   ReLU neural network ${g^N}:\R^{d_{\cX}}\to\R^{d_{\cY}}$ of size $O(N)$ such that
   \begin{equation*}
     \sup_{x\in B_r(\cX^s)} \norm[\cY]{\cG(x)-\cD\circ {g^N}\circ \cE(x)}\le C N^{-\min\{s-1,t\}+\delta}.
   \end{equation*}
   Here $B_r(\cX^s)$ denotes the ball of radius $r$ and center $0\in \cX^s$.
\end{theorem}

\begin{remark} \label{rmk:OthAct}
  In the cited reference, the approximator network $g^N$ has been constructed with ReLU activations.
  However, we expect that the same results hold also with other, smooth activations such as $\tanh( )$
  or sigmoidal activations.
\end{remark}

\begin{remark}\label{rmk:sg}
  Sparse-grid interpolation surrogates can achieve the same rate in terms of the amount of data $m$, \cite[Theorem 4]{HSZ24}. 
As mentioned, these surrogates are explicitly computable without the need for training, 
\emph{under the provision that one can numerically access (``query'') 
  $\cG$
      with any admissible input data, to prescribed accuracy or ``fidelity''}.
In many practical situations, such as data-assimilation, this is not possible and one has to rely on 
a probabilistic setting, assuming training data is drawn i.i.d.\ subject to a probability measure charging
the ensemble of admissible inputs. 
The assumption that this measure be known is quite strong and not easily verifiable.
\end{remark}

Theorem~\ref{thm:HSZ}
states that it is possible to approximate holomorphic operators on balls in $\cX^s$. 
While such a ball is \emph{not} compact in $\cX^s$, 
it \emph{is} compact in the topology of $\cX$. %
Therefore, Theorem \ref{thm:HSZ} does not contradict Proposition \ref{prop:compactness} if we take $\cG={\rm Id}$ and $\cY=\cX$. Sharper results, convergence in $L^2$, and further details can be found in \cite{HSZ24}. The proof of the theorem is intricate and long, so we do not repeat the full argument. We do give some intuition however.

It is instructive to consider the Fourier type frame basis in Example \ref{ex:fourier} on $\cX=\cY=L^2([0,1])$. 
Suppose that $\cG:\cX\to\cY^t$. The encoder $\cE$ computes a finite number of input coefficients, $d$. 
The neural network ${g}^N:\R^{d_\cX}\to\R^{d_\cY}$ then approximates the action of the operator in the coefficient space; 
it learns the mapping for the first $d_{\cY}$ output coefficients $(g_i)_{i=1}^{d_\cY}$. 
Finally, the decoder $\cD$ synthesizes these $d_\cY$ coefficients back into a function in $\cY=L^2([0,1])$.

In the current setting, for $x\in\cX=L^2([0,1])$ with
\begin{equation}
   \coeff = (\coeff[i])_{i\in\N},\qquad  \coeff[i] = \dup{x}{\sin(i\pi\cdot)}_{L^2([0,1])}
\end{equation}
we want to approximate the operator $\cG$ on the set
\begin{equation*}
  B_r(\cX^s)=\setc{x\in\cX}{\sum_{j\in\N}  \coeff[j]^2 j^{2s}<r^2}\subseteq H^s([0,1]).
\end{equation*}
With the above choice of basis,
\begin{equation*}
  g_i( \coeff) = \dup{\sin(i\pi\cdot)}{\cG\Big(\sum_{j\in\N}\coeff[j] \sin(\pi j\cdot)\Big)}_{L^2([0,1])}.
\end{equation*}
In words, $g_i$ maps the sequence of Fourier coefficients of the input $x\in \cX$, to the $i$th Fourier coefficient of the output $\cG(x)\in\cY$.

Due to the holomorphy of $\cG$, the map $g_i$ is holomorphic in each $\coeff[i]$. %
{A particular consequence is}
that $g_i$ allows (locally) a Taylor polynomial chaos expansion
\begin{equation}\label{eq:gitrunc}
  g_i( \coeff) = \sum_{\Bnu\in\cF}
  \underbrace{\frac{\partial^\Bnu g_i(\Bnul)}{\Bnu!}}_{=:c_{i,\Bnu}}
   \coeff^\Bnu
\simeq \sum_{\Bnu\in\Lambda_i}c_{i,\Bnu} \coeff^\Bnu \qquad\forall i\in\N,
\end{equation}
whereas before $\Lambda_i\subseteq\cF$ is a finite subset.

To accommodate a limited budget, we wish to compute and store only the most important coefficients $c_{i,\Bnu}$ contributing to the approximation of $\cG$ via \eqref{eq:pcoperator}; we thus wish to choose the $\Lambda_i$ 
possibly efficient to achieve maximum compression of the operator {(in particular with small cardinality of $\Lambda_i$)}.

There are two major factors at play, as described in Section \ref{sec:pc}: Both for the input $x\in \cX^s=\set{x\in H^s}{x(0)=x(1)=0}$ and the output $\cG(x)\in\cY^t=\set{x\in H^t}{x(0)=x(1)=0}$ the low frequency part dominates the high frequency part. Thus input coefficients $\coeff[j]$ with lower $j$ tend to be larger, and output coefficients $g_i( \coeff)$ with lower $i$ tend to be larger. By consequence terms with large $i$, or with $\Bnu$ such that $\nu_j\neq 0$ for large $j$, are small in \eqref{eq:gitrunc}. The parameters $s$ and $t$ encode this decay in importance: the larger $s$ and $t$, the stronger the coefficient decay, the better the achievable compression. This translates to faster convergence rates for truncated polynomial expansions.

On a technical level, the Taylor expansions coefficients $c_{i,\Bnu}$ can be bounded by investigating the domain of holomorphy of the $g_i$, e.g., \cite{CDS11}. Finally, the last step of the proof of Theorem \ref{thm:HSZ} is to reapproximate the truncated Taylor expansions \eqref{eq:gitrunc} by ReLU networks; this exploits polynomial approximation results \cite{Yarotsky16}. The proof of Theorem \ref{thm:HSZ} is thus based on a careful analysis of the influence of input frequencies on frequencies of the operator output.

\begin{codebox}
We compare the empirical convergence of spectral surrogates, PCA-Net, and FNO in notebook \texttt{5\_Comparison.ipynb}. In particular, we compare these operator surrogates in terms of their cost-accuracy trade-off, where we separately consider cost contributions from training data generation, surrogate setup runtime, and surrogate evaluation runtime.
\end{codebox}

\subsection{Algorithm unrolling}
\label{sec:Unrol}
The idea of constructing deep neural networks based
on iterated maps is not new. It has been used 
e.g.\ in the Telgarsky and Yarotsky constructions of 
ReLU approximations of polynomials \cite{Yarotsky16},
and for exponentially convergent expressions by ReLU NNs 
of certain fractal function systems with low Besov regularity
in \cite{DDFHP22}.
The corresponding NNs exhibit in part a recurrent structure, 
with the recurrent block corresponding to {a ReLU emulation}
of the iterator map. 
{The terminology `unrolling' for this idea 
    seems to have been coined in \cite{monga2020algorithm}. 
    Algorithm unrolling has been leveraged for building deep solution operator surrogates 
    of operator equations in \cite{FSZ25_1134}. 
    In \cite{zhao2026hycophybridcompositionoperators}, 
    unrolling of time-stepping schemes 
    is combined with spatial (elliptic) solution operator surrogates 
    to define solution operators for parabolic evolution equations.
    In \cite{MECHeiss24}, unrolled multi-grid algorithms have been used
    to build ONets with U-Net architecture for parametric, elliptic PDEs.}

Mathematical convergence guarantees for 
iterative solvers of {fixed-point} %
or Newton-like iterations 
are generally true for open intervals of steering parameters.
\emph{ONet approximations of the iterator-maps} 
in constructive existence proofs, also in function space,
or direct emulation of discretization schemes for PDEs, 
will imply existence of finite-parametric operator networks with convergence guarantees.

In unrolling emulations of existing, standard PDE solution algorithms with 
convergence guarantees and error vs. work bounds, 
such as for example multigrid solvers \cite{MGNetJCP22,MECHeiss24} 
of Finite Element (FE) Galerkin discretizations of elliptic PDEs, 
the curse of dimensionality is overcome by mesh-independent preconditioning and/or
mesh-independent convergence rates of approximate, iterative solvers.

\emph{Existing iterative PDE solvers and constructive mathematical existence proofs}
are ``unrolled'' into deep operator networks with 
(sub-networks exhibiting) a recurrent structure,
corresponding usually to either timestepping (for evolution equations)
or to iterative (multilevel) solvers {(for stationary problems).}

\emph{{Neural operators} obtained via the unrolling approach
require lower regularity of input-output maps than
the holomorphic, polynomial chaos based approximation architectures.
}
For example, 
for scalar hyperbolic conservation laws where e.g.\ the flux-to-entropy-solution
dependence is known to be at best Lipschitz (see, e.g., \cite{HoldenFrtTrck}), 
unrolled {neural operators} can provide efficient flux-to-solution maps \cite{DeRSM24}
{which are a key component in finite-volume solvers of hyperbolic PDE.}

Due to mesh-independent contraction rates of 
Newton- or Fixed-point-iterations,
{once the iterator (i.e. the operator which realizes one Newton- or Fixed-point-iteration step)
has been learned, to sufficient accuracy,}
the depth of unrolled {neural operators} 
scales polylogarithmically with the target accuracy.

Examples for this approach are, for elliptic PDEs,
{the ``MGNet'' architecture \cite{MGNetJCP22}, \cite{MS2023,MS24_1105} 
obtained by unrolling multi-grid solution algorithms}, 
and {neural operators obtained from unrolling of 
time-stepping schemes} for nonlinear evolution equations \cite{DeRSM24,choi2025datafreeneuraloperator}.
A framework for leveraging constructive mathematical existence proofs 
(e.g.\ via fixed-point or Newton-like iterations in function space) 
resulting in deep recurrent surrogate architectures 
is developed in \cite{FSZ25_1134}.
Examples of {neural operators} 
built by unrolling PDE solution algorithms are available
for hyperbolic PDEs \cite{DeRSM24}, parabolic PDEs 
\cite{DeRyckSM-2024b,lorenz2024errorestimationphysicsinformedneural,choi2025datafreeneuraloperator}, 
and elliptic PDEs \cite{MS2023,MS24_1105,MGNetJCP22}.

\section{Generalization Errors}\label{sec:generalization}
\label{sec:StatErr}
While an analysis of approximation errors yields theoretically achievable convergence rates for certain types of architectures, these rates are not necessarily observed when training an architecture to minimize an empirical risk, e.g., \cite{Grohs2024,grohs2025theorytopracticegapneuralnetworks}. 
{Exceptions include}
interpolation based architectures 
(which do not {necessarily} require risk minimization, {see Section \ref{sec:sg}}) %
and also architectures with linear parameter dependence, e.g.\ \cite{BATLLE2024112549,ADCOCK2025106761}. 

Let $\mu$ be a measure on $\cX$ and suppose $x_1,\dots,x_m\overset{\text{iid}}{\sim}\mu$ are random samples. Given data pairs $(x_i,y_i)_{i=1}^m$ generated via \eqref{eq:data}, as explained in Section \ref{sec:training}, during training the empirical risk \eqref{eq:erm} is minimized. Denote such an empirical risk minimizer by $\hat\cG_\theta$. Our ultimate goal is to achieve a small population risk, which is the expected squared error over the true data distribution $\mu$
\begin{equation}\label{eq:populationrisk}
  \norm[L^2(\mu)]{\cG(x)-\hat\cG_\theta(x)}^2,
\end{equation}
also referred to as the population risk. The question of generalization errors is the following: Is the empirical risk close to the population risk? Intuitively we assume this to be true, if the amount of data samples $m$ in \eqref{eq:erm} is large enough, since \eqref{eq:erm} can be interpreted as an approximation to \eqref{eq:populationrisk}. This implies in particular that a small empirical risk yields a small population risk. In this case the approximation $\hat\cG_\theta$ of the operator $\cG$ generalizes well to unseen data outside of our training sample $x_1,\dots,x_m$.

We next recall a statement of this type from \cite[Theorem 3.15]{RWZ24}, which is in the same setting as the approximation result Theorem \ref{thm:HSZ}; in particular $\cG$ is a holomorphic operator.

\begin{theorem}
\label{thm:CnvRt}
Let $\cX$, $\cY$ be two separable Hilbert spaces, 
let $\cE$, $\cD$ be as in Section \ref{sec:smoothness} and 
let $\delta>0$. 
Suppose further $\cG:\cX\to\cY$ satisfies 
Assumption \ref{ass:G} for some $s>1$, $t>0$, 
and assume that $\mu$ is a probability measure on $\cX$ 
such that
  \begin{equation*}
    {\rm supp}(\mu)\subseteq \setc{x\in\cX}{\sup_{j\in\N}j^t|\varphi_{\cX,j}(x)|\le R}.
  \end{equation*}
Let data $(x_i,y_i)_{i=1}^m$ be generated by \eqref{eq:data}, 
where $\eps_i$ are i.i.d.\ draws from some subgaussian\footnote{We say that a $\cY$ valued random variable $\eps_i$ is subgaussian, if $\norm[\cY]{\eps_i}$ is a real-valued subgaussian random variable.} noise in $\cY$. 
There exists
a ReLU neural network architecture, 
such that a network $\hat\coeffmap_\theta$ minimizing the empirical risk
\begin{equation*}
    \frac{1}{m}\sum_{j=1}^m \norm[\cY]{\cD\circ\hat\coeffmap_\theta\circ\cE(x_j)-y_j}^2,
\end{equation*}
satisfies with $\kappa=\min\{s-1,t\}$
\begin{equation*}
    \bbE[\norm[L^2(\mu)]{\cD\circ \hat\coeffmap_\theta \circ\cE-\cG}^2]\le C m^{-\frac{2\kappa}{2\kappa+2}+\delta}.
\end{equation*}
Here the constant $C$ depends on $\delta$, but is independent of $m$, 
and the expectation is over the random choice of $x_i$, and noise terms $\eps_i$, for $i=1,\dots,m$.
\end{theorem}

{We point out that $\kappa$ in Theorem \ref{thm:CnvRt} corresponds
  to the \emph{approximation rate} in Theorem \ref{thm:HSZ}; for
  holomorphic operators, asymptotically, it is solely governed by the
  algebraic decay of the input ($s$) and output ($t$) coefficients.
  As discussed, these parameters describe the intrinsic dimensionality
  of the problem, or more precisely, they determine how fast the
  latent input/output dimensions need to increase as the desired
  accuracy increases. As $s$, $t\to\infty$ jointly, the convergence rate
  $\frac{2\kappa}{2\kappa+2}$ in Theorem \ref{thm:CnvRt} gets
  arbitrarily close to $1$.}

\begin{remark}
As mentioned in Remark \ref{rmk:whitenoise}, 
it is common to model the noise not as subgaussian, but as white noise. 
The statement also extends to this case, but with a more technical formulation.
\end{remark}

The theorem gives convergence rates for empirical risk minimizers
trained on noisy data.  In particular it shows that holomorphic
operators can be learned with moderate amounts of data.  Its proof is
based on adapting standard techniques from M-estimation and empirical
process theory to the case of operators \cite{Dirksen.2015,VDG00}.
While we give a simplified version postulating only existence of a
suitable network, the precise architecture of $\coeffmap_\theta$ is
given in \cite{RWZ24}. 
{Importantly, 
it corresponds to a \emph{sparsely connected deep neural network} 
(hence, with nonlinear parameter dependence); 
the terminology ``sparsely connected''
refers to a neural network architecture
for which various weights are apriorily set to zero, and which
are not adapted during training (i.e., do not belong to the
set of trainable parameters). The theorem remains true also for
\emph{fully connected} ReLU networks, albeit with a
worse convergence rate.} For neural networks for which only the
weights in the last layer are trained (hence, with linear parameter dependence), related results can also be found in
\cite{ADCOCK2025106761,NEURIPS2024_3100d29d}. Moreover, we have
already seen that weaker regularity assumptions in general lead to a
curse of dimensionality in terms of the required data, see Proposition
\ref{prop:data}. For Lipschitz operators this is discussed in more
detail in \cite{JMLR:v25:22-0719}, where statistical convergence rates
are derived. For lower bounds we refer once more to
\cite{kovachki2024datacomplexityestimatesoperator}, and for the FNO
architecture, some results regarding empirical risk minimization
without noisy data are discussed in \cite{Kovachki2021}. 
{Finally,
  we also refer to \cite{2512.17805} for recent results on minimax
  rates for the class of H\"older continuous operators.}

\section{Conclusion {and Further Directions}}
\label{sec:Concl}
Parsimonious operator surrogates have seen and continue to see at the time
of this writing a very dynamic development, both on the algorithmic, as
well as on the theoretical side. 
They emerge as one key building block of SciML, in rather diverse 
computational tasks as, e.g., optimal design, 
shape-optimization \cite{OGPDEOptDeriv,harbrecht2026neuralshapeoperatorsurrogates}, 
Bayesian Inversion \cite{OGBipDeriv},
and the currently active area of 
Digital Twins \cite{liu2024deepneuraloperatorenabled,KWDigTwin24,
kapteyn2025tumortwinpythonframeworkpatientspecific, tezzele2024adaptiveplanningriskawarepredictive}.

Universality of neural operators has been established for a wide range of architectures in 
\cite{li2020fourier,schwab2023deep,lanthaler2023nonlocal,acciaio2022metric,godeke2025newuniversaloperatorapproximation}.

Currently, 
\emph{approximation rate theory in smoothness scales} is available only in few cases.
A broad class of operators which afford expression rates free from the CoD
are holomorphic maps 
\cite{CDS11,CSZ18_2319,cohen_devore_2015,SLSMGEM22,AJSZ20_2734,JSZ17,HSZ24,RWZ24}.
Here, neural network based surrogates are, 
essentially, equivalent to other {classes of approximators $g^N$ in \eqref{eq:DgEN},}
such as spectral surrogates as in \cite{FanaskOselDAN23,HSZ24}.

Expression rate results are also available for 
{neural operator surrogates} built from `unrolling' existing PDE solvers into neural network architectures.
We refer to \cite{deng2021convergence,MishraActa24} for a survey of results
in this direction, {to \cite{FSZ25_1134}, and to} 
the discussion in Section~\ref{sec:Unrol}.

An important area not covered here is the training error,
when {neural operator} architectures are calibrated with i.i.d.\ training data.
For general theory of
neural network training we refer, e.g., to the recent 
\cite{marcotte2025intrinsictrainingdynamicsdeep}.
{
Of particular interest is enhancing efficiency in Operator Network training
in architectures \eqref{eq:DgE} with 
\emph{multi-level wavelet- or frame-encoders and 
-decoders}, 
as outlined in Example~\ref{ex:frames} item four, and Figure~\ref{fig:wavelet}. 
Multi-level encoders and decoders naturally induce a 
multi-level structure in the parametric approximator map $g^N$ in \eqref{eq:DgEN},
with different components of the parametric surrogates $g^N$ 
requiring different accuracy in order to ensure a prescribed accuracy of the associated 
neural operator ${\hat\calG}^N$ in \eqref{eq:DgEN}, as map between function spaces.
This is naturally accommodated in the presented framework by suitable choices of 
weight-sequences in \eqref{eq:XsYt}, 
see also \cite{HSZ24} for more general choices of weight sequences than those in \eqref{eq:XsYt}. 
In the context of Neural Operator Training, 
multi-level splitting of input- and output-spaces is leveraged for efficiency in \cite{MLMC4ONet}.
}

Approximation architectures other than (generalized) polynomial chaos 
have been explored for operator learning, in particular for dynamical systems
in high- or infinite-dimensional state- and parameter spaces.
We mention only kernel methods e.g. \cite{BATLLE2024112549,KernDS24}, 
and {Gaussian} processes e.g. \cite{GPRegr21} and the references there,
{
and 
random-feature type neural approximations in solution representation formulas,
e.g. for initial-value problems for wave-equations 
\cite{cho2025learninglowrankneural}.
}
Detailed discussion and comparisons are beyond the scope of this chapter.

{Finally, we emphasize that the performance of different approaches and models can differ significantly depending on the problem; careful numerical experiments and comparisons are required to identify suitable architectures that are computationally efficient for specific applications. We refer to \cite{GEK-Fair22,JML-1-3,ESZ24} for several recent works in this direction. Moreover, \cite{Westermann2026Performance} provides an empirical comparison between neural and polynomial operator surrogates. This study observes that spectral methods typically outperform learning-based approaches in regimes of high regularity and low intrinsic parametric dimensionality. In other regimes, neural networks and learning-based approaches may be capable of detecting structures and forms of compression that appear unreachable by direct polynomial methods. Ultimately, however, these comparative advantages remain heavily problem-dependent.}

{\small
\bibliographystyle{abbrv} 
\bibliography{library,mybibfile}
}
\end{document}